\documentclass{article}

\PassOptionsToPackage{numbers, compress}{natbib}

\usepackage[main, final]{neurips_2026}

\usepackage[utf8]{inputenc}
\usepackage[T1]{fontenc}
\usepackage{hyperref}
\usepackage{url}
\usepackage{booktabs}
\usepackage{amsfonts}
\usepackage{amsmath}
\usepackage{amssymb}
\usepackage{amsthm}
\usepackage{nicefrac}
\usepackage{microtype}
\usepackage{xcolor}
\usepackage{graphicx}
\usepackage{algorithm}
\usepackage{algorithmic}
\usepackage{multirow}
\usepackage{subcaption}
\usepackage{wrapfig}
\usepackage{bm}
\usepackage{newtxmath}
\usepackage{colortbl}
\usepackage{tabularx} 
\usepackage{enumitem}
\usepackage{etoc}        

\hypersetup{colorlinks=true, linkcolor=red, citecolor=blue, urlcolor=blue}

\newcommand{\T}{\top}
\newcommand{\mF}{\boldsymbol{F}}
\newcommand{\mG}{\boldsymbol{G}}
\newcommand{\mK}{\boldsymbol{K}}
\newcommand{\mV}{\boldsymbol{V}}
\newcommand{\mU}{\boldsymbol{U}}
\newcommand{\mI}{\boldsymbol{I}}

\newtheorem{proposition}{Proposition}

\title{PhysGuard: Fisher-Guided Gradient Projection for Sim-to-Real Neural PDE Surrogates}

\author{
  Changjian Zhou$^{1}$ \thanks{Email: changjianz@student.unimelb.edu.au.} \quad
  Junfeng Fang$^{2}$ \quad
  Negin Yousefpour$^{1}$ \thanks{Corresponding author. Email: negin.yousefpour@unimelb.edu.au.} \quad
  Peng Wu$^{3}$ \\ [1mm]
  \bfseries Bin Yan$^{1}$ \quad
  \bfseries Guillermo A. Narsilio$^{1}$ \\[2mm]
  $^{1}$Faculty of Engineering and IT, University of Melbourne \\
  $^{2}$School of Computing, National University of Singapore \\
  $^{3}$Artificial Intelligence Research Institute, IFLYTEK Co., Ltd.%
}

\begin{document}

\maketitle
\setcounter{footnote}{0}
\etocdepthtag.toc{maintext}   

\begin{abstract}
Neural operator models trained on simulation data often lose accuracy when applied to experimental measurements due to the \textit{sim-to-real gap}. Standard fine-tuning with limited real data can reduce this gap, but it may also damage the core physics-relevant representations learned during pretraining. Although knowledge-preserving adaptation has been widely investigated in vision or language tasks, it remains unclear whether these methods are suitable for neural operators whose architectures and protected knowledge are fundamentally different. Neural operators need to preserve core-scale physical structures rather than semantic or visual features. We propose PhysGuard, a physics-preserving framework for accurate sim-to-real adaptation of neural operators. Specifically, PhysGuard uses the empirical Fisher Information Matrix computed on simulation data to identify physics-critical parameter directions, then restricts fine-tuning updates to directions that do not interfere with them. A layer-wise Gram-matrix formulation makes this efficient for models with millions of parameters, while an adaptive threshold automatically determines the protected subspace size. A spectral probe experiment shows that the dominant Fisher directions are strongly associated with low-frequency output structures. Experiments on RealPDEBench across four neural operator architectures and different physical systems show that PhysGuard performs strongly on most evaluation metrics compared to baselines. The benefits are most evident under severe domain shift, where it reduces low-frequency error by up to 32\% compared to standard fine-tuning while maintaining adaptability. Our code is available at {\urlstyle{same}\url{https://github.com/Data-Driven-Computational-Geotechnics/PhysGuard}}.
\end{abstract}

\section{Introduction}
\label{sec:intro}

Training on simulations and deploying in the real world has become a standard workflow in scientific machine learning (SciML)~\citep{karniadakis2021physics, brunton2020machine}. Neural operator models have emerged as a particularly promising class of neural PDE surrogates within this paradigm~\citep{kovachki2023neural}, learning solution mappings from large-scale simulation data and enabling rapid inference for complex physical systems~\citep{mccabe2024multiple, herde2024poseidon}. Yet their performance often degrades substantially when applied to real experimental measurements, a phenomenon known as the sim-to-real gap~\citep{hu2026realpdebench, mohan2024what}, as shown in Figure~\ref{fig:001-intro}. This gap stems from the idealized assumptions that underlie the simulations, whereas the experimental data exhibit measurement noise, unmodeled effects, and other sources of variability. Consequently, neural operators trained exclusively on simulation data may fail to generalize to real-world scenarios, limiting their practical effectiveness in engineering applications.

Fine-tuning on limited real data is the most direct solution~\citep{wang2025luna}, but it comes with an important risk. Current studies have shown that unconstrained fine-tuning can damage the low-frequency physical structures learned during pretraining~\citep{rahaman2019spectral, xu2019frequency}. Specifically, during optimization, the model may focus on fitting high-frequency noise while neglecting large-scale coherent patterns such as vortex streets and mean flow profiles~\citep{qin2024toward}, which often capture the core physics of the system.

\begin{wrapfigure}{r}{0.5\linewidth}
  \centering
  \includegraphics[width=\linewidth]{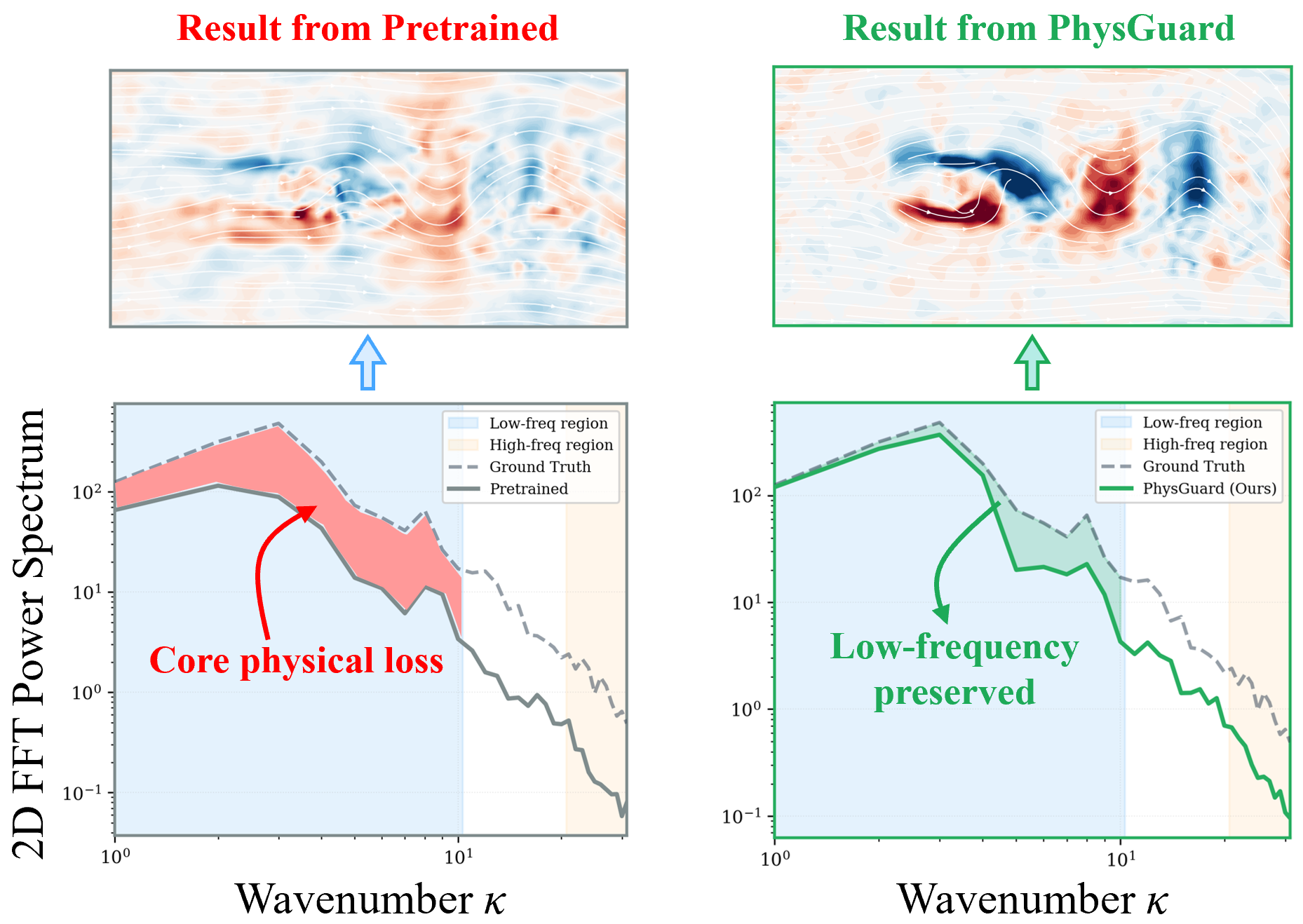}
  \caption{Predicted fields (top) and power spectra (bottom) for the simulation-pretrained model (left) and the same model after PhysGuard fine-tuning on real data (right). Dashed curves show the ground truth. PhysGuard better matches the ground-truth spectrum at low frequencies. Best viewed in color.}
  \label{fig:001-intro}
\end{wrapfigure}
A natural approach is to regularize the fine-tuning objective, for instance, by penalizing deviations from the pretrained weights~\citep{li2018explicit}. However, such penalties apply uniformly across all parameters and are sensitive to the choice of regularization strength. More structured alternatives have proven effective for vision and language tasks, including importance-weighted regularization (EWC~\citep{kirkpatrick2017overcoming}), gradient projection (GPM~\citep{saha2021gradient}), and parameter-efficient fine-tuning (LoRA~\citep{hu2022lora}). Yet, neural operators differ substantially from these models in both architecture and learning objective. Rather than encoding semantic or class-level features, they rely on mechanisms such as spectral convolutions~\citep{li2021fourier}, branch-trunk factorisations~\citep{lu2021learning}, and physics-aware attention~\citep{wu2024transolver} to preserve governing PDE physics. Whether these techniques transfer to the neural operator setting has not been systematically investigated.

We propose \textbf{PhysGuard}, a physics-preserving framework for sim-to-real adaptation of neural operators. The key insight is that physical knowledge in a pretrained neural operator concentrates along a small number of principal directions in parameter space (empirically characterized in Section~\ref{sec:fisher_analysis}). Building on this, PhysGuard computes the empirical Fisher Information Matrix (FIM) on simulation data and projects fine-tuning gradients onto the null space of its top eigenvectors. This blocks updates that would overwrite physics-critical parameters while leaving all other directions free for adaptation. The method introduces no auxiliary loss terms, and applies to any differentiable architecture. We validate it on the RealPDEBench~\citep{hu2026realpdebench} across four architectures and three physical scenarios. Our contributions are:
\begin{itemize}[leftmargin=*, itemsep=2pt, topsep=2pt]
\item We propose PhysGuard, a framework for sim-to-real adaptation of neural operators that preserves pretrained physics. A Gram-matrix kernel makes the subspace estimation tractable for million-parameter models.
\item Through a spectral probe, we reveal that the FIM spectrum concentrates in a small number of directions in the tested architectures, and its top eigenvectors are associated with low-frequency output structures. This empirical finding has not been thoroughly reported.
\item PhysGuard improves most evaluation metrics across four neural operator architectures and three scenarios from RealPDEBench. To our knowledge, this is the first systematic study of knowledge-preserving fine-tuning for neural PDE surrogates.
\end{itemize}

\section{Related Work}
\paragraph{Neural operators for PDE surrogate modelling.}
Conventional PDE solvers are often too expensive for tasks that require repeated evaluation. Neural operators address this challenge by learning mappings between function spaces, such as from an initial condition or coefficient field to a PDE solution, and serving as fast surrogates~\citep{kovachki2023neural}. Concretely, a neural operator maps a discretized input function (e.g., an initial condition or a forcing field) to a discretized output function. Representative architectures include FNO~\citep{li2021fourier} and its factorised variant~\citep{tran2023factorized}, CNO~\citep{raonic2023convolutional}, DeepONet~\citep{lu2021learning}, Transolver~\citep{wu2024transolver}, and DPOT~\citep{hao2024dpot}. Recent efforts have also focused on multi-spatiotemporal-scale generalisation~\citep{gupta2022towards}. Detailed descriptions of the operators used in this work are given in Appendix~\ref{app:operators}. However, many existing models are trained and benchmarked only on simulation data~\citep{li2021fourier, lu2021learning, kovachki2023neural}, and current benchmarks such as PDEBench~\citep{takamoto2022pdebench} and CFDBench~\citep{luo2024cfdbench} provide only simulated ground truth. As a result, deploying these models on real-world measurements affected by sensor noise and systematic domain shift remains largely unexplored. RealPDEBench~\citep{hu2026realpdebench} provides the first benchmark with paired numerical and experimental data for this problem.
\paragraph{Sim-to-real transfer and physics-informed learning.}
Bridging the simulation-to-reality gap has been widely studied in robotics and computer vision, where domain randomisation~\citep{tobin2017domain, peng2018sim} and feature alignment~\citep{ganin2016domain} are standard tools. A complementary line of research introduces physical priors directly into the learning process. Physics-informed neural networks~\citep{raissi2019physics} enforce PDE residual constraints, while physics-informed machine learning more broadly incorporates structural knowledge into model architectures and training objectives~\citep{karniadakis2021physics}. However, these methods primarily aim to improve generalisation during training or rely on physics supervision throughout optimisation. They do not directly address the challenge of adapting an already-trained surrogate to a new domain while preserving the physics it has already acquired. Appendix~\ref{app:related_ext} discusses several potential approaches to transfer learning for neural operators.
\paragraph{Knowledge-preserving adaptation.}
Catastrophic forgetting~\citep{mccloskey1989catastrophic} is a central challenge when adapting pre-trained models to new domains. In the broader machine learning community, a rich set of strategies has been developed. Regularisation-based methods such as EWC~\citep{kirkpatrick2017overcoming}, Synaptic Intelligence~\citep{zenke2017continual}, and L$_2$-SP~\citep{li2018explicit} penalise deviations from pre-trained weights. Parameter-efficient methods such as LoRA~\citep{hu2022lora} and Adapters~\citep{houlsby2019parameter} freeze the backbone and insert small trainable modules. GPM~\citep{saha2021gradient} projects gradients away from activation subspaces, whereas GEM~\citep{lopez2017gradient} constrains interference with stored examples. OGD~\citep{farajtabar2020orthogonal} and ORTHOG-SUBSPACE~\citep{chaudhry2020continual} also use orthogonality constraints.  Similar ideas have also been explored in the context of knowledge editing in language models, as exemplified by AlphaEdit~\citep{fang2025alphaedit} and CrispEdit~\citep{ikram2026crispedit}. PhysGuard shares the gradient-projection mechanism with prior work but differs in its subspace construction. GPM builds its subspace from intermediate representations accumulated across a sequence of tasks, whereas PhysGuard constructs the subspace from loss-level Fisher information on a single simulation dataset, and targets the preservation of low-frequency physics rather than sequential task knowledge. These techniques have been extensively studied for language and vision models, yet their application to neural operator fine-tuning remains less explored. Appendix~\ref{app:related_ext_cl} compares these methods and their applicability.

\section{Method: Fisher-Guided Gradient Projection}
In this section, we present our approach to sim-to-real adaptation for neural operators. We first formulate the adaptation problem (Section~\ref{sec:formulation}), then identify physics-critical parameter directions via the empirical Fisher Information Matrix~\citep{amari1998natural, martens2020new} using a Gram-matrix formulation to address the associated computational challenges (Section~\ref{sec:fisher}), and finally constrain fine-tuning gradients to the complementary safe subspace (Section~\ref{sec:projection}).

\subsection{Problem Formulation}\label{sec:formulation}
Consider a neural operator model $f_\theta(\cdot)$ pretrained on a large simulation dataset $\mathcal{D}_{\text{sim}} = \{(x_i, y_i)\}_{i=1}^{M}$, yielding trainable parameters $\theta^* \in \mathbb{R}^d$ that capture the underlying physics of PDEs. Our goal is to adapt this model to real experimental data $\mathcal{D}_{\text{real}} = \{(\tilde{x}_j, \tilde{y}_j)\}_{j=1}^{M'}$ during fine-tuning, while retaining the physical knowledge encoded in pretraining.

As shown in Figure~\ref{fig:001-intro}, the core physics learned from simulation is reflected in the low-frequency components of the solution, which encode large-scale, globally coherent structures. We therefore use a low-frequency reconstruction loss $\mathcal{L}_{\text{low-f}}$ to quantify physics preservation, and express the desired fine-tuning objective as: 
\begin{equation}
\min_\theta \;\; \underbrace{\mathcal{L}_{\mathrm{real}}(\theta) = \frac{1}{M'}\sum_{j=1}^{M'}\ell\!\left(f_\theta(\tilde{x}_j),\, \tilde{y}_j\right)}_{\;\text{adapt to real data}} \quad \text{s.t.} \quad \underbrace{\mathcal{L}_{\text{low-f}}(\theta) \leq \mathcal{L}_{\text{low-f}}(\theta^*)}_{\text{preserve pre-trained knowledge}},
\label{eq:objective}
\end{equation}
where $\mathcal{L}_{\text{real}}$ is the loss in real data, and $\mathcal{L}_{\text{low-f}}$ measures the prediction error on the low-frequency components of the solution.

Immediately enforcing this constraint with a penalty term can introduce additional hyperparameters~\citep{kirkpatrick2017overcoming, li2018explicit, zenke2017continual}. Instead, we adopt a geometric perspective whereby fine-tuning updates are restricted to parameter directions along which the pretrained loss is insensitive, approximately satisfying the constraint by construction. The key question then becomes how to identify these directions, which we will address next.

\subsection{Physics Subspace Identification}\label{sec:fisher}
In Section~\ref{sec:formulation}, we define the sim-to-real training objective for neural operators, where the key challenge focuses on identifying the physics-critical subspaces from the pretrained model. This section describes how we achieve this in three steps.
\paragraph{Fisher Information Matrix.}
Since neural operator architectures vary widely, we require an architecture-agnostic measure of prediction sensitivity with respect to each parameter direction. The empirical Fisher Information Matrix (FIM) serves this purpose. Given $N$ samples from the simulation dataset, we perform a forward and backward pass with $\theta^*$ held fixed to evaluate the per-sample gradient $\textbf{\textit{g}}_i = \nabla_\theta \ell(\theta^*;\, x_i, y_i) \in \mathbb{R}^d$, where $d$ is the dimension of parameters. Stacking these into a matrix $\mG = [\textbf{\textit{g}}_1, \ldots, \textbf{\textit{g}}_N]^{\T} \in \mathbb{R}^{N \times d}$, the empirical FIM is then:
\begin{equation}
\mF = \frac{1}{N}\, \mG^{\T} \mG \;\in\; \mathbb{R}^{d \times d}.
\end{equation}
Eigenvectors of $\mF$ with large eigenvalues identify parameter directions along which predictions are most sensitive, corresponding to the physics-critical directions we wish to protect. Conversely, eigenvectors with small or zero eigenvalues span directions along which predictions are largely insensitive, representing directions available for adaptation.

Intuitively, because the pretrained model is optimized on smooth simulation data that is dominated by large-scale physical patterns, the directions to which the loss is most sensitive may be associated with output modes encoding these low-frequency structures. We empirically verify this connection in Section~\ref{sec:fisher_analysis}.

\begin{figure}
  \centering
  \includegraphics[width=\linewidth]{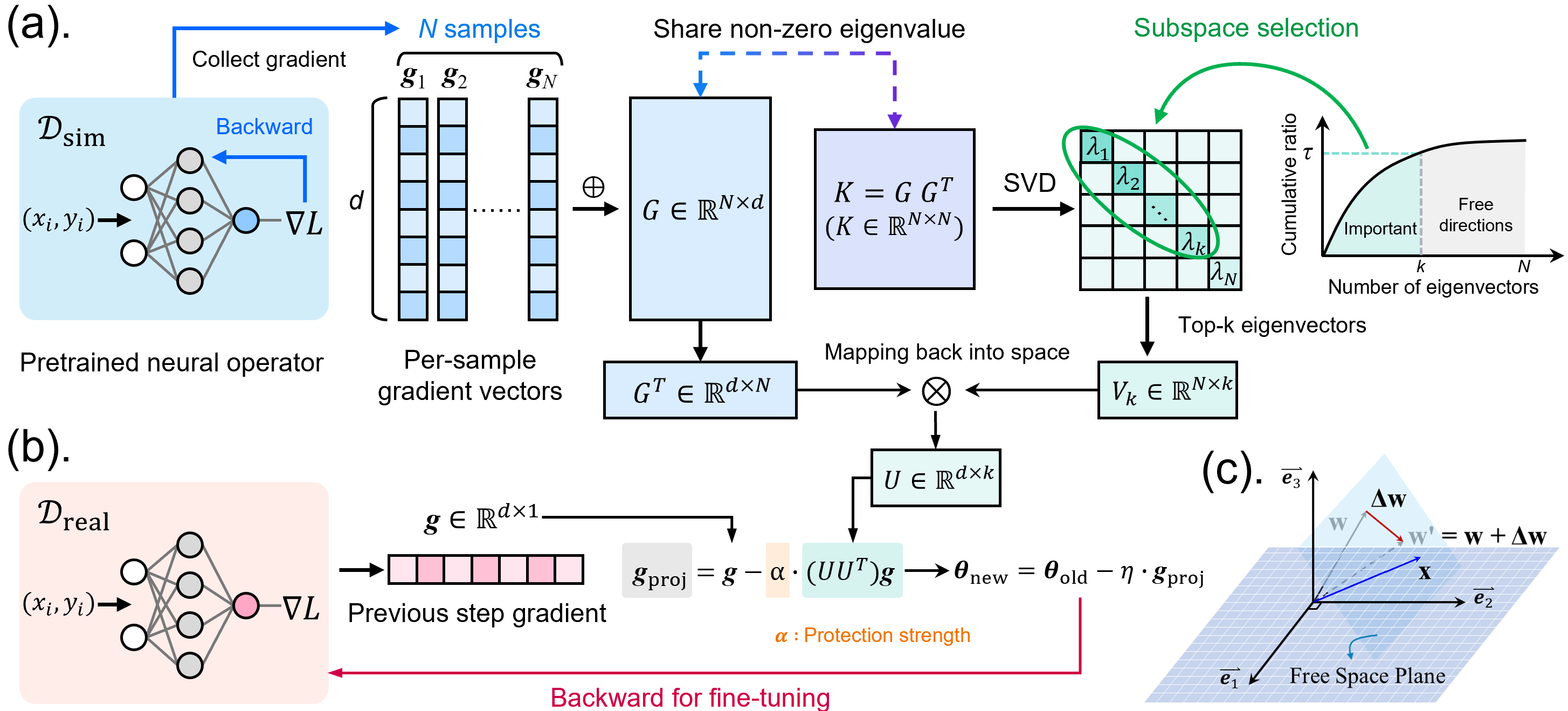}
  \caption{Overview of the PhysGuard framework.
  (a) Gradient vectors from a pre-trained model are stacked into $\mG$, and SVD of the Gram matrix $\mK$ is applied to identify the physics-critical subspace $\mU$;
  (b) During fine-tuning, each gradient $\textbf{\textit{g}}$ is projected onto the complement of $\mU$, ensuring parameter updates do not overwrite physical knowledge learned during pre-training;
  (c) Visualization of the null-space projection.
  Best viewed in color.}
  \label{fig:002-method}
\end{figure}

\paragraph{Efficient eigendecomposition based on a Gram matrix.}
Since the parameter dimension $d$ can reach millions, directly eigendecomposing the FIM poses a significant computational and memory challenge. To overcome this, we exploit the identity that $\mG^{\T}\mG$ and $\mG\mG^{\T}$ share the same non-zero eigenvalues. We therefore construct the much smaller Gram matrix\footnote{A small subset of samples suffices to estimate the subspace in the tested configurations. See Appendix~\ref{app:gram} for details.} $\mK = \mG\mG^{\T} \in \mathbb{R}^{N \times N}$ and compute its decomposition as follows.
\begin{equation}
\left\{\boldsymbol{V},\; \boldsymbol{\Lambda},\; \boldsymbol{V}^{\T}\right\} = \text{SVD}\!\left(\mK\right)
\label{eq:gram_eigh}
\end{equation}
where $\boldsymbol{\Lambda} = \mathrm{diag}(\lambda_1, \ldots, \lambda_N)$ with $\lambda_1 \geq \cdots \geq \lambda_N \geq 0$, and the columns of $\boldsymbol{V} \in \mathbb{R}^{N \times N}$ are the corresponding eigenvectors. Because $\mK$ is symmetric positive semi-definite, this is an eigendecomposition (equivalently, an SVD with matching left and right singular vectors). The corresponding FIM eigenvectors in parameter space are given by 
$u_j = \mG^{\T} v_j/\sqrt{\lambda_j}$ for unit $v_j$ and $\lambda_j>0$, with Fisher eigenvalues $\lambda_j/N$, thereby avoiding the explicit construction of the $d \times d$ matrix. We provide a detailed proof of the eigenvalue correspondence and the recovery procedure in Appendix~\ref{app:gram}.

\paragraph{Adaptive subspace selection.}
The SVD provides a ranked list of eigenvectors, but the number of directions to protect is not fixed. To select the most relevant ones adaptively, we define the cumulative fraction of Fisher information captured by the top $k$ eigenvalues:
\begin{equation}
\rho(k) = \frac{\sum_{j=1}^{k}\lambda_j}{\sum_{j=1}^{N}\lambda_j}.
\label{eq:adaptive_k}
\end{equation}
We choose the smallest $k$ such that $\rho(k) \ge \tau$, and then assemble the physics-subspace basis $\mV_k = [v_1, \ldots, v_k] \in \mathbb{R}^{N \times k}$ from the corresponding top-$k$ eigenvectors. \footnote{The same threshold is used in all main experiments, please see Appendix~\ref{app:ablation_tau} for a detailed discussion.}

Subsequently, these directions are mapped back to the original parameter space via $\mU = \mathrm{normalise}(\mG^{\T}\mV_k) \in \mathbb{R}^{d \times k}$, where each column is normalized. The columns of $\mU$ span the critical physics subspace, while its orthogonal complement defines the safe subspace for fine-tuning. A detailed illustration is provided in Figure~\ref{fig:002-method}(a).

\subsection{Constrained Fine-Tuning via Orthogonal Projection}\label{sec:projection}

Having identified the physics-critical subspace from Section~\ref{sec:fisher}, we now describe how to exploit it during fine-tuning. The key idea is to constrain each gradient update to lie in the orthogonal complement of this subspace, so that adaptation to real data cannot overwrite the physics structure encoded by pretraining. Concretely, given a gradient $\textbf{\textit{g}} \in \mathbb{R}^{d}$ and the subspace basis $\mU$ for a parameter group, we project out its physics-critical component:
\begin{equation}
\textbf{\textit{g}}_{\text{proj}} = \textbf{\textit{g}} - \alpha \cdot \mU\!{\mU}^{\T} \textbf{\textit{g}}
\label{eq:projection}
\end{equation}
where $\alpha \in [0, 1]$ controls the strength of protection. $\alpha = 0$ recovers standard fine-tuning, while $\alpha = 1$ enforces full projection onto the null space of $\mU^{\T}$, removing updates along physics-critical directions. Values in between provide a softer trade-off, partially suppressing these directions while still allowing some movement along them.

In practice, each layer $m$ maintains its own subspace basis $\mU^{(m)}$, estimated independently from simulation data as described in Section~\ref{sec:fisher}. The projection in Equation~\eqref{eq:projection} is then applied per-layer after every backward pass. Figure~\ref{fig:002-method}(b) illustrates this procedure. The complete two-phase workflow (offline subspace estimation followed by iterative constrained fine-tuning) is formalized in Algorithm~\ref{alg:physguard}. Crucially, the subspace is computed only once and can be reused across multiple downstream tasks without recomputation. Further analysis of the projected updates is provided in Appendix~\ref{app:orthogonality}.

In our setting, spectral weights in architectures such as FNO are parameterized as complex numbers. To ensure compatibility with standard operations, we decompose each complex-valued gradient into its real and imaginary components. These are then used to construct the Fisher subspace and perform the projection. Further details are provided in Appendix~\ref{app:complex}.
\begin{algorithm}[t]
\caption{PhysGuard: Fisher-Guided Gradient Projection for Sim-to-Real Adaptation}
\label{alg:physguard}
\begin{algorithmic}[1]
\REQUIRE Pre-trained $\theta^*$, simulation data $\mathcal{D}_{\mathrm{sim}}$, real data $\mathcal{D}_{\mathrm{real}}$, samples $N$, threshold $\tau$, strength $\alpha$
\STATE \textbf{Phase 1: Subspace Estimation} (one-time)
\FOR{each layer $m = 1, \ldots, L$}
    \FOR{$i = 1, \ldots, N$}
        \STATE Sample $(x_i, y_i) \sim \mathcal{D}_{\mathrm{sim}}$;\; compute $\textbf{\textit{g}}_i^{(m)} = \nabla_{\theta^{(m)}} \ell\!\left(\theta^*;\, x_i, y_i\right)$
    \ENDFOR
    \STATE $\mG^{(m)} \leftarrow [\textbf{\textit{g}}_1^{(m)}, \ldots, \textbf{\textit{g}}_N^{(m)}]^{\T}$;\quad $\mK^{(m)} \leftarrow \mG^{(m)}{\mG^{(m)}}^{\T}$
    \STATE $\{\boldsymbol{V}^{(m)},\, \boldsymbol{\Lambda}^{(m)}\} \leftarrow \mathrm{SVD}\!\left(\mK^{(m)}\right)$;\; select $k_m$ via Equation~\eqref{eq:adaptive_k}
    \STATE $\mV_{k_m}^{(m)} \leftarrow [\boldsymbol{v}_1^{(m)}, \ldots, \boldsymbol{v}_{k_m}^{(m)}]$;\quad $\mU^{(m)} \leftarrow \mathrm{normalise}\!\left({\mG^{(m)}}^{\T}\, \mV_{k_m}^{(m)}\right)$
\ENDFOR
\STATE \textbf{Phase 2: Constrained Fine-Tuning} (iterative)
\FOR{each step $t = 1, \ldots, T$}
    \STATE Sample batch $\mathcal{B} \sim \mathcal{D}_{\mathrm{real}}$;\; forward \& backward to obtain $\textbf{\textit{g}}^{(m)}$ for all $m$
    \FOR{each layer $m$}
        \STATE $\textbf{\textit{g}}_{\mathrm{proj}}^{(m)} \leftarrow \textbf{\textit{g}}^{(m)} - \alpha\, \mU^{(m)}\!\left({\mU^{(m)}}^{\T} \textbf{\textit{g}}^{(m)}\right)$ \hfill $\triangleright$ \textit{Equation~\eqref{eq:projection}}
    \ENDFOR
    \STATE Update $\theta$ with optimiser using $\{\textbf{\textit{g}}_{\mathrm{proj}}^{(m)}\}$
\ENDFOR
\RETURN Adapted parameters $\theta_{\mathrm{adapted}}$
\end{algorithmic}
\end{algorithm}

\section{Experiments}
\label{sec:experiments}
In this section, we conduct experiments to address the following research questions:

\begin{itemize}[leftmargin=*, itemsep=0pt, topsep=0pt, parsep=0pt]
\item \textbf{RQ1}: Do the FIM eigenvectors align with low-frequency physics?
\item \textbf{RQ2}: Can PhysGuard improve sim-to-real performance across architectures and scenarios?
\item \textbf{RQ3}: Can PhysGuard preserve low-frequency physical structures during real-data adaptation?
\end{itemize}

\subsection{Experimental Setup}\label{sec:setup}

\textbf{Benchmark.}
We adopt RealPDEBench~\citep{hu2026realpdebench}, which pairs real experimental measurements with matched numerical simulations. We select three scenarios of increasing difficulty: \textbf{cylinder flow}, \textbf{controlled cylinder}, and \textbf{turbulent combustion}. Details are provided in Appendix~\ref{app:exp_details}.

\textbf{Architectures.}
We evaluate four neural operator architectures (3.5M to 50.4M parameters): \textbf{FNO}~\citep{li2021fourier} (Fourier-domain convolution), \textbf{CNO}~\citep{raonic2023convolutional} (multi-resolution CNN), \textbf{DeepONet}~\citep{lu2021learning} (branch-trunk factorization), and \textbf{Transolver}~\citep{wu2024transolver} (physics-aware attention). Configurations are in Appendix~\ref{app:exp_details}.

\textbf{Methods.}
We compare five approaches. \textbf{Pretrained} applies the simulation-trained model directly without adaptation. \textbf{DFT} (Direct Fine-Tuning) updates all parameters on real data without constraints. \textbf{L$_2$-SP}~\citep{li2018explicit} penalizes deviation from pretrained weights. \textbf{EWC}~\citep{kirkpatrick2017overcoming} applies a diagonal Fisher penalty to protect important parameters. \textbf{PhysGuard} projects gradients away from the physics-critical subspace ($\alpha = 1.0$). Additional comparisons with LoRA, layer freezing, and matched-rank random projection, along with ablations on $\alpha$ and $\tau$, are reported in Appendix~\ref{app:ablation}.

{\textbf{Training.} All methods share one fine-tuning budget with checkpoint selection based on the lowest validation RMSE. We report test-set metrics at the selected checkpoint, keeping the comparison focused on the update rule rather than per-method tuning. Protocol details and multi-seed results are in Appendices~\ref{app:exp_details} and~\ref{app:robustness}.

\textbf{Metrics.}
We report both data-oriented metrics (\textbf{RMSE}, \textbf{$R^2$}) and physics-oriented metrics (\textbf{fRMSE}, its \textbf{Low-$f$} band, and \textbf{relative low-frequency error}). Full metric definitions and additional results are provided in Appendix~\ref{app:metrics}.

\subsection{FIM Eigenvectors and Low-Frequency Physics (RQ1)}\label{sec:fisher_analysis}

\begin{figure}[t]
  \centering
  \includegraphics[width=\linewidth]{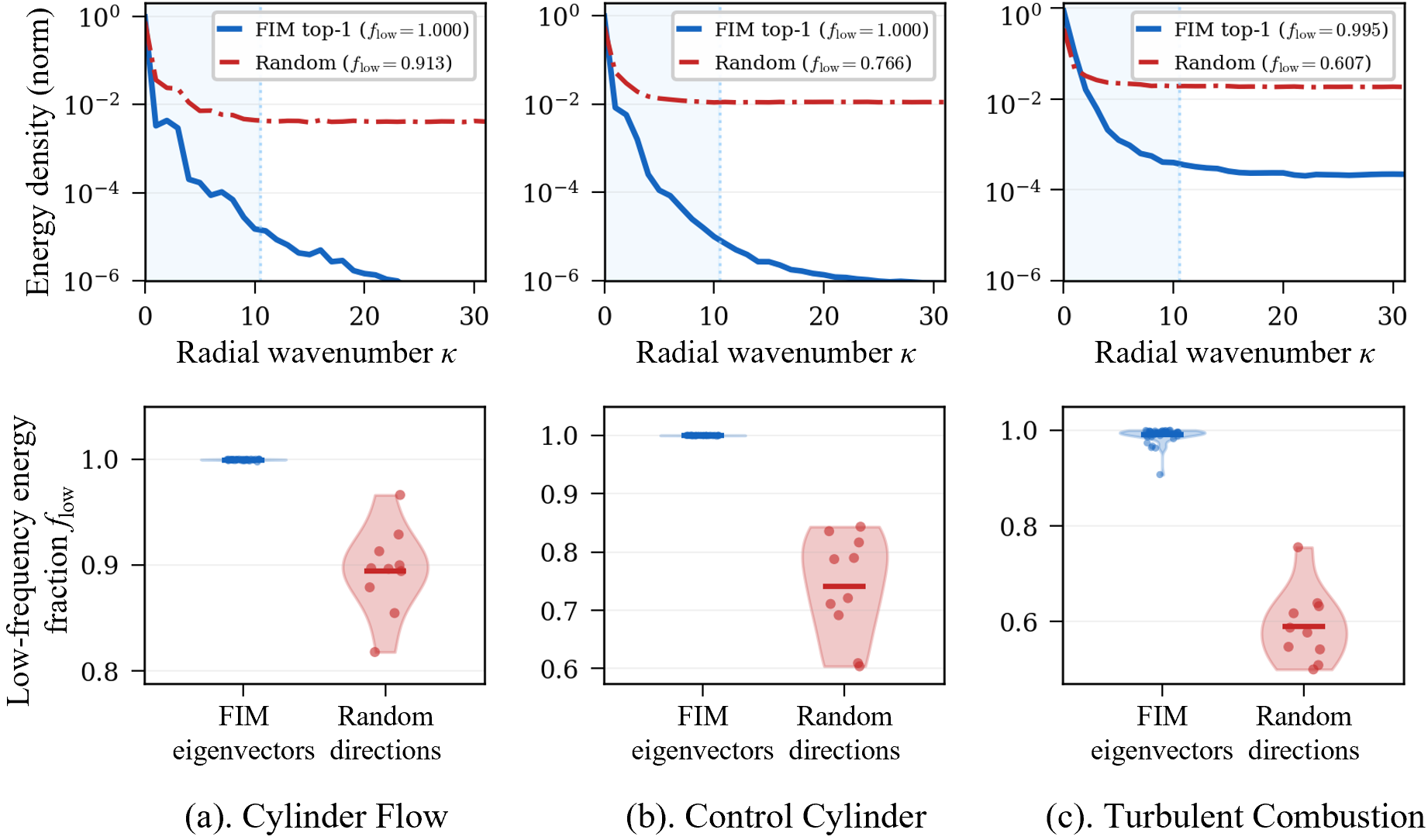}
  \caption{Diagnostic spectral probe for FIM eigenvectors. \textit{Top}: Radial energy spectrum of the output perturbation induced by a \textcolor[RGB]{31,119,180}{FIM top-1 eigenvector} vs.\ a \textcolor[RGB]{190,60,50}{random direction}; shaded region = low-frequency band. \textit{Bottom}: Low-frequency energy fraction $f_{\mathrm{low}}$ for the \textcolor[RGB]{31,119,180}{probed leading FIM eigenvectors} vs.\ \textcolor[RGB]{190,60,50}{random directions}. Each point is one direction. See Appendix~\ref{app:probe} for detailed definitions. Best viewed in color.}
  \label{fig:eigenspectrum}
\end{figure}

While the low intrinsic dimensionality of neural network parameter spaces has been established for language models and image classifiers~\citep{li2018measuring, aghajanyan2021intrinsic}, it remains unexamined for neural operators trained on physical simulations. We therefore first verify this assumption before evaluating transfer performance, testing whether leading Fisher directions affect low-frequency output structures.

We design a spectral probe to test this on a pretrained FNO (full protocol in Appendix~\ref{app:probe}). For a given parameter direction $v$, we slightly perturb the pretrained weights along $v$ and observe how the model output changes. We then apply a 2D Fourier transform to this output change and summarize it by the radial wavenumber $\kappa$, which characterizes spatial scale. Small $\kappa$ corresponds to large-scale patterns such as vortices, whereas large $\kappa$ captures fine-grained fluctuations. We define $f_{\mathrm{low}}$ as the fraction of total spectral energy in the low-frequency band, namely the lowest third of $\kappa$ modes. When $f_{\mathrm{low}} \approx 1.0$, the direction mainly affects large-scale structures. When it is small, the effect is distributed across higher frequencies.

We compare the leading Fisher directions from the diagnostic probe described in Appendix~\ref{app:probe} with random directions as a control. Figure~\ref{fig:eigenspectrum} (top) compares the FIM top-1 eigenvector with a random direction. Across all three scenarios, the FIM eigenvector produces perturbations concentrated almost entirely at low frequencies ($f_{\mathrm{low}} \geq 0.995$), while random directions spread energy more uniformly. Figure~\ref{fig:eigenspectrum} (bottom) extends this to all probed Fisher directions, confirming a consistent trend.

These results support PhysGuard's design by linking sensitive parameter directions in the pretrained model to low-frequency output structures. A similar pattern is observed in probes of CNO, DeepONet, and Transolver (Appendix~\ref{app:probe_arch}).

\subsection{Sim-to-Real Transfer Performance (RQ2)}\label{sec:main_results}

Table~\ref{tab:main} reports representative metrics across all architectures and scenarios. We summarize the main observations below.

\paragraph{Consistent improvement across architectures.}
PhysGuard achieves competitive performance in most metrics in Table~\ref{tab:main}.} All exceptions occur in Controlled Cylinder, where the sim-to-real gap is small, and all methods perform similarly. Multi-seed comparisons show that PhysGuard's gains over DFT on Cylinder Flow consistently exceed the standard deviation across seeds. On Controlled Cylinder and Turbulent Combustion, however, several smaller margins fall within one standard deviation and cannot be distinguished from seed variation (Appendix~\ref{app:robustness_seeds}).

The improvement holds across architectures with fundamentally different inductive biases, ranging from FNO's spectral convolution to DeepONet's branch-trunk factorisation and Transolver's physics-attention, suggesting that the FIM-based subspace captures properties of the learned physics shared by these architectures. Among the baselines, L$_2$-SP and EWC are less stable. L$_2$-SP underperforms DFT on several FNO metrics in the Controlled Cylinder scenario, indicating that a uniform $\ell_2$ penalty can over-restrict adaptation along directions that are not physics-critical. EWC performs close to DFT but rarely surpasses it by a large margin, consistent with its diagonal Fisher approximation not capturing the correlations among parameter directions that the full FIM eigenvectors encode. PhysGuard, by projecting gradients away from the physics-critical subspace rather than penalising individual weights, selectively blocks only the most sensitive directions while leaving the rest unconstrained.

\paragraph{DFT can degrade low-frequency structures.}
On Cylinder Flow, DFT raises the Low-$f$ error above the pretrained baseline for CNO. This means standard fine-tuning can overwrite large-scale physical patterns even when overall RMSE improves. PhysGuard avoids this regression and reduces Low-$f$ by 22\% to 32\% relative to DFT on FNO, CNO, and DeepONet.

\paragraph{Larger gains under larger domain shifts.}
PhysGuard's gains over DFT are most evident on Cylinder Flow, where pretrained $R^2$ values range from 0.34 to 0.72 and PhysGuard delivers the largest improvements. For DeepONet, PhysGuard achieves a relative $R^2$ improvement of over 50\%, compared to less than 20\% for DFT. On Transolver, all three baselines slightly degrade Cylinder Flow RMSE and $R^2$, with DFT even increasing RMSE, while PhysGuard still yields a relative $R^2$ gain of approximately 16\%.

This pattern is consistent with reduced interference along protected directions, although the metrics alone do not establish the gradient geometry. Matched-rank random controls further support Fisher-based direction selection (Appendix~\ref{app:orthogonality}). On Controlled Cylinder, inter-method differences are small. On Turbulent Combustion, an intermediate-shift scenario, PhysGuard remains effective but with narrower margins.

\begin{table}[t]
\centering
\caption{Statistical results across all RealPDEBench scenarios. The best and second-best results are highlighted in \colorbox[rgb]{.741,.843,.933}{\textbf{Blue}} and \colorbox[rgb]{.886,.937,.855}{Green}. Best viewed in color.}
\label{tab:main}
\scriptsize
\setlength{\tabcolsep}{1.8pt}
\begin{tabular}{ll cccc cccc cccc}
\toprule
\multirow{2}{*}{Architecture} & \multirow{2}{*}{Method} & \multicolumn{4}{c}{Cylinder Flow} & \multicolumn{4}{c}{Controlled Cylinder} & \multicolumn{4}{c}{Turbulent Combustion} \\
\cmidrule(lr){3-6} \cmidrule(lr){7-10} \cmidrule(lr){11-14}
& & RMSE\,$\downarrow$ & $R^2$\,$\uparrow$ & fRMSE\,$\downarrow$ & Low-$f$\,$\downarrow$ & RMSE\,$\downarrow$ & $R^2$\,$\uparrow$ & fRMSE\,$\downarrow$ & Low-$f$\,$\downarrow$ & RMSE\,$\downarrow$ & $R^2$\,$\uparrow$ & fRMSE\,$\downarrow$ & Low-$f$\,$\downarrow$ \\
\midrule
 \multirow{5}{*}{\shortstack{\textbf{FNO}\\ (50.4\,M)}} & Pretrained & 0.08060 & 0.6530 & 0.01266 & 0.01829 & 0.02788 & 0.8455 & 0.00566 & 0.00646 & 0.03626 & 0.2735 & 0.00507 & 0.00811 \\
  & DFT & 0.06487 & 0.7753 & 0.01048 & 0.01651 & \cellcolor[rgb]{ .741,  .843,  .933}\textbf{0.00997} & \cellcolor[rgb]{ .886,  .937,  .855}0.9802 & \cellcolor[rgb]{ .741,  .843,  .933}\textbf{0.00119} & \cellcolor[rgb]{ .741,  .843,  .933}\textbf{0.00106} & 0.02548 & 0.6414 & 0.00334 & 0.00419 \\
  & L$_2$-SP & 0.06556 & 0.7704 & 0.01075 & 0.01739 & \cellcolor[rgb]{ .886,  .937,  .855}0.01191 & 0.9718 & 0.00139 & 0.00126 & 0.02779 & 0.5732 & 0.00374 & 0.00494 \\
  & EWC & \cellcolor[rgb]{ .886,  .937,  .855}0.06316 & \cellcolor[rgb]{ .886,  .937,  .855}0.7869 & \cellcolor[rgb]{ .886,  .937,  .855}0.01031 & \cellcolor[rgb]{ .886,  .937,  .855}0.01596 & \cellcolor[rgb]{ .741,  .843,  .933}\textbf{0.00997} & \cellcolor[rgb]{ .741,  .843,  .933}\textbf{0.9803} & \cellcolor[rgb]{ .886,  .937,  .855}0.00120 & \cellcolor[rgb]{ .886,  .937,  .855}0.00108 & \cellcolor[rgb]{ .886,  .937,  .855}0.02544 & \cellcolor[rgb]{ .886,  .937,  .855}0.6424 & \cellcolor[rgb]{ .886,  .937,  .855}0.00333 & \cellcolor[rgb]{ .886,  .937,  .855}0.00417 \\
  & PhysGuard & \cellcolor[rgb]{ .741,  .843,  .933}\textbf{0.06036} & \cellcolor[rgb]{ .741,  .843,  .933}\textbf{0.8054} & \cellcolor[rgb]{ .741,  .843,  .933}\textbf{0.00914} & \cellcolor[rgb]{ .741,  .843,  .933}\textbf{0.01131} & \cellcolor[rgb]{ .741,  .843,  .933}\textbf{0.00997} & \cellcolor[rgb]{ .886,  .937,  .855}0.9802 & 0.00121 & 0.00110 & \cellcolor[rgb]{ .741,  .843,  .933}\textbf{0.02532} & \cellcolor[rgb]{ .741,  .843,  .933}\textbf{0.6458} & \cellcolor[rgb]{ .741,  .843,  .933}\textbf{0.00331} & \cellcolor[rgb]{ .741,  .843,  .933}\textbf{0.00414} \\
\midrule
 \multirow{5}{*}{\shortstack{\textbf{CNO}\\ (8.0\,M)}} & Pretrained & 0.07239 & 0.7201 & 0.01108 & 0.01379 & 0.01652 & 0.9458 & 0.00226 & 0.00181 & 0.04265 & $-$0.0049 & 0.00603 & 0.01057 \\
  & DFT & 0.04988 & 0.8671 & 0.00865 & 0.01585 & \cellcolor[rgb]{ .886,  .937,  .855}0.00885 & \cellcolor[rgb]{ .886,  .937,  .855}0.9844 & \cellcolor[rgb]{ .886,  .937,  .855}0.00105 & \cellcolor[rgb]{ .886,  .937,  .855}0.00091 & 0.02688 & 0.6008 & 0.00361 & 0.00483 \\
  & L$_2$-SP & 0.04915 & 0.8710 & 0.00855 & 0.01506 & 0.00936 & 0.9826 & 0.00113 & 0.00100 & 0.02685 & 0.6018 & 0.00360 & 0.00481 \\
  & EWC & \cellcolor[rgb]{ .886,  .937,  .855}0.04855 & \cellcolor[rgb]{ .886,  .937,  .855}0.8741 & \cellcolor[rgb]{ .886,  .937,  .855}0.00773 & \cellcolor[rgb]{ .886,  .937,  .855}0.01355 & 0.00886 & \cellcolor[rgb]{ .886,  .937,  .855}0.9844 & \cellcolor[rgb]{ .886,  .937,  .855}0.00105 & \cellcolor[rgb]{ .886,  .937,  .855}0.00091 & \cellcolor[rgb]{ .886,  .937,  .855}0.02670 & \cellcolor[rgb]{ .886,  .937,  .855}0.6060 & \cellcolor[rgb]{ .886,  .937,  .855}0.00358 & \cellcolor[rgb]{ .886,  .937,  .855}0.00476 \\
  & PhysGuard & \cellcolor[rgb]{ .741,  .843,  .933}\textbf{0.04635} & \cellcolor[rgb]{ .741,  .843,  .933}\textbf{0.8853} & \cellcolor[rgb]{ .741,  .843,  .933}\textbf{0.00759} & \cellcolor[rgb]{ .741,  .843,  .933}\textbf{0.01148} & \cellcolor[rgb]{ .741,  .843,  .933}\textbf{0.00884} & \cellcolor[rgb]{ .741,  .843,  .933}\textbf{0.9845} & \cellcolor[rgb]{ .741,  .843,  .933}\textbf{0.00104} & \cellcolor[rgb]{ .741,  .843,  .933}\textbf{0.00088} & \cellcolor[rgb]{ .741,  .843,  .933}\textbf{0.02648} & \cellcolor[rgb]{ .741,  .843,  .933}\textbf{0.6125} & \cellcolor[rgb]{ .741,  .843,  .933}\textbf{0.00352} & \cellcolor[rgb]{ .741,  .843,  .933}\textbf{0.00462} \\
\midrule
 \multirow{5}{*}{\shortstack{\textbf{DeepONet}\\ (3.5\,M)}} & Pretrained & 0.10072 & 0.4581 & 0.01612 & 0.02793 & 0.05065 & 0.4903 & 0.01239 & 0.01284 & 0.04036 & 0.0999 & 0.00562 & 0.00947 \\
  & DFT & 0.09309 & 0.5371 & 0.01507 & \cellcolor[rgb]{ .886,  .937,  .855}0.02527 & \cellcolor[rgb]{ .741,  .843,  .933}\textbf{0.03004} & \cellcolor[rgb]{ .741,  .843,  .933}\textbf{0.8207} & \cellcolor[rgb]{ .741,  .843,  .933}\textbf{0.00570} & \cellcolor[rgb]{ .886,  .937,  .855}0.00446 & 0.02757 & 0.5799 & 0.00369 & 0.00448 \\
  & L$_2$-SP & 0.09121 & 0.5556 & 0.01428 & 0.02627 & 0.03223 & 0.7936 & 0.00680 & 0.00653 & 0.02805 & 0.5653 & 0.00377 & 0.00464 \\
  & EWC & \cellcolor[rgb]{ .886,  .937,  .855}0.09055 & \cellcolor[rgb]{ .886,  .937,  .855}0.5621 & \cellcolor[rgb]{ .886,  .937,  .855}0.01424 & 0.02599 & \cellcolor[rgb]{ .886,  .937,  .855}0.03006 & \cellcolor[rgb]{ .886,  .937,  .855}0.8205 & \cellcolor[rgb]{ .741,  .843,  .933}\textbf{0.00570} & \cellcolor[rgb]{ .741,  .843,  .933}\textbf{0.00439} & \cellcolor[rgb]{ .886,  .937,  .855}0.02745 & \cellcolor[rgb]{ .886,  .937,  .855}0.5837 & \cellcolor[rgb]{ .886,  .937,  .855}0.00366 & \cellcolor[rgb]{ .886,  .937,  .855}0.00444 \\
  & PhysGuard & \cellcolor[rgb]{ .741,  .843,  .933}\textbf{0.07435} & \cellcolor[rgb]{ .741,  .843,  .933}\textbf{0.7048} & \cellcolor[rgb]{ .741,  .843,  .933}\textbf{0.01246} & \cellcolor[rgb]{ .741,  .843,  .933}\textbf{0.01957} & 0.03023 & 0.8184 & \cellcolor[rgb]{ .886,  .937,  .855}0.00573 & \cellcolor[rgb]{ .741,  .843,  .933}\textbf{0.00439} & \cellcolor[rgb]{ .741,  .843,  .933}\textbf{0.02743} & \cellcolor[rgb]{ .741,  .843,  .933}\textbf{0.5844} & \cellcolor[rgb]{ .741,  .843,  .933}\textbf{0.00365} & \cellcolor[rgb]{ .741,  .843,  .933}\textbf{0.00442} \\
\midrule
 \multirow{5}{*}{\shortstack{\textbf{Transolver}\\ (4.3\,M)}} & Pretrained & 0.11147 & 0.3364 & 0.01692 & 0.03362 & 0.02642 & 0.8613 & 0.00570 & 0.00703 & 0.04441 & $-$0.0898 & 0.00623 & 0.01125 \\
  & DFT & \cellcolor[rgb]{ .886,  .937,  .855}0.11236 & \cellcolor[rgb]{ .886,  .937,  .855}0.3257 & \cellcolor[rgb]{ .886,  .937,  .855}0.01640 & \cellcolor[rgb]{ .886,  .937,  .855}0.03349 & \cellcolor[rgb]{ .886,  .937,  .855}0.01898 & \cellcolor[rgb]{ .886,  .937,  .855}0.9284 & 0.00268 & 0.00252 & \cellcolor[rgb]{ .886,  .937,  .855}0.03995 & \cellcolor[rgb]{ .886,  .937,  .855}0.1181 & \cellcolor[rgb]{ .886,  .937,  .855}0.00534 & \cellcolor[rgb]{ .886,  .937,  .855}0.00833 \\
  & L$_2$-SP & 0.11361 & 0.3107 & 0.01651 & 0.03370 & 0.01920 & 0.9268 & 0.00276 & 0.00265 & 0.04087 & 0.0773 & 0.00541 & 0.00864 \\
  & EWC & 0.11330 & 0.3144 & 0.01648 & 0.03361 & \cellcolor[rgb]{ .886,  .937,  .855}0.01898 & \cellcolor[rgb]{ .886,  .937,  .855}0.9284 & \cellcolor[rgb]{ .886,  .937,  .855}0.00266 & \cellcolor[rgb]{ .886,  .937,  .855}0.00246 & 0.04070 & 0.0845 & 0.00538 & 0.00854 \\
  & PhysGuard & \cellcolor[rgb]{ .741,  .843,  .933}\textbf{0.10695} & \cellcolor[rgb]{ .741,  .843,  .933}\textbf{0.3891} & \cellcolor[rgb]{ .741,  .843,  .933}\textbf{0.01567} & \cellcolor[rgb]{ .741,  .843,  .933}\textbf{0.03097} & \cellcolor[rgb]{ .741,  .843,  .933}\textbf{0.01852} & \cellcolor[rgb]{ .741,  .843,  .933}\textbf{0.9318} & \cellcolor[rgb]{ .741,  .843,  .933}\textbf{0.00254} & \cellcolor[rgb]{ .741,  .843,  .933}\textbf{0.00235} & \cellcolor[rgb]{ .741,  .843,  .933}\textbf{0.03841} & \cellcolor[rgb]{ .741,  .843,  .933}\textbf{0.1848} & \cellcolor[rgb]{ .741,  .843,  .933}\textbf{0.00512} & \cellcolor[rgb]{ .741,  .843,  .933}\textbf{0.00758} \\
\bottomrule
\end{tabular}
\end{table}

\subsection{Low-Frequency Preservation (RQ3)}\label{sec:qualitative}
\paragraph{Qualitative comparison.}
\begin{figure}[t]
  \centering
  \includegraphics[width=\linewidth]{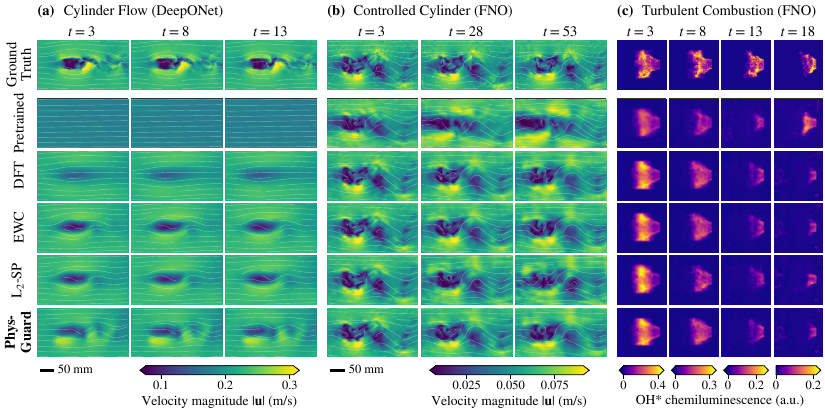}
  \caption{Predicted fields on three scenarios. (a)~Cylinder Flow (DeepONet), (b)~Controlled Cylinder (FNO), (c)~Turbulent Combustion (FNO). Each panel shows Ground Truth, Pretrained, DFT, EWC, L$_2$-SP, and PhysGuard at selected time steps. Full per-architecture results are in Appendix~\ref{app:visualization}. Best viewed in color.}
  \label{fig:qualitative}
\end{figure}

Figure~\ref{fig:qualitative} compares predicted flow fields.
In panel~(a), the ground truth shows alternating vortex cores (K\'{a}rm\'{a}n vortex street). The pretrained DeepONet produces a nearly uniform field. DFT and EWC recover some structure but miss the sharp vortex boundaries. PhysGuard produces the closest match to ground truth, consistent with its lower Low-$f$ error in Table~\ref{tab:main}.
In panel~(b), all methods track the ground truth well because the sim-to-real gap is small on this scenario.
In panel~(c), the combustion field has localized high-intensity zones. Pretrained predictions are blurred. All baselines improve substantially, with PhysGuard producing slightly sharper boundaries.

\paragraph{Low-frequency error across architectures.}\label{sec:spectral}

\begin{figure}[t]
  \centering
  \includegraphics[width=\linewidth]{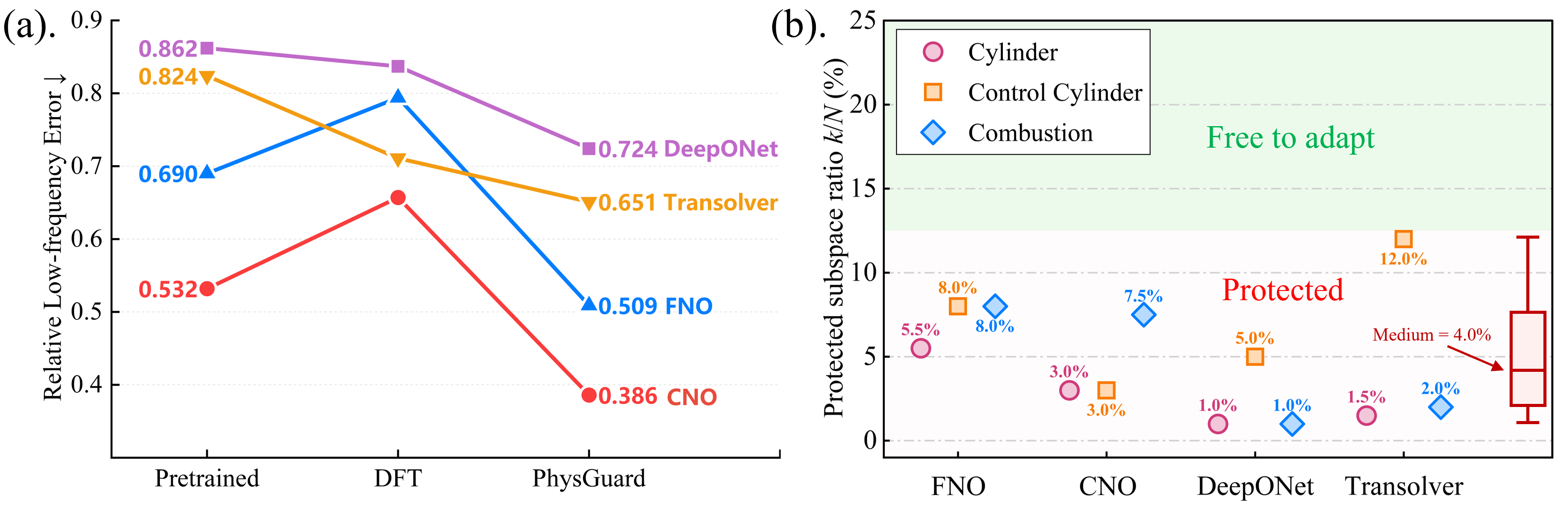}
  \caption{(a) Relative low-frequency error ($\downarrow$) on Cylinder Flow for the pretrained model, DFT, and PhysGuard. (b) Protected subspace ratio $k/N$ (\%) per architecture and scenario. Here $N$ is the number of sampled gradients and the green shading is schematic. Best viewed in color.}
  \label{fig:spectrum}
\end{figure}

Figure~\ref{fig:spectrum}(a) compares relative low-frequency error on Cylinder Flow. This metric normalizes low-frequency spectral errors by the corresponding reference amplitudes, providing a dimensionless measure of large-scale flow fidelity that complements overall accuracy metrics (Appendix~\ref{app:metrics}). PhysGuard reduces this error relative to the pretrained model by 26.1\% (FNO), 27.4\% (CNO), 21.0\% (Transolver), and 16.0\% (DeepONet). Table~\ref{tab:relative_low_frequency} reports results for all three scenarios.
Figure~\ref{fig:spectrum}(b) shows the protected subspace ratio $k/N$. Across all 12 architecture-scenario combinations, this ratio ranges from 1.0\% to 12.0\%, with a median of 4.0\%. Because $N$ bounds the rank of the estimated FIM, $k/N$ is the fraction of candidate directions that PhysGuard protects. As a fraction of the full parameter space, the protected subspace is much smaller ($k/d$, since $N \ll d$).

\section{Limitations}
\label{sec:limitations}

PhysGuard is designed for settings where the FIM spectrum exhibits a low-rank structure, as a compact subspace can then efficiently capture the directions most critical to physical knowledge. This structure is observed across the four architectures in our main experiments, including FNO, CNO, DeepONet, and Transolver. We note that PhysGuard may be less effective for certain pretrained operators (e.g., DPOT~\citep{hao2024dpot}) whose Fisher spectrum does not exhibit a clear low-rank pattern. We report this result in Appendix~\ref{app:dpot}. The four-stage analysis there shows spectrum flattening during scenario-specific training, which cannot be attributed solely to joint pretraining or the backbone. The leading Fisher directions must also correspond to stable signal structures; whether this holds in other signal domains remains untested.

Beyond this dependence on the estimated spectrum, our evaluation is limited to fluid-mechanics scenarios from RealPDEBench~\citep{hu2026realpdebench}. To our knowledge, this is currently the only benchmark in the community that provides results on real experimental data, which is why we focus our evaluation here. That said, whether PhysGuard generalises to other PDE families and sensing modalities remains an open question, and we hope this work encourages the SciML community to develop broader real-world benchmarks for evaluating neural PDE solvers.

\section{Conclusion}
\label{sec:conclusion}

We presented PhysGuard, a framework that preserves pretrained physical knowledge during sim-to-real adaptation of neural PDE surrogates through Fisher-guided gradient projection. The key idea is to project fine-tuning gradients onto the orthogonal complement of the physics-critical subspace identified by the empirical Fisher Information Matrix, which reduces low-frequency degradation in the tested settings without requiring additional penalty terms. Experiments on RealPDEBench across four architectures and three physical scenarios confirm consistent improvements, particularly in low-frequency fidelity and under large domain shifts. That said, challenges remain in scaling to foundation-model-scale operators and updating the protected subspace during adaptation, both of which we see as promising directions for future work.

We have recently noted that graph neural simulators offer another promising approach to modelling complex physical systems~\citep{sanchez_gonzalez2020gns,pfaff2021meshgraphnets,zhou2026trace}. An interesting extension of this work would be to explore how these simulators can retain useful prior physical knowledge while learning new dynamics across different physical scenarios.

\begin{ack}
The first author acknowledges financial support from the Melbourne Research Scholarship (MRS) awarded by the University of Melbourne. This work was also supported in part by the Australian Department of Education under grant AEA IV240100185.
\end{ack}

\bibliographystyle{unsrtnat}
\bibliography{references}

\newpage
\appendix
\renewcommand{\thefigure}{\Alph{section}.\arabic{figure}}
\counterwithin{figure}{section}
\renewcommand{\thetable}{\Alph{section}.\arabic{table}}
\counterwithin{table}{section}
\setcounter{equation}{0}
\renewcommand{\theequation}{\Alph{section}.\arabic{equation}}
\counterwithin{equation}{section}
\raggedbottom   

\etocdepthtag.toc{appendix}
\etocsettagdepth{maintext}{none}
\etocsettagdepth{appendix}{subsection}
\makeatletter
\etocsetstyle{section}{}{}
  {\addpenalty\@secpenalty\addvspace{0.55em \@plus\p@}%
   \@dottedtocline{1}{0em}{1.6em}%
     {\bfseries\hb@xt@\@tempdima{\etocnumber\hfil}\etocname}{\bfseries\etocpage}}
  {}
\etocsetstyle{subsection}{}{}
  {\@dottedtocline{2}{1.6em}{2.4em}%
     {\hb@xt@\@tempdima{\etocnumber\hfil}\etocname}{\etocpage}}
  {}
\makeatother
\renewcommand{\contentsname}{Appendix Contents}
{\small\setlength{\parskip}{0pt}\hypersetup{linkcolor=blue}\tableofcontents}
\newpage


\section{Broader Impacts}
\label{app:broader_impact}

PhysGuard studies how pretrained neural PDE surrogates can be adapted to real measurements while limiting changes to simulation-sensitive directions. Reusing pretrained models may reduce the need to train a separate surrogate for each experimental setting. The present results concern three fluid-mechanics scenarios and do not establish reliability in other application domains.

Residual sim-to-real error and distribution shifts remain relevant when surrogate predictions inform engineering or environmental decisions. Such use requires physics-based validation, uncertainty assessment, and domain-expert review beyond the comparisons reported here.

\section{Extended Related Work}
\label{app:related_ext}

This section expands the related work by reviewing transfer learning methods for neural operators, discussing the potential of continual learning approaches for operator architectures, and comparing ORTHOG-SUBSPACE with PhysGuard.

\subsection{Neural Operator Transfer and Adaptation}
\label{app:related_ext_no}

Table~\ref{tab:no_transfer} summarises eight representative transfer and adaptation methods. PhysGuard uses the same Fisher-based parameter-space construction across the four evaluated operator families.

\begin{table}[H]
\centering
\caption{Transfer and adaptation methods for neural PDE operators.}
\label{tab:no_transfer}
\scriptsize
\setlength{\tabcolsep}{4pt}
\renewcommand{\arraystretch}{1.15}
\begin{tabularx}{\textwidth}{>{\hsize=0.72\hsize\raggedright\arraybackslash}X >{\hsize=0.80\hsize\raggedright\arraybackslash}X >{\hsize=1.18\hsize\raggedright\arraybackslash}X >{\hsize=1.30\hsize\raggedright\arraybackslash}X}
\toprule
\textbf{Work} & \textbf{Backbone} & \textbf{Transfer or adaptation setting} & \textbf{Adaptation mechanism} \\
\midrule
TL-DeepONet~\citep{goswami2022deep} & DeepONet & Conditional shifts in PDE dynamics or geometry & Fine-tunes task-specific layers using supervised and distribution-matching objectives \\
PINO~\citep{li2024pino} & FNO & Instance-wise adaptation within a PDE family & Test-time optimisation using the PDE residual \\
Subramanian et al.~\citep{subramanian2023towards} & FNO & Transfer across PDEs and out-of-distribution physical parameters & Mixed-physics pretraining followed by full-model fine-tuning \\
F-Adapter~\citep{zhang2025fadapter} & Scaled FNO & Adaptation of pretrained large operator models & Frequency-adaptive, architecture-specific adapters \\
Fan and Pang~\citep{fan2025fewshot} & CNO & Few-shot transfer across several PDE systems & Compares full fine-tuning, LoRA, and neuron linear transformation \\
Poseidon~\citep{herde2024poseidon} & Multiscale operator transformer & Fluid-dynamics pretraining followed by 15 downstream PDE tasks & Large-scale pretraining followed by supervised fine-tuning \\
DPOT~\citep{hao2024dpot} & Fourier-attention operator transformer & Pretraining on more than ten PDE datasets followed by downstream adaptation & Autoregressive denoising pretraining followed by fine-tuning \\
CoDA-NO~\citep{rahman2024codano} & Codomain-attention operator & Transfer from pretrained PDE systems to few-shot multiphysics tasks & Self-supervised pretraining followed by downstream fine-tuning \\
\midrule
PhysGuard (ours) & FNO, CNO, DeepONet, and transformer operators & Simulation-to-real adaptation across different PDE scenarios & Identifies and protects simulation-sensitive Fisher directions using the same parameter-space mechanism across all backbones \\
\bottomrule
\end{tabularx}
\end{table}

Further architectural descriptions are given in Appendix~\ref{app:operators}.

\subsection{Continual Learning Methods}
\label{app:related_ext_cl}

Preserving simulation knowledge while adapting to real data is a form of the forgetting problem studied in continual learning. Table~\ref{tab:cl_families} summarises four relevant families.

\begin{table}[H]
\centering
\caption{Continual-learning methods for sim-to-real adaptation.}
\label{tab:cl_families}
\scriptsize
\setlength{\tabcolsep}{4pt}
\renewcommand{\arraystretch}{1.15}
\begin{tabularx}{\textwidth}{>{\hsize=0.55\hsize\raggedright\arraybackslash}X >{\hsize=1.25\hsize\raggedright\arraybackslash}X >{\hsize=1.10\hsize\raggedright\arraybackslash}X >{\hsize=1.10\hsize\raggedright\arraybackslash}X}
\toprule
\textbf{Family} & \textbf{Representative methods} & \textbf{Relevance to our setting} & \textbf{Main limitation} \\
\midrule
Regularisation & EWC~\citep{kirkpatrick2017overcoming}, SI~\citep{zenke2017continual}, MAS~\citep{aljundi2018memory}, L$_2$-SP~\citep{li2018explicit} & Directly applicable; EWC and L$_2$-SP are evaluated in Table~\ref{tab:main} & Diagonal importance misses correlated parameter directions \\
Projection & OGD~\citep{farajtabar2020orthogonal}, ORTHOG-SUBSPACE~\citep{chaudhry2020continual}, GPM~\citep{saha2021gradient}, Adam-NSCL~\citep{wang2021training} & Most closely related to PhysGuard & Subspaces built from task sequences, gradients, or features often need architecture-specific construction \\
Replay & iCaRL~\citep{rebuffi2017icarl}, ER~\citep{chaudhry2019tiny}, A-GEM~\citep{chaudhry2019efficient}, DER++~\citep{buzzega2020dark} & Simulation data could be replayed during adaptation & Adds training cost and may retain simulator biases along with useful physics \\
Parameter isolation & Progressive Networks~\citep{rusu2016progressive}, PackNet~\citep{mallya2018packnet}, HAT~\citep{serra2018overcoming} & Potentially useful when one operator must serve several real domains & Needs task identities, extra capacity, or parameter partitioning \\
\bottomrule
\end{tabularx}
\end{table}

EWC and L$_2$-SP are directly applicable and serve as baselines. PhysGuard instead protects correlated parameter directions identified by the empirical Fisher. Replay could complement this protection at additional training cost, while parameter isolation may be useful when one operator serves several domains.

\subsection{Comparison with ORTHOG-SUBSPACE}
\label{app:related_ext_orthog}

ORTHOG-SUBSPACE~\citep{chaudhry2020continual} assigns low-rank orthogonal subspaces to tasks. PhysGuard estimates its protected directions from a converged simulation model. Table~\ref{tab:orthog_compare} compares their subspace construction and use.

\begin{table}[H]
\centering
\caption{ORTHOG-SUBSPACE versus PhysGuard.}
\label{tab:orthog_compare}
\scriptsize
\setlength{\tabcolsep}{4pt}
\renewcommand{\arraystretch}{1.15}
\begin{tabularx}{\textwidth}{>{\hsize=0.60\hsize\raggedright\arraybackslash}X >{\hsize=1.15\hsize\raggedright\arraybackslash}X >{\hsize=1.25\hsize\raggedright\arraybackslash}X}
\toprule
\textbf{Aspect} & \textbf{ORTHOG-SUBSPACE} & \textbf{PhysGuard} \\
\midrule
Subspace construction & Predefined before training, data-independent & Estimated from the dominant eigenvectors of the empirical Fisher \\
Subspace interpretation & Allocates orthogonal capacity to different tasks & Identifies directions most sensitive to the pretrained simulation behaviour \\
Protection target & Accuracy on previously learned tasks & Physical behaviour retained from simulation pretraining \\
Application setting & Constrained training across a sequence of tasks & Post hoc adaptation of a converged operator, no retraining or architecture changes \\
\bottomrule
\end{tabularx}
\end{table}

\newpage

\section{Implementation Details}
\label{app:implementation}

\subsection{Architectures}
\label{app:operators}

The four neural operator architectures used in the main experiments aim to learn a mapping from an input field to an output field (e.g., from initial conditions to future states). The key difference lies in how this mapping is represented and computed.

\paragraph{Fourier Neural Operator (FNO)~\citep{li2021fourier}.}
FNO performs convolution in frequency space. Each layer lifts the input to a multi-channel feature map, applies a Fast Fourier Transform (FFT), retains the lowest-frequency modes, multiplies them by learned complex weights, and maps the result back to physical space via an inverse FFT. A pointwise linear bypass is added in parallel,

\begin{equation}
  v_{\ell+1}(\mathbf{x})
  = \sigma\!\Bigl(W_\ell\, v_\ell(\mathbf{x})
    + \mathcal{F}^{-1}\!\bigl[R_\ell \cdot \mathcal{F}[v_\ell]\bigr](\mathbf{x})\Bigr),
\end{equation}
where $R_\ell$ is the learned spectral filter and $W_\ell$ is the pointwise map. The implicit low-frequency bias makes FNO well suited to smooth, large-scale flow phenomena. In our experiments we use 4 layers, channel width 64, and mode truncations $(4, 12, 16)$ for the cylinder scenario and $(4, 16, 16)$ for combustion, totalling 50.4M parameters. Because FNO's spectral weights are complex-valued (\texttt{torch.complex64}), a real-embedding trick is needed during gradient projection; see Appendix~\ref{app:complex}.

\paragraph{Convolutional Neural Operator (CNO)~\citep{raonic2023convolutional}.}
CNO uses anti-aliased convolutional filters whose aliasing error provably vanishes as the grid is refined, guaranteeing convergence in the $L^2$ norm under resolution refinement. Architecturally, CNO follows a U-Net-style hierarchy: the input is progressively downsampled (by $2\times$ at each stage) through anti-aliased convolutions, then symmetrically upsampled, with skip connections linking matching resolution levels. Spectral normalisation and BatchNorm are applied at each resolution to promote stable Lipschitz behaviour across scales. We use 3 encoder/decoder levels, a channel multiplier of 32, and BatchNorm, giving 8.0M parameters.

\paragraph{Deep Operator Network (DeepONet)~\citep{lu2021learning}.}
DeepONet builds on the universal approximation theorem for operators and trains two separate networks in parallel: a branch net that encodes the input function sampled at fixed sensor locations, and a trunk net that encodes the query location at which the output is to be evaluated. Their outputs are combined via a dot product:
\begin{equation}
  \mathcal{G}^\dagger(u)(\mathbf{y})
  \approx \sum_{k=1}^{p} b_k(u)\, t_k(\mathbf{y}) + \text{bias},
\end{equation}
where $b_k$ and $t_k$ are the $k$-th outputs of the branch and trunk nets, respectively. We use $p = 128$ output dimensions ($p = 256$ on Controlled Cylinder, Table~\ref{tab:pretrain_configs}), 6-layer GELU MLPs for both networks, and dropout 0.1, yielding 3.5M parameters, the smallest model in our benchmark.

\paragraph{Transolver~\citep{wu2024transolver}.}
Standard Transformer attention has a cost of $O(n^2)$ in the number of grid points $n$.
Transolver reduces this cost using \emph{physics-attention}. It compresses the spatial domain
into $S$ learned \emph{slice tokens}, each aggregating information from a local neighbourhood
of physical points through soft assignment weights. Multi-head attention among these tokens
costs $O(S^2)$ and allows long-range information exchange. A final scatter step maps the
slice representations back to the original grid.

In our 3D combustion setting we use $S = 16$ slices, 8 attention heads, 1 Transformer layer, and
hidden dimension 256, yielding 4.3M parameters. The shallow depth (single layer) is imposed by the
memory and compute constraints of the high-dimensional combustion grid
($64 \times 64 \times 20$); its effect on representational capacity is discussed in
Section~\ref{sec:main_results}.

\paragraph{Architectural comparison.}
Table~\ref{tab:arch_comparison} provides a side-by-side summary of the four architectures across
the dimensions most relevant to sim-to-real transfer.

\begin{table}[t]
\centering
\caption{Main neural operator architectures.}
\label{tab:arch_comparison}
\renewcommand{\arraystretch}{1.3}
\resizebox{\textwidth}{!}{%
\setlength{\tabcolsep}{6pt}
\begin{tabular}{l r >{\raggedright\arraybackslash}p{1.8cm} >{\raggedright\arraybackslash}p{2.8cm} c c >{\raggedright\arraybackslash}p{2.2cm} c}
\toprule
\textbf{Architecture}
  & \textbf{Params}
  & \textbf{Core Operation}
  & \textbf{Inductive Bias}
  & \textbf{Complexity}\textsuperscript{\emph{a}}
  & \textbf{D.I.}\textsuperscript{\emph{b}}
  & \textbf{Config (this work)}
  & \textbf{Fisher}\textsuperscript{\emph{c}} \\
\midrule
\addlinespace[2pt]
FNO~\citep{li2021fourier}
  & 50.4\,M
  & Spectral (Fourier) convolution
  & Convolution in frequency domain; high-mode truncation enforces a low-frequency bias
  & $O(n\log n)$
  & \checkmark
  & 4 layers, width 64, modes $(4,12,16)$\textsuperscript{\emph{d}}
  & \textbf{Low-rank}\textsuperscript{1} \\
\addlinespace[2pt]
CNO~\citep{raonic2023convolutional}
  & 8.0\,M
  & Anti-aliased multi-scale convolution
  & Alias-free $L^2$-stable filters; U-Net hierarchy ensures convergence under resolution refinement
  & $O(n)$
  & \checkmark
  & 3 enc/dec levels, mult.\,32, BatchNorm
  & \textbf{Moderate}\textsuperscript{2} \\
\addlinespace[2pt]
DeepONet~\citep{lu2021learning}
  & 3.5\,M
  & Branch--trunk dot-product sum
  & Factored basis-function approximation; trunk net queries arbitrary output locations
  & $O(n_s {+} n_q)$
  & \checkmark
  & $p{=}128$, 6-layer MLP, dropout 0.1
  & \textbf{Broad}\textsuperscript{3} \\
\addlinespace[2pt]
Transolver~\citep{wu2024transolver}
  & 4.3\,M
  & Physics-attention over learned slice tokens
  & $S$ learned spatial slices enable global interaction at $O(S^2)$ cost, linear in DOF
  & $O(S^2)$
  & \checkmark
  & 1 layer, dim 256, 8 heads, $S{=}16$
  & \textbf{Dense}\textsuperscript{4} \\
\addlinespace[2pt]
\midrule
\multicolumn{8}{l}{%
\begin{minipage}{0.97\textwidth}
\footnotesize\raggedright
\textsuperscript{\emph{a}}~Dominant cost per forward pass; $n$: spatial degrees of freedom (DOF), $S$: number of Transolver slices.\\
\textsuperscript{\emph{b}}~D.I. = Discretisation-invariant: the model can be evaluated at a different spatial resolution than it was trained on.\\
\textsuperscript{\emph{c}}~Fisher structure describes how simulation-critical Fisher information is concentrated across parameters:\\
\hspace*{1em}\textsuperscript{1}~\textbf{Low-rank}: dominated by a few spectral weight directions. \\
\hspace*{1em}\textsuperscript{2}~\textbf{Moderate}: spread across multiple resolution scales. \\
\hspace*{1em}\textsuperscript{3}~\textbf{Broad}: spread over both branch and trunk MLP weights rather than concentrated in one module, yet still captured by a few directions per layer (1--12\% of $N$ at $\tau=0.9$, Figure~\ref{fig:spectrum}(b)).\\
\hspace*{1em}\textsuperscript{4}~\textbf{Dense}: concentrated but shallow; few parameters per head due to single-layer design.\\
\textsuperscript{\emph{d}}~Modes $(4,16,16)$ for the combustion scenario.
\end{minipage}}\\
\bottomrule
\end{tabular}
}
\end{table}

\subsection{Evaluation Metrics}
\label{app:metrics}

We evaluate model performance from two perspectives: (1) \textbf{prediction accuracy}, and (2) \textbf{physical consistency}.

\subsubsection{Data-Oriented Metrics}

\paragraph{Root mean square error (RMSE).}
RMSE is the square root of the mean squared difference between predictions and ground truth:

\begin{equation}
\mathrm{RMSE} =
\sqrt{
\frac{1}{n}
\sum (\hat{u} - u)^2
},
\end{equation}
where $n$ is the total number of elements.

\paragraph{Relative $L_2$ error.}
\begin{equation}
\mathrm{Rel}\,L_2 =
\frac{1}{B}
\sum_{i=1}^{B}
\frac{\|\hat{u}_i - u_i\|_2}{\|u_i\|_2},
\end{equation}
where $B$ is the number of test samples. This normalises the error by the magnitude of the true field.

\paragraph{$R^2$ score.}
\begin{equation}
R^2 =
1 - \frac{\sum (\hat{u} - u)^2}{\sum (u - \bar{u})^2},
\end{equation}
where $\bar{u}$ is the mean value.
A higher $R^2$ indicates better predictions.

\subsubsection{Physics-Oriented Metrics}

\paragraph{Spectral error (fRMSE).}
To evaluate whether the model captures structures at different scales, we transform the data into the frequency domain and compare errors there.

\begin{equation}
\mathrm{fRMSE} =
\sqrt{
\frac{1}{K}
\sum_{\kappa=0}^{K-1} E_\kappa
},
\end{equation}

where $E_\kappa$ measures the error at frequency level $\kappa$.

We also report errors separately for low-, mid-, and high-frequency bands. These bands capture large-scale structures and dominant physical modes, intermediate-scale features, and small-scale fluctuations, respectively. The bandwise errors indicate which scales are affected by fine-tuning.

\paragraph{Relative low-frequency error.}
Following the RealPDEBench evaluation implementation~\citep{hu2026realpdebench}, we compute
\begin{equation}
E_{\mathrm{rel,low}} =
\frac{1}{C\lvert\mathcal{K}_{\mathrm{low}}\rvert}
\sum_{c=1}^{C}\sum_{\kappa\in\mathcal{K}_{\mathrm{low}}}
\frac{\lVert\mathbf{U}^{\mathrm{pred}}_{\kappa,c}-\mathbf{U}^{\mathrm{ref}}_{\kappa,c}\rVert_2}
{\lVert \mathbf{U}^{\mathrm{ref}}_{\kappa,c}\rVert_2}.
\label{eq:relative_low_frequency}
\end{equation}
Here $\mathbf{U}^{\mathrm{ref}}_{\kappa,c}$ and $\mathbf{U}^{\mathrm{pred}}_{\kappa,c}$ collect the reference and predicted Fourier coefficients in radial frequency bin $\kappa$ and channel $c$ across an evaluation batch. The transform spans time and the spatial dimensions; $\mathcal{K}_{\mathrm{low}}$ contains the lowest third of retained bins, and $C$ is the number of channels. We average the metric across evaluation batches. Normalization by the corresponding reference amplitudes gives a dimensionless measure of low-frequency fidelity; lower values are better.

\subsection{Experimental Setup}
\label{app:exp_details}

\paragraph{Datasets.}
We use three scenarios from RealPDEBench, covering fluid flow and combustion.
Each dataset is split into training, validation, and test sets using official splits to avoid data leakage.

\paragraph{Pretraining.}
We retrain the four main architectures on the numerical training split using AdamW and a cosine schedule. The official RealPDEBench settings serve as the reference, with batch sizes and learning rates adjusted for our hardware. Table~\ref{tab:pretrain_configs} gives the actual settings. DPOT-S instead starts from its official checkpoint (Appendix~\ref{app:dpot_pretrain}).

\begin{table}[h]
\centering
\definecolor{PhHdr}{HTML}{FFFFFF}   
\definecolor{PhScn}{HTML}{FFFFFF}   
\definecolor{PhAlt}{HTML}{FFFFFF}   
\definecolor{PhRule}{HTML}{000000}   
\caption{Pretraining configurations across architectures and scenarios.}
\label{tab:pretrain_configs}
\renewcommand{\arraystretch}{1.35}
\resizebox{\textwidth}{!}{%
\setlength{\tabcolsep}{7pt}
\begin{tabular}{l r l r c c r >{\footnotesize\raggedright\arraybackslash}p{8.6cm}}
\specialrule{1.6pt}{0pt}{0pt}
\rowcolor{PhHdr}
\textbf{\color{black}Architecture} &
\textbf{\color{black}Params} &
\textbf{\color{black}LR} &
\textbf{\color{black}Iters} &
\textbf{\color{black}Batch/GPU} &
\textbf{\color{black}GPUs} &
\textbf{\color{black}Eff.\,batch} &
\textbf{\color{black}Architecture hyperparameters} \\[-1pt]
\arrayrulecolor{PhRule}\specialrule{0.45pt}{0pt}{0pt}
\multicolumn{8}{>{\columncolor{PhScn}}l}{%
  \rule{0pt}{10pt}%
  \textbf{\small\color{black}Scenario 1\enspace·\enspace Cylinder flow}%
  \quad{\color{black}\footnotesize grid $128\!\times\!64$,\enspace
    fields $u,v,p$,\enspace horizon $T\!=\!20$,\enspace
    ${\approx}5{,}000$ real train samples}%
  \rule[-3pt]{0pt}{3pt}} \\
 
FNO & 50.4\,M & $1\!\times\!10^{-4}$ &  4{,}000 & 16 & 2 & 32 & \texttt{n\_layers}=4, \texttt{width}=64, modes $(f_x,f_y,f_t)=(4,12,16)$ \\
 
\rowcolor{PhAlt}
CNO &  8.0\,M & $3\!\times\!10^{-4}$ &  5{,}000 &  4 & 4 & 16 & \texttt{N\_layers}=3, \texttt{channel\_multiplier}=32, \texttt{N\_res}=1, \texttt{N\_res\_neck}=8, BatchNorm, \texttt{latent\_lift\_proj\_dim}=64, LeakyReLU \\
 
DeepONet &  3.5\,M & $1\!\times\!10^{-4}$ &  5{,}000 &  4 & 4 & 16 & $p=128$, \texttt{dropout}=0.1 \\
 
\rowcolor{PhAlt}
Transolver &  4.3\,M & $1.80\!\times\!10^{-4}$ & 5{,}000 & 4 & 1 & 4 & \texttt{n\_layers}=1, \texttt{n\_hidden}=256, \texttt{n\_head}=8, \texttt{slice\_num}=16,\enspace grid $H\!\times\!W\!\times\!D=128\!\times\!64\!\times\!20$ \\[1pt]

\arrayrulecolor{PhRule}\specialrule{0.45pt}{0pt}{0pt}
 
\multicolumn{8}{>{\columncolor{PhScn}}l}{%
  \rule{0pt}{10pt}%
  \textbf{\small\color{black}Scenario 2\enspace·\enspace Controlled cylinder}%
  \quad{\color{black}\footnotesize grid $128\!\times\!64$,\enspace
    fields $u,v,p$,\enspace horizon $T\!=\!20$,\enspace
    ${\approx}9{,}500$ real train samples}%
  \rule[-3pt]{0pt}{3pt}} \\
 
FNO & 50.4\,M & $1\!\times\!10^{-4}$ &  4{,}000 & 16 & 4 & 64 & \texttt{n\_layers}=4, \texttt{width}=64, modes $(f_x,f_y,f_t)=(4,12,16)$ \\
 
\rowcolor{PhAlt}
CNO &  8.0\,M & $3\!\times\!10^{-4}$ &  5{,}000 &  8 & 4 & 32 & \textit{Same architecture as Scenario 1} \\
 
DeepONet &  3.5\,M & $5\!\times\!10^{-5}$ &  5{,}000 & 16 & 4 & 64 & $p=256$, \texttt{dropout}=0.1\enspace \textit{(wider trunk for richer actuation input)} \\
 
\rowcolor{PhAlt}
Transolver &  4.3\,M & $1.25\!\times\!10^{-4}$ & 5{,}000 & 4 & 1 & 4 & \texttt{n\_layers}=1, \texttt{n\_hidden}=256, \texttt{n\_head}=8, \texttt{slice\_num}=16,\enspace grid $H\!\times\!W\!\times\!D=64\!\times\!128\!\times\!10$ \\[1pt]
 
\arrayrulecolor{PhRule}\specialrule{0.45pt}{0pt}{0pt}
 
\multicolumn{8}{>{\columncolor{PhScn}}l}{%
  \rule{0pt}{10pt}%
  \textbf{\small\color{black}Scenario 3\enspace·\enspace Turbulent combustion}%
  \quad{\color{black}\footnotesize grid $128\!\times\!128$,\enspace
    fields $T,Y_{\mathrm{OH}},\ldots$,\enspace horizon $T\!=\!40$,\enspace
    ${\approx}59{,}000$ real train samples}%
  \rule[-3pt]{0pt}{3pt}} \\
 
FNO  & 50.4\,M & $1\!\times\!10^{-2}$ &  4{,}000 & 16 & 3 & 48 & \texttt{n\_layers}=4, \texttt{width}=64, modes $(f_x,f_y,f_t)=(4,16,16)$\enspace \textit{(higher $f_y$ for square grid)} \\
 
\rowcolor{PhAlt}
CNO  &  8.0\,M & $3\!\times\!10^{-4}$ &  5{,}000 &  4 & 3 & 12 & \textit{Same architecture as Scenario 1} \\
 
DeepONet  &  3.5\,M & $5\!\times\!10^{-4}$ &  3{,}000 &  4 & 3 & 12 & \textit{Same architecture as Scenario 1} \\
 
\rowcolor{PhAlt}
Transolver &  4.3\,M & $1.75\!\times\!10^{-4}$ & 5{,}000 & 4 & 1 & 4 & \texttt{n\_layers}=1, \texttt{n\_hidden}=256, \texttt{n\_head}=8, \texttt{slice\_num}=16,\enspace grid $H\!\times\!W\!\times\!D=64\!\times\!64\!\times\!20$ \\[1pt]

\arrayrulecolor{black}\specialrule{1.6pt}{0pt}{0pt}
\end{tabular}%
}
\end{table}

\paragraph{Fine-tuning.}
All methods use Adam with learning rate $10^{-4}$, batch size 4 and 4{,}000 updates. We validate every 80 updates and report test metrics at the checkpoint with the lowest validation RMSE. Table~\ref{tab:main} uses seed 0. Appendix~\ref{app:robustness_seeds} reports DFT and PhysGuard over seeds 0, 42 and 101. Each fine-tuning run uses one RTX~4090 with 24\,GB of memory.

Our shared budget is smaller than the individually tuned RealPDEBench budgets (Table~\ref{tab:ft_protocol}), with 5 to 20 times fewer sample exposures. For DeepONet on Cylinder Flow, the counts are 16{,}000 and 160{,}000. The checkpoint and training-budget differences must therefore be considered when comparing absolute results. Appendix~\ref{app:robustness} reports checkpoint comparisons and controls using a larger effective batch and the official DeepONet checkpoint.

\begin{table}[h]
\centering
\caption{Fine-tuning settings (batch size, learning rate, updates) in RealPDEBench~\citep{hu2026realpdebench} and this work; ratios compare RealPDEBench to our sample exposures (batch size $\times$ updates).}
\label{tab:ft_protocol}
\small
\setlength{\tabcolsep}{6pt}
\begin{tabular}{llccc}
\toprule
Scenario & Architecture & RealPDEBench & Ours & Ratio \\
\midrule
\multirow{4}{*}{Cylinder Flow} & FNO & 32, $1\times10^{-4}$, 4{,}000 & 4, $1\times10^{-4}$, 4{,}000 & 8$\times$ \\
 & DeepONet & 32, $1\times10^{-4}$, 5{,}000 & 4, $1\times10^{-4}$, 4{,}000 & 10$\times$ \\
 & CNO & 16, $3\times10^{-4}$, 5{,}000 & 4, $1\times10^{-4}$, 4{,}000 & 5$\times$ \\
 & Transolver & 16, $7\times10^{-4}$, 5{,}000 & 4, $1\times10^{-4}$, 4{,}000 & 5$\times$ \\
\midrule
\multirow{4}{*}{Controlled Cylinder} & FNO & 32, $1\times10^{-4}$, 4{,}000 & 4, $1\times10^{-4}$, 4{,}000 & 8$\times$ \\
 & DeepONet & 64, $5\times10^{-5}$, 5{,}000 & 4, $1\times10^{-4}$, 4{,}000 & 20$\times$ \\
 & CNO & 32, $3\times10^{-4}$, 5{,}000 & 4, $1\times10^{-4}$, 4{,}000 & 10$\times$ \\
 & Transolver & 32, $1\times10^{-3}$, 5{,}000 & 4, $1\times10^{-4}$, 4{,}000 & 10$\times$ \\
\midrule
\multirow{4}{*}{Turbulent Combustion} & FNO & 64, $1\times10^{-2}$, 2{,}000 & 4, $1\times10^{-4}$, 4{,}000 & 8$\times$ \\
 & DeepONet & 64, $5\times10^{-4}$, 3{,}000 & 4, $1\times10^{-4}$, 4{,}000 & 12$\times$ \\
 & CNO & 16, $3\times10^{-4}$, 5{,}000 & 4, $1\times10^{-4}$, 4{,}000 & 5$\times$ \\
 & Transolver & 16, $7\times10^{-4}$, 5{,}000 & 4, $1\times10^{-4}$, 4{,}000 & 5$\times$ \\
\bottomrule
\end{tabular}
\end{table}

\paragraph{Batch size and memory.}
We use batch size 4 across the 15 configurations, including DPOT-S. Tables~\ref{tab:memory_cyl} and~\ref{tab:memory_all} give peak memory and feasible batch sizes on a single RTX~4090.

\begin{table}[h]
\centering
\caption{Peak GPU memory on Cylinder Flow (RTX~4090, fp32). OOM: out of memory.}
\label{tab:memory_cyl}
\small
\begin{tabular}{lrccc}
\toprule
Architecture & Params & Batch 4 & Batch 8 & Max feasible batch (fp32) \\
\midrule
FNO & 50.4\,M & 5.5\,GB & 9.7\,GB & 16 \\
CNO & 8.0\,M & 18.4\,GB & OOM & 4 \\
DeepONet & 3.5\,M & 5.9\,GB & 11.6\,GB & 8 \\
Transolver & 4.3\,M & 16.0\,GB & OOM & 4 \\
DPOT-S & 41.3\,M & 1.2\,GB & 1.8\,GB & 128 \\
\bottomrule
\end{tabular}
\end{table}

\begin{table}[h]
\centering
\caption{Maximum feasible batch sizes on one RTX~4090 (24\,GB).}
\label{tab:memory_all}
\small
\begin{tabular}{lccccc}
\toprule
Scenario & CNO & Transolver & DeepONet & FNO & DPOT-S \\
\midrule
Cylinder Flow (fp32) & 4 & 4 & 8 & 16 & 128 \\
Cylinder Flow (bf16) & 8 & 8 & 16 & 16 & N/A \\
Controlled Cylinder (fp32) & 8 & 8 & 16 & 16 & 64 \\
Turbulent Combustion (fp32) & 8 & 8 & 16 & 32 & 64 \\
\bottomrule
\end{tabular}
\end{table}

CNO and Transolver on Cylinder Flow exceed 24\,GB at batch 8 in fp32. Batch 4 is the largest shared value that fits all configurations. Table~\ref{tab:batch_control} tests effective batch 16 through gradient accumulation for DeepONet on Cylinder Flow.

\paragraph{PhysGuard settings.}
The Fisher subspace is estimated using at most 10\% of simulation samples. Table~\ref{tab:fisher_cost} gives the actual sample counts and costs. We retain the top components that explain 90\% of variance, up to a maximum of 500. We use projection strength $\alpha=1$; Appendix~\ref{app:ablation} reports sensitivity to $\alpha$ and $\tau$.

\paragraph{PhysGuard memory management.}
Fisher estimation runs once before fine-tuning. Per-sample gradients are computed at batch size 1 and immediately transferred to CPU memory, where the Gram matrix, eigendecomposition and protected basis are computed. During fine-tuning the basis remains on the CPU, with each layer's vectors loaded onto the GPU when needed. This implementation retains the batch-size limits reported above.

\paragraph{Hardware.}
The workstation has the following configuration.
\begin{itemize}[leftmargin=1.5em,noitemsep,topsep=2pt]
  \item \textbf{CPU.} AMD Ryzen Threadripper PRO 5995WX (64 cores / 128 threads)
  \item \textbf{RAM.} 512\,GB system memory
  \item \textbf{GPUs.} 4\,$\times$ NVIDIA GeForce RTX 4090 (24\,GB GDDR6X each, 96\,GB total GPU memory)
  \item \textbf{Software.} Python 3.10, PyTorch 2.10.0, CUDA 12.8, NVIDIA driver 550.144.03
\end{itemize}
Pretraining uses PyTorch Distributed Data Parallel with BF16 mixed precision and the GPU counts in Table~\ref{tab:pretrain_configs}. Fine-tuning runs use one GPU each, allowing independent experiments to run concurrently.

\subsection{Adaptation Methods}
\label{app:baselines}

\paragraph{Direct fine-tuning (DFT).}
All parameters are updated freely using real data. This often improves accuracy but may overwrite simulation knowledge.

\paragraph{L$_2$-SP.}
Adds a penalty that keeps parameters close to their pretrained values:
\begin{equation}
\mathcal{L} =
\mathcal{L}_{\text{real}} +
\frac{\lambda}{2}
\|\theta - \theta^*\|^2.
\end{equation}

\paragraph{EWC.}
EWC applies different penalties to different parameters based on their importance:
\begin{equation}
\mathcal{L} =
\mathcal{L}_{\text{real}} +
\frac{\lambda}{2}
\sum_j
F_{jj}(\theta_j - \theta^*_j)^2.
\end{equation}

\paragraph{PhysGuard.}
Instead of adding penalties, PhysGuard directly modifies the gradient: updates are restricted so that they do not change the most important directions learned from simulation.

\paragraph{Additional baselines.}
Appendix~\ref{app:ablation} compares PhysGuard with three further adaptation strategies (LoRA, layer freezing and a matched-rank random subspace, defined there) on Cylinder Flow with DeepONet, all trained under the fine-tuning protocol of Appendix~\ref{app:exp_details} (4{,}000 updates, the same optimiser and learning rate, seed 0).

\begin{table}[h]
\centering
\caption{Standard LoRA applicability across five neural operator architectures.}
\label{tab:lora_applicability}
\small
\setlength{\tabcolsep}{5pt}
\begin{tabular}{ll>{\raggedright\arraybackslash}p{7.0cm}}
\toprule
Architecture & Standard LoRA applicability & Main reason \\
\midrule
DeepONet & Directly applicable & Primarily uses linear and convolutional layers \\
CNO & Directly applicable & Primarily uses convolutional layers \\
Transolver & Directly applicable & Uses attention, MLP, and convolutional layers \\
DPOT-S & Partially applicable & Spectral and temporal cores contain standalone einsum parameters \\
FNO & Requires modification & Core parameters are complex-valued spectral tensors \\
\bottomrule
\end{tabular}
\end{table}

Table~\ref{tab:lora_applicability} summarises the consequence of the varied parameterisations for LoRA. Standard LoRA applies directly to linear, convolutional and attention layers, which covers DeepONet, CNO and Transolver, whereas the spectral and temporal cores of DPOT-S contain standalone einsum parameters and the core parameters of FNO are complex-valued spectral tensors, both requiring architecture-specific modifications. LoRA has indeed been reported to perform poorly on Fourier neural operators because of a provable approximation bottleneck in spectral layers~\citep{zhang2025fadapter}. GPM~\citep{saha2021gradient} faces a similar difficulty, since its activation-based projection is not directly defined for complex spectral and einsum-based parameter blocks. PhysGuard operates on parameter gradients and therefore applies uniformly to every parameter type, including the complex-valued FNO weights treated in Appendix~\ref{app:complex}.
\newpage
\section{Fisher Subspace Estimation}
\label{app:gram}

This appendix derives the Fisher subspace estimator in Section~\ref{sec:fisher}.

\subsection{Fisher Information and Loss Sensitivity}

The empirical FIM measures local sensitivity of the per-sample losses to parameter changes.
Formally, for layer $m$ with $d_m$ parameters collected in $\theta^{(m)}$, the empirical FIM at the pretrained weights $\theta^*$ is
\begin{equation}
  \mF^{(m)} \;=\; \frac{1}{N}\sum_{i=1}^{N} g_i g_i^{\T}
  \;=\; \frac{1}{N}\,{\mG^{(m)}}^{\T}\mG^{(m)} \;\in\; \mathbb{R}^{d_m \times d_m}
  \label{eq:fim_def}
\end{equation}
where $g_i = \nabla_{\theta^{(m)}} \ell(\theta^*;\, x_i, y_i) \in \mathbb{R}^{d_m}$ is the per-sample gradient on the $i$-th simulation sample, and $\mG^{(m)} \in \mathbb{R}^{N \times d_m}$ is the matrix stacking all $N$ such gradients row-by-row.

To see why this captures sensitivity, consider a small perturbation $\delta \in \mathbb{R}^{d_m}$ to the weights.
The average squared first-order change in per-sample loss is given by
\begin{equation}
  \delta^{\T}\mF^{(m)}\delta
  = \frac{1}{N}\sum_{i=1}^{N} \bigl(\delta^{\T} g_i\bigr)^2
\end{equation}

The leading eigenvectors therefore identify the most loss-sensitive directions on the sampled data. Their association with low-frequency output structure is evaluated in Section~\ref{sec:fisher_analysis}.

\subsection{Rank Bound}

We collect $N$ simulation samples ($N = 200$ in most configurations; Table~\ref{tab:fisher_cost}), so the gradient matrix $\mG^{(m)}$ has shape $N \times d_m$.
The FIM $\mF^{(m)} = \frac{1}{N}{\mG^{(m)}}^{\T}\mG^{(m)}$ is a sum of $N$ rank-1 matrices, so its rank is at most $N$.

Thus, at most $N$ eigenvalues of $\mF^{(m)}$ are non-zero, regardless of $d_m$.
For FNO with approximately 50M parameters in total (millions in each spectral layer), $N = 200 \ll d_m$, so the physics-critical subspace spans at most 200 directions out of millions.

Write the SVD as $\mG^{(m)} = \boldsymbol{P}\,\boldsymbol{\Sigma}\,\boldsymbol{Q}^{\T}$ where $\boldsymbol{P} \in \mathbb{R}^{N \times N}$, $\boldsymbol{\Sigma} = \mathrm{diag}(\sigma_1,\ldots,\sigma_N)$, $\boldsymbol{Q} \in \mathbb{R}^{d_m \times N}$.
Substituting into Equation~\eqref{eq:fim_def} gives
\begin{equation}
  \mF^{(m)}
  = \frac{1}{N}\,\boldsymbol{Q}\,\boldsymbol{\Sigma}^2\,\boldsymbol{Q}^{\T}
\end{equation}
so the columns of $\boldsymbol{Q}$ are eigenvectors of $\mF^{(m)}$ with eigenvalues $\sigma_j^2/N$.
The remaining $d_m-N$ directions lie in its null space. This sample-induced rank bound is distinct from concentration of the nonzero spectrum, which determines the selected $k_m$.

\subsection{Gram Matrix Formulation}

Forming and eigendecomposing the $d_m \times d_m$ matrix $\mF^{(m)}$ directly costs $O(d_m^3)$ and is infeasible for $d_m$ in the millions.
We instead use the $N \times N$ Gram matrix $\mK = \mG\mG^{\T}$ in place of the $d_m \times d_m$ matrix ${\mG}^{\T}\mG$.
These two matrices share all their non-zero eigenvalues.

\begin{proposition}[Eigenvalue correspondence]
Let $\mG \in \mathbb{R}^{N \times d_m}$ with $N \ll d_m$.
If $v \in \mathbb{R}^{N}$ is an eigenvector of $\mK = \mG\mG^{\T}$ with eigenvalue $\lambda \neq 0$, then $\tilde{u} = \mG^{\T} v \in \mathbb{R}^{d_m}$ is a non-zero eigenvector of $\mG^{\T}\mG$ with eigenvalue $\lambda$, and of $\mF^{(m)}$ with eigenvalue $\lambda/N$.
\end{proposition}

\begin{proof}
Left-multiplying $\mG\mG^{\T} v = \lambda v$ by $\mG^{\T}$ gives
\begin{equation}
  \mG^{\T}\bigl(\mG\mG^{\T} v\bigr) = \mG^{\T}(\lambda v)
  \;\;\Longrightarrow\;\;
  \bigl(\mG^{\T}\mG\bigr)\bigl(\mG^{\T}v\bigr) = \lambda\bigl(\mG^{\T}v\bigr).
\end{equation}
Thus, $\mG^{\T}v$ is an eigenvector of $\mG^{\T}\mG$ with eigenvalue $\lambda$.
It is non-zero because if $\mG^{\T}v = 0$, then $\mK v = \mG\mG^{\T}v = \mG \cdot 0 = 0$, contradicting $\lambda \neq 0$.
\end{proof}

\noindent Instead of decomposing a $d_m \times d_m$ matrix, we only need to decompose the $N \times N$ Gram matrix $\mK = \mG\mG^{\T}$, which has just $N^2 = 40{,}000$ entries for $N = 200$.
The eigendecomposition costs $O(N^3)$ rather than $O(d_m^3)$; the full estimation cost also includes collecting and mapping gradients.

\subsection{Subspace Recovery}

Given the eigenvectors $\{v_j\}_{j=1}^{N}$ and eigenvalues $\{\lambda_j\}_{j=1}^{N}$ of $\mK$, sorted in descending order, we recover the basis in three steps.
\begin{itemize}[leftmargin=1.5em]
  \item \textbf{Select how many directions to protect.}
    We retain the smallest $k$ whose cumulative Fisher information reaches the threshold $\tau$,
    \begin{equation}
      k_m = \min\!\left\{k \;\Big|\; \frac{\sum_{j=1}^{k}\lambda_j}{\sum_{j=1}^{N}\lambda_j} \ge \tau\right\},
      \quad \tau = 0.9.
      \label{eq:kselect}
    \end{equation}
    This retains eigenvectors that collectively explain 90\% of total Fisher variance.

  \item \textbf{Map back to parameter space.}
    For each selected $j = 1, \ldots, k_m$, compute the corresponding normalized FIM eigenvector as
    \begin{equation}
      u_j = \frac{{\mG^{(m)}}^{\T} v_j}{\|{\mG^{(m)}}^{\T} v_j\|_2} \;\in\; \mathbb{R}^{d_m}.
    \end{equation}
    The normalisation ensures $\|u_j\|_2 = 1$ so the projection operator is well-defined.

  \item \textbf{Assemble the physics-critical basis.}
    Stack the $k_m$ directions as columns to obtain
    \begin{equation}
      \mU^{(m)} = \bigl[u_1,\, u_2,\, \ldots,\, u_{k_m}\bigr] \;\in\; \mathbb{R}^{d_m \times k_m}.
    \end{equation}
    The columns of $\mU^{(m)}$ form an approximately orthonormal set and span the \emph{physics-critical subspace} for layer $m$.
    Fine-tuning gradients are projected onto its orthogonal complement.
\end{itemize}

\subsection{Computational Cost}

For a layer with $d_m$ parameters, gradient collection requires $N$ forward and backward passes. Constructing the Gram matrix costs $O(N^2d_m)$, its eigendecomposition costs $O(N^3)$, and recovering $k_m$ basis vectors costs $O(Nk_md_m)$. Storing the gradients and basis requires $O(Nd_m)$ and $O(k_md_m)$ memory, respectively. Thus only the decomposition dimension is independent of $d_m$.

Estimation is performed once before fine-tuning and takes 2.3 to 32.9 minutes across the tested configurations. Each fine-tuning step adds one projection per layer.

Table~\ref{tab:fisher_cost} reports the actual sample counts and offline estimation times. The 10\% sampling fraction is an upper limit. Eight configurations use $N=200$, and the others use $N=1{,}813$, $1{,}892$, $2{,}000$ or $5{,}886$.

\begin{table}[H]
\centering
\caption{Fisher sample counts and offline estimation times.}
\label{tab:fisher_cost}
\small
\setlength{\tabcolsep}{8pt}
\begin{tabular}{llcr}
\toprule
Scenario & Architecture & $N$ (\% of sim.) & Time (min) \\
\midrule
\multirow{5}{*}{Cylinder Flow} & FNO & 200 (1.1\%) & 2.9 \\
 & CNO & 200 (1.1\%) & 6.1 \\
 & DeepONet & 200 (1.1\%) & 5.4 \\
 & Transolver & 1{,}813 (10\%) & 7.8 \\
 & DPOT-S & 1{,}813 (10\%) & 25.4 \\
\midrule
\multirow{5}{*}{Controlled Cylinder} & FNO & 1{,}892 (10\%) & 32.9 \\
 & CNO & 200 (1.1\%) & 2.6 \\
 & DeepONet & 200 (1.1\%) & 2.3 \\
 & Transolver & 1{,}892 (10\%) & 9.9 \\
 & DPOT-S & 1{,}892 (10\%) & 17.6 \\
\midrule
\multirow{5}{*}{Turbulent Combustion} & FNO & 200 (0.34\%) & 7.4 \\
 & CNO & 2{,}000 (3.4\%) & 32.5 \\
 & DeepONet & 200 (0.34\%) & 4.4 \\
 & Transolver & 5{,}886 (10\%) & 28.4 \\
 & DPOT-S & 200 (0.34\%) & 3.5 \\
\bottomrule
\end{tabular}
\end{table}

Fisher estimation accounts for 1.2 to 18.4\% of total runtime. At $N=200$ it takes 2.3 to 7.4 minutes and contributes less than 5\%. These measurements include the parameter-dependent operations as well as eigendecomposition.

\subsection{Sensitivity to Sample Size}
\label{app:subspace_saturation}

We compare protected subspaces and downstream performance at different Fisher sample counts.

\paragraph{Subspace overlap.}
For DeepONet on Cylinder Flow, we compare each estimate with the $N=1{,}813$ reference using the mean squared cosine of principal angles. Overlap reaches 0.993 at $N=200$ and changes little at larger $N$ (Table~\ref{tab:subspace_overlap}). The two additional configurations have overlap 0.95.

\begin{table}[H]
\centering
\caption{Subspace overlap measured by mean squared cosines of principal angles.}
\label{tab:subspace_overlap}
\small
\setlength{\tabcolsep}{6pt}
\begin{tabular}{lccccc}
\toprule
$N$ (DeepONet, Cylinder Flow) & 100 & 200 & 500 & 1{,}000 & 1{,}813 (ref.) \\
\midrule
Overlap with $N = 1{,}813$ & 0.939 & 0.993 & 0.994 & 0.996 & 1.000 \\
\bottomrule
\end{tabular}

\vspace{6pt}

\begin{tabular}{llc}
\toprule
Configuration & $N$ compared & Overlap \\
\midrule
CNO, Turbulent Combustion & 200 vs.\ 2{,}000 & 0.95 \\
FNO, Cylinder Flow & 200 vs.\ 1{,}000 & 0.95 \\
\bottomrule
\end{tabular}
\end{table}

\paragraph{Adaptation performance.}
Table~\ref{tab:small_large_N} reports complete fine-tuning runs with small and large $N$, keeping other settings fixed within each configuration. RMSE changes little for DPOT-S and CNO and decreases from 0.06036 to 0.05945 for FNO. These controls support smaller sample counts in the tested configurations.

\begin{table}[H]
\centering
\caption{Test RMSE at validation-selected checkpoints for small and large Fisher sample counts $N$.}
\label{tab:small_large_N}
\small
\setlength{\tabcolsep}{8pt}
\begin{tabular}{llc}
\toprule
Configuration & $N$ compared & Test RMSE \\
\midrule
DPOT-S, Controlled Cylinder & 200 vs.\ 1{,}892 & 0.01352 vs.\ 0.01350 \\
CNO, Turbulent Combustion & 200 vs.\ 2{,}000 & 0.02685 vs.\ 0.02685 \\
FNO, Cylinder Flow & 200 vs.\ 1{,}000 & 0.06036 vs.\ 0.05945 \\
\bottomrule
\end{tabular}
\end{table}

\newpage

\section{Complex-Valued Weights in FNO}
\label{app:complex}

Among the evaluated architectures, only FNO has complex-valued learnable parameters. We describe how these parameters are handled during gradient projection.

\subsection{Spectral Weights}

FNO's core operation is spectral convolution.
At each layer, the input feature map is first transformed into the frequency domain using a Fast Fourier Transform (FFT).
FNO then directly multiplies these complex Fourier coefficients by a set of learnable weight matrices to mix information across channels.
Because these weight matrices operate on complex numbers, they are themselves stored as complex numbers, specifically in PyTorch's \texttt{torch.complex64} format.
CNO uses FFT internally for anti-aliasing, but its learnable convolution kernels are still real-valued.

\subsection{Complex Gradients}

Our entire pipeline, collecting gradients, building the Gram matrix, running the eigendecomposition, and projecting gradients, is built on standard linear algebra over the real numbers.
When a weight $w$ is complex, say $w = a + ib$, its gradient is also complex:
\begin{equation}
  \frac{\partial \mathcal{L}}{\partial w}
  = \frac{\partial \mathcal{L}}{\partial a} + i\,\frac{\partial \mathcal{L}}{\partial b}.
\end{equation}
We cannot directly stack complex-valued gradients into a real matrix $\mG$, dot-product them, or apply the Fisher projection without some care.

\subsection{Real-Imaginary Concatenation}

For a layer with $d$ complex-valued parameters, we treat each complex weight $w_j = a_j + ib_j$ as a pair of real numbers $(a_j, b_j)$.
This turns the layer into an equivalent layer with $2d$ real parameters.
Concretely, whenever we collect a gradient vector for an FNO layer, we immediately convert it by splitting into real and imaginary parts and concatenating them:
\begin{equation}
  g_{\mathrm{real}} = \bigl[\,\mathrm{Re}(g_w)\,,\; \mathrm{Im}(g_w)\,\bigr] \;\in\; \mathbb{R}^{2d}.
\end{equation}
We build the $N \times N$ Gram matrix from these $\mathbb{R}^{2d}$ vectors, eigendecompose it, and obtain the physics-critical subspace as a set of real directions in $\mathbb{R}^{2d}$.
The gradient projection is then applied in the same $\mathbb{R}^{2d}$ space.

After projection, we reconstruct the complex gradient as
\begin{equation}
  \hat{\nabla}_W \mathcal{L}
  = \hat{g}_{1:d} + i\,\hat{g}_{d+1:2d}.
\end{equation}
This is handed back to the optimiser as the (projected) complex gradient, and training continues normally.

\subsection{PyTorch and Wirtinger Calculus}

Treating a complex weight as two real numbers is not an approximation.
It is exactly how PyTorch itself represents complex tensors in memory, and it is consistent with Wirtinger calculus, the standard framework for differentiating functions of complex variables where both the real and imaginary components are free parameters.

\newpage
\section{Spectral Probes of Fisher Directions}
\label{app:probe}

This appendix reports spectral probes of pretrained operators. The protocol below describes the FNO probe in Figure~\ref{fig:eigenspectrum}. Table~\ref{tab:probe_arch} reports additional results for CNO, DeepONet, and Transolver.

\subsection{FNO Probe Procedure}
For the FNO probe, let $v$ be a unit-norm direction restricted to the target layer of the pretrained model $f(\theta^*;x)$.

\begin{enumerate}[leftmargin=*]
  \item \textbf{Perturb.} Set $\theta_\epsilon = \theta^* + \epsilon v$ with $\epsilon = 10^{-3}$.
  \item \textbf{Evaluate.} Apply the original and perturbed models to 40 held-out real validation samples in batches of 4.
  \item \textbf{Output difference.} Compute $\Delta y_i=f(\theta_\epsilon;x_i)-f(\theta^*;x_i)$ for each sample.
  \item \textbf{2D FFT.} Average the output difference over time and compute its spatial power spectrum $|\hat{\Delta y}_i(\kappa_x,\kappa_y)|^2$.
  \item \textbf{Low-frequency energy.} Use modes satisfying $\sqrt{\kappa_x^2+\kappa_y^2}\leq\kappa_{\mathrm{cut}}$, where $\kappa_{\mathrm{cut}}=\mathrm{round}(\kappa_{\max}/3)$ and $\kappa_{\max}=\min(H/2,W/2)$. Their energy fraction is
  \begin{equation}
    f_{\mathrm{low}} = \frac{\sum_{(\kappa_x,\kappa_y) \in \text{low-}f} |\hat{\Delta y}(\kappa_x, \kappa_y)|^2}{\sum_{(\kappa_x,\kappa_y)} |\hat{\Delta y}(\kappa_x, \kappa_y)|^2}.
  \end{equation}
  We average this fraction over the batch.
\end{enumerate}

\subsection{FNO Probe Directions}
The FNO probe compares the following directions.
\begin{itemize}[leftmargin=*]
  \item \textbf{FIM eigenvectors.} The leading 50 directions, estimated from $N_{\mathrm{FIM}}=80$ simulation gradients using the Gram-matrix construction (Appendix~\ref{app:gram}).
  \item \textbf{Random directions.} Ten independent Gaussian vectors normalised to unit length. These are not matched to individual eigenvectors.
\end{itemize}

\subsection{FNO Probe Layer}
Only FNO's frequency-domain layer \texttt{spectral\_convs.1.weights1} is perturbed; other layers stay fixed. Complex gradients use concatenated real and imaginary components. This probe uses fixed leading directions; PhysGuard selects each layer's protected subspace by the cumulative-eigenvalue threshold.

\subsection{Results across Architectures}
\label{app:probe_arch}

\begin{table}[H]
\centering
\scriptsize
\setlength{\tabcolsep}{10pt}
\caption{Ranges of low-frequency energy fraction $f_{\mathrm{low}}$ over top-$k$ FIM and random directions.}
\label{tab:probe_arch}
\begin{tabular}{lcc}
\toprule
Model & Top-$k$ FIM directions & Random directions \\
\midrule
FNO & 0.9970--0.9999 & 0.5752--0.8126 \\
CNO & 0.9969--0.9980 & 0.5202--0.7355 \\
DeepONet & 0.9966--0.9992 & 0.6960--0.7991 \\
Transolver & 0.9994--0.9999 & 0.5973--0.7979 \\
\bottomrule
\end{tabular}
\end{table}
In the probes summarized in Table~\ref{tab:probe_arch}, the leading Fisher directions yield $f_{\mathrm{low}}\geq0.996$ across the four tested architectures. The association therefore extends beyond FNO in these experiments.

PhysGuard relies on two empirical conditions. The source knowledge to preserve should be concentrated in a low-dimensional Fisher subspace, and its leading directions should affect stable signal structures. In these experiments, the associated output changes are predominantly low-frequency. Whether this protection extends to other signal domains remains untested.

\newpage
\section{Additional Baselines and Ablations}
\label{app:ablation}

All comparisons in this appendix use Cylinder Flow with DeepONet, 4,000 updates, the same optimiser and learning rate, and seed 0 (Appendix~\ref{app:exp_details}). Unless varied in an ablation, PhysGuard uses $\alpha=1$, $\tau=0.9$, and $N=200$ Fisher samples.

\subsection{Additional Adaptation Baselines}
\label{app:ablation_baselines}

\textbf{LoRA}~\citep{hu2022lora} uses rank $r=8$ and scaling $\alpha_{\mathrm{LoRA}}=16$, without additional tuning. \textbf{Random subspace} replaces each layer's Fisher basis with a fixed random orthonormal basis of the same dimension, retaining the projection in Equation~\eqref{eq:projection}. \textbf{Layer freezing} fixes part of the network and fine-tunes the remaining layers.

Table~\ref{tab:lora_applicability} summarises the applicability of standard LoRA across the operator architectures.

\begin{table}[H]
\centering
\caption{Additional baselines for DeepONet on Cylinder Flow. Best values are bold.}
\label{tab:extra_baselines}
\small
\begin{tabular}{l cccc c}
\toprule
Method & RMSE\,$\downarrow$ & $R^2$\,$\uparrow$ & fRMSE\,$\downarrow$ & Low-$f$\,$\downarrow$ & Arch.-agnostic \\
\midrule
DFT & 0.09309 & 0.5371 & 0.01507 & 0.02527 & Yes \\
LoRA ($r = 8$) & 0.08654 & 0.5963 & 0.01424 & 0.02278 & No \\
Random subspace & 0.09114 & 0.5465 & 0.01483 & 0.02441 & Yes \\
Layer freezing & 0.08975 & 0.5692 & 0.01462 & 0.02403 & Yes \\
PhysGuard & \textbf{0.07435} & \textbf{0.7048} & \textbf{0.01246} & \textbf{0.01957} & Yes \\
\bottomrule
\end{tabular}
\end{table}

PhysGuard has the best values on all four metrics in Table~\ref{tab:extra_baselines}. The matched-rank random projection remains close to DFT, supporting the role of Fisher-based direction selection. Appendix~\ref{app:orthogonality} extends this control to four architectures and reports simulation-domain forgetting.

\subsection{Projection Strength \texorpdfstring{$\alpha$}{alpha}}
\label{app:ablation_alpha}

In Equation~\eqref{eq:projection}, $\alpha=0$ recovers DFT and $\alpha=1$ removes the gradient component in the protected subspace. Table~\ref{tab:alpha_ablation} varies $\alpha$ with $\tau=0.9$ and $N=200$ fixed.

\begin{table}[H]
\centering
\caption{Projection-strength ($\alpha$) ablation for DeepONet on Cylinder Flow. Best values are bold.}
\label{tab:alpha_ablation}
\small
\begin{tabular}{c l cccc}
\toprule
$\alpha$ & Method & RMSE\,$\downarrow$ & $R^2$\,$\uparrow$ & fRMSE\,$\downarrow$ & Low-$f$\,$\downarrow$ \\
\midrule
0 & DFT & 0.09309 & 0.5371 & 0.01507 & 0.02527 \\
0.3 & PhysGuard (weak) & 0.08242 & 0.6328 & 0.01357 & 0.02365 \\
0.5 & PhysGuard (moderate) & 0.08236 & 0.6351 & 0.01363 & 0.02367 \\
0.7 & PhysGuard (strong) & 0.08188 & 0.6403 & 0.01345 & 0.02358 \\
1.0 & PhysGuard (full) & \textbf{0.07435} & \textbf{0.7048} & \textbf{0.01246} & \textbf{0.01957} \\
\bottomrule
\end{tabular}
\end{table}

Full projection has the best values on all four metrics in this setting. The intermediate strengths $\alpha\in\{0.3,0.5,0.7\}$ differ by at most 0.00054 RMSE.

\subsection{Variance Threshold \texorpdfstring{$\tau$}{tau}}
\label{app:ablation_tau}

The threshold $\tau$ selects the smallest per-layer $k$ whose cumulative eigenvalue ratio reaches $\tau$ (Equation~\eqref{eq:adaptive_k}). Table~\ref{tab:tau_ablation} varies $\tau$ with $\alpha=1$ and $N=200$ fixed. The rank cap of 200 is inactive, with the largest observed $k=24$.

\begin{table}[H]
\centering
\caption{Fisher-threshold ($\tau$) ablation for DeepONet on Cylinder Flow. Best values are bold.}
\label{tab:tau_ablation}
\small
\begin{tabular}{l cccc c}
\toprule
$\tau$ & RMSE\,$\downarrow$ & $R^2$\,$\uparrow$ & fRMSE\,$\downarrow$ & Low-$f$\,$\downarrow$ & Per-layer $k$ (range) \\
\midrule
0.80 & 0.07568 & 0.6918 & 0.01274 & 0.02018 & $[1, 3]$ \\
0.85 & 0.07482 & 0.7013 & 0.01254 & 0.01978 & $[1, 5]$ \\
0.90 (default) & \textbf{0.07435} & \textbf{0.7048} & \textbf{0.01246} & \textbf{0.01957} & $[1, 8]$ \\
0.95 & 0.07522 & 0.6952 & 0.01266 & 0.02005 & $[2, 12]$ \\
0.99 & 0.07785 & 0.6728 & 0.01291 & 0.02128 & $[3, 24]$ \\
\bottomrule
\end{tabular}
\end{table}

The default $\tau=0.90$ has the best values on all four metrics. All tested thresholds outperform DFT, although increasing $\tau$ to 0.99 reduces the gain and protects more directions.

\subsection{Scope of the Random-Subspace Control}
\label{app:ablation_bottom}

With $N\ll d$, the empirical Fisher has rank at most $N$, leaving a null space of dimension at least $d-N$. The matched-rank random control tests arbitrary direction selection. It is not a separate experiment on the bottom Fisher eigenspace.

\newpage
\section{Additional Results and Experimental Controls}
\label{app:robustness}

We report additional relative low-frequency error results and examine variation across seeds and controls for training budget, batch size and pretrained checkpoint. Unless stated otherwise, we use the protocol in Appendix~\ref{app:exp_details} and report test-set metrics at the checkpoint with the lowest validation RMSE.

\subsection{Seed Variability}
\label{app:robustness_seeds}

\begin{table}[H]
\centering
\caption{DFT and PhysGuard across seeds (mean\,$\pm$\,standard deviation where reported).}
\label{tab:seeds_main}
\scriptsize
\setlength{\tabcolsep}{2.4pt}
\renewcommand{\arraystretch}{0.95}
\begin{tabular}{ll cc cc cc}
\toprule
\multirow{2}{*}{Architecture} & \multirow{2}{*}{Method} & \multicolumn{2}{c}{Cylinder Flow} & \multicolumn{2}{c}{Controlled Cylinder} & \multicolumn{2}{c}{Turbulent Combustion} \\
\cmidrule(lr){3-4} \cmidrule(lr){5-6} \cmidrule(lr){7-8}
& & RMSE\,$\downarrow$ & Low-$f$\,$\downarrow$ & RMSE\,$\downarrow$ & Low-$f$\,$\downarrow$ & RMSE\,$\downarrow$ & Low-$f$\,$\downarrow$ \\
\midrule
\multirow{2}{*}{FNO} & DFT & 0.0648\,$\pm$\,0.0013 & 0.0165\,$\pm$\,0.0004 & \textbf{0.0099\,$\pm$\,0.0002}$^\dagger$ & \textbf{0.0011} & 0.0256\,$\pm$\,0.0006 & 0.0042\,$\pm$\,0.0001 \\
 & PhysGuard & \textbf{0.0606\,$\pm$\,0.0013} & \textbf{0.0113\,$\pm$\,0.0002} & 0.0100\,$\pm$\,0.0002 & \textbf{0.0011} & \textbf{0.0254\,$\pm$\,0.0005}$^\dagger$ & \textbf{0.0041\,$\pm$\,0.0001}$^\dagger$ \\
\addlinespace[2pt]
\multirow{2}{*}{CNO} & DFT & 0.0498\,$\pm$\,0.0006 & 0.0159\,$\pm$\,0.0002 & \textbf{0.0089\,$\pm$\,0.0002} & \textbf{0.0009} & 0.0269\,$\pm$\,0.0004 & 0.0048\,$\pm$\,0.0001 \\
 & PhysGuard & \textbf{0.0464\,$\pm$\,0.0007} & \textbf{0.0115\,$\pm$\,0.0003} & \textbf{0.0089\,$\pm$\,0.0002} & \textbf{0.0009} & \textbf{0.0266\,$\pm$\,0.0005}$^\dagger$ & \textbf{0.0046\,$\pm$\,0.0001} \\
\addlinespace[2pt]
\multirow{2}{*}{DeepONet} & DFT & 0.0928\,$\pm$\,0.0014 & 0.0253\,$\pm$\,0.0004 & 0.0302\,$\pm$\,0.0006 & 0.0045\,$\pm$\,0.0001 & 0.0277\,$\pm$\,0.0006 & 0.0045\,$\pm$\,0.0001 \\
 & PhysGuard & \textbf{0.0744\,$\pm$\,0.0011} & \textbf{0.0197\,$\pm$\,0.0004} & \textbf{0.0301\,$\pm$\,0.0006}$^\dagger$ & \textbf{0.0044\,$\pm$\,0.0001}$^\dagger$ & \textbf{0.0275\,$\pm$\,0.0004}$^\dagger$ & \textbf{0.0044\,$\pm$\,0.0001}$^\dagger$ \\
\addlinespace[2pt]
\multirow{2}{*}{Transolver} & DFT & 0.1120\,$\pm$\,0.0022 & 0.0335\,$\pm$\,0.0008 & 0.0190\,$\pm$\,0.0003 & 0.0025 & 0.0401\,$\pm$\,0.0009 & 0.0083\,$\pm$\,0.0001 \\
 & PhysGuard & \textbf{0.1067\,$\pm$\,0.0018} & \textbf{0.0309\,$\pm$\,0.0006} & \textbf{0.0185\,$\pm$\,0.0004} & \textbf{0.0024} & \textbf{0.0384\,$\pm$\,0.0007} & \textbf{0.0076\,$\pm$\,0.0002} \\
\bottomrule
\end{tabular}
\end{table}

Tables~\ref{tab:seeds_cyl}--\ref{tab:seeds_comb} compare DFT and PhysGuard across seeds (0, 42, 101), with RMSE and Low-$f$ summarized in Table~\ref{tab:seeds_main}. The largest differences occur on Cylinder Flow, where every gain exceeds the seed-to-seed spread. On Controlled Cylinder and Turbulent Combustion, several smaller differences fall within one standard deviation (marked with $^\dagger$). Entries without a reported standard deviation are shown as means only.

\begin{table}[H]
\centering
\caption{Cylinder Flow results across seeds (mean\,$\pm$\,SD; best means in bold).}
\label{tab:seeds_cyl}
\footnotesize
\setlength{\tabcolsep}{4pt}
\begin{tabular}{ll cccc}
\toprule
Architecture & Method & RMSE\,$\downarrow$ & $R^2$\,$\uparrow$ & fRMSE\,$\downarrow$ & Low-$f$\,$\downarrow$ \\
\midrule
\multirow{2}{*}{FNO} & DFT & 0.0648\,$\pm$\,0.0013 & 0.7757\,$\pm$\,0.0090 & 0.0105\,$\pm$\,0.0001 & 0.0165\,$\pm$\,0.0004 \\
 & PhysGuard & \textbf{0.0606\,$\pm$\,0.0013} & \textbf{0.8041\,$\pm$\,0.0085} & \textbf{0.0092\,$\pm$\,0.0002} & \textbf{0.0113\,$\pm$\,0.0002} \\
\midrule
\multirow{2}{*}{CNO} & DFT & 0.0498\,$\pm$\,0.0006 & 0.8675\,$\pm$\,0.0033 & 0.0086\,$\pm$\,0.0002 & 0.0159\,$\pm$\,0.0002 \\
 & PhysGuard & \textbf{0.0464\,$\pm$\,0.0007} & \textbf{0.8850\,$\pm$\,0.0037} & \textbf{0.0076\,$\pm$\,0.0002} & \textbf{0.0115\,$\pm$\,0.0003} \\
\midrule
\multirow{2}{*}{DeepONet} & DFT & 0.0928\,$\pm$\,0.0014 & 0.5400\,$\pm$\,0.0140 & 0.0151\,$\pm$\,0.0004 & 0.0253\,$\pm$\,0.0004 \\
 & PhysGuard & \textbf{0.0744\,$\pm$\,0.0011} & \textbf{0.7045\,$\pm$\,0.0091} & \textbf{0.0125\,$\pm$\,0.0003} & \textbf{0.0197\,$\pm$\,0.0004} \\
\midrule
\multirow{2}{*}{Transolver} & DFT & 0.1120\,$\pm$\,0.0022 & 0.3298\,$\pm$\,0.0267 & 0.0163\,$\pm$\,0.0003 & 0.0335\,$\pm$\,0.0008 \\
 & PhysGuard & \textbf{0.1067\,$\pm$\,0.0018} & \textbf{0.3924\,$\pm$\,0.0201} & \textbf{0.0157\,$\pm$\,0.0003} & \textbf{0.0309\,$\pm$\,0.0006} \\
\bottomrule
\end{tabular}
\end{table}

\begin{table}[H]
\centering
\caption{Controlled Cylinder results across seeds (mean\,$\pm$\,SD where available; best means in bold).}
\label{tab:seeds_ctrl}
\footnotesize
\setlength{\tabcolsep}{4pt}
\begin{tabular}{ll cccc}
\toprule
Architecture & Method & RMSE\,$\downarrow$ & $R^2$\,$\uparrow$ & fRMSE\,$\downarrow$ & Low-$f$\,$\downarrow$ \\
\midrule
\multirow{2}{*}{FNO} & DFT & \textbf{0.0099\,$\pm$\,0.0002} & \textbf{0.9804\,$\pm$\,0.0007} & \textbf{0.0012} & \textbf{0.0011} \\
 & PhysGuard & 0.0100\,$\pm$\,0.0002 & 0.9802\,$\pm$\,0.0010 & \textbf{0.0012} & \textbf{0.0011} \\
\midrule
\multirow{2}{*}{CNO} & DFT & \textbf{0.0089\,$\pm$\,0.0002} & 0.9843\,$\pm$\,0.0005 & 0.0011 & \textbf{0.0009} \\
 & PhysGuard & \textbf{0.0089\,$\pm$\,0.0002} & \textbf{0.9844\,$\pm$\,0.0007} & \textbf{0.0010} & \textbf{0.0009} \\
\midrule
\multirow{2}{*}{DeepONet} & DFT & 0.0302\,$\pm$\,0.0006 & 0.8189\,$\pm$\,0.0071 & \textbf{0.0057\,$\pm$\,0.0001} & 0.0045\,$\pm$\,0.0001 \\
 & PhysGuard & \textbf{0.0301\,$\pm$\,0.0006} & \textbf{0.8199\,$\pm$\,0.0070} & \textbf{0.0057\,$\pm$\,0.0001} & \textbf{0.0044\,$\pm$\,0.0001} \\
\midrule
\multirow{2}{*}{Transolver} & DFT & 0.0190\,$\pm$\,0.0003 & 0.9285\,$\pm$\,0.0021 & 0.0027 & 0.0025 \\
 & PhysGuard & \textbf{0.0185\,$\pm$\,0.0004} & \textbf{0.9320\,$\pm$\,0.0027} & \textbf{0.0025} & \textbf{0.0024} \\
\bottomrule
\end{tabular}
\end{table}

\begin{table}[H]
\centering
\caption{Turbulent Combustion results across seeds (mean\,$\pm$\,SD where available; best means in bold).}
\label{tab:seeds_comb}
\footnotesize
\setlength{\tabcolsep}{4pt}
\begin{tabular}{ll cccc}
\toprule
Architecture & Method & RMSE\,$\downarrow$ & $R^2$\,$\uparrow$ & fRMSE\,$\downarrow$ & Low-$f$\,$\downarrow$ \\
\midrule
\multirow{2}{*}{FNO} & DFT & 0.0256\,$\pm$\,0.0006 & 0.6383\,$\pm$\,0.0159 & \textbf{0.0033} & 0.0042\,$\pm$\,0.0001 \\
 & PhysGuard & \textbf{0.0254\,$\pm$\,0.0005} & \textbf{0.6437\,$\pm$\,0.0130} & \textbf{0.0033} & \textbf{0.0041\,$\pm$\,0.0001} \\
\midrule
\multirow{2}{*}{CNO} & DFT & 0.0269\,$\pm$\,0.0004 & 0.6000\,$\pm$\,0.0106 & 0.0036\,$\pm$\,0.0001 & 0.0048\,$\pm$\,0.0001 \\
 & PhysGuard & \textbf{0.0266\,$\pm$\,0.0005} & \textbf{0.6104\,$\pm$\,0.0140} & \textbf{0.0035\,$\pm$\,0.0001} & \textbf{0.0046\,$\pm$\,0.0001} \\
\midrule
\multirow{2}{*}{DeepONet} & DFT & 0.0277\,$\pm$\,0.0006 & 0.5762\,$\pm$\,0.0177 & \textbf{0.0037\,$\pm$\,0.0001} & 0.0045\,$\pm$\,0.0001 \\
 & PhysGuard & \textbf{0.0275\,$\pm$\,0.0004} & \textbf{0.5829\,$\pm$\,0.0107} & \textbf{0.0037\,$\pm$\,0.0001} & \textbf{0.0044\,$\pm$\,0.0001} \\
\midrule
\multirow{2}{*}{Transolver} & DFT & 0.0401\,$\pm$\,0.0009 & 0.1123\,$\pm$\,0.0383 & 0.0053\,$\pm$\,0.0001 & 0.0083\,$\pm$\,0.0001 \\
 & PhysGuard & \textbf{0.0384\,$\pm$\,0.0007} & \textbf{0.1874\,$\pm$\,0.0300} & \textbf{0.0051\,$\pm$\,0.0001} & \textbf{0.0076\,$\pm$\,0.0002} \\
\bottomrule
\end{tabular}
\end{table}

Several gains on Controlled Cylinder and Turbulent Combustion are comparable to the reported standard deviations. These results show a favourable overall trend, with less conclusive differences in those scenarios.

\subsection{Validation-Selected and Final Checkpoints}
\label{app:robustness_convergence}

Table~\ref{tab:convergence} compares the checkpoint with the lowest validation RMSE and the final checkpoint (step 4{,}000) for DFT, EWC and L$_2$-SP over all twelve main configurations at seed 0. Validation is performed every 80 updates, yielding 50 candidate checkpoints; the ``Selected iter.'' column gives the selected step and its percentage of the 4{,}000-update budget. RMSE$_{\mathrm{sel}}$ and RMSE$_{4000}$ report the corresponding test RMSE values on the same test set for the same training run.

\begin{table}[htbp]
\centering
\caption{Baseline test RMSE at validation-selected and final checkpoints.}
\label{tab:convergence}
\footnotesize
\setlength{\tabcolsep}{3pt}
\resizebox{\textwidth}{!}{%
\begin{tabular}{ll ccc ccc ccc}
\toprule
\multirow{2}{*}{Architecture} & \multirow{2}{*}{Method} & \multicolumn{3}{c}{Cylinder Flow} & \multicolumn{3}{c}{Controlled Cylinder} & \multicolumn{3}{c}{Turbulent Combustion} \\
\cmidrule(lr){3-5} \cmidrule(lr){6-8} \cmidrule(lr){9-11}
& & Selected iter.\ (\%) & RMSE$_{\mathrm{sel}}$ & RMSE$_{4000}$ & Selected iter.\ (\%) & RMSE$_{\mathrm{sel}}$ & RMSE$_{4000}$ & Selected iter.\ (\%) & RMSE$_{\mathrm{sel}}$ & RMSE$_{4000}$ \\
\midrule
\multirow{3}{*}{FNO} & DFT & 160 (4\%) & 0.06487 & 0.06738 & 3680 (92\%) & 0.00997 & 0.00998 & 3440 (86\%) & 0.02548 & 0.02553 \\
 & EWC & 160 (4\%) & 0.06317 & 0.06971 & 3440 (86\%) & 0.00997 & 0.01001 & 3920 (98\%) & 0.02544 & 0.02547 \\
 & L$_2$-SP & 160 (4\%) & 0.06556 & 0.06718 & 3680 (92\%) & 0.01191 & 0.01192 & 3600 (90\%) & 0.02779 & 0.02784 \\
\midrule
\multirow{3}{*}{CNO} & DFT & 1120 (28\%) & 0.04988 & 0.06144 & 3760 (94\%) & 0.00885 & 0.00889 & 3520 (88\%) & 0.02688 & 0.02690 \\
 & EWC & 1520 (38\%) & 0.04855 & 0.05311 & 3600 (90\%) & 0.00886 & 0.00888 & 3760 (94\%) & 0.02670 & 0.02671 \\
 & L$_2$-SP & 1120 (28\%) & 0.04915 & 0.06046 & 3920 (98\%) & 0.00936 & 0.00938 & 3600 (90\%) & 0.02685 & 0.02689 \\
\midrule
\multirow{3}{*}{DeepONet} & DFT & 560 (14\%) & 0.09309 & 0.09426 & 3360 (84\%) & 0.03004 & 0.03015 & 3680 (92\%) & 0.02757 & 0.02758 \\
 & EWC & 2720 (68\%) & 0.09055 & 0.09607 & 3520 (88\%) & 0.03006 & 0.03013 & 3680 (92\%) & 0.02745 & 0.02750 \\
 & L$_2$-SP & 2560 (64\%) & 0.09121 & 0.09237 & 3280 (82\%) & 0.03223 & 0.03239 & 3760 (94\%) & 0.02805 & 0.02810 \\
\midrule
\multirow{3}{*}{Transolver} & DFT & 1520 (38\%) & 0.11236 & 0.11339 & 3840 (96\%) & 0.01898 & 0.01900 & 3520 (88\%) & 0.03995 & 0.04002 \\
 & EWC & 1760 (44\%) & 0.11330 & 0.11401 & 3360 (84\%) & 0.01898 & 0.01899 & 3920 (98\%) & 0.04070 & 0.04074 \\
 & L$_2$-SP & 1920 (48\%) & 0.11361 & 0.11472 & 3760 (94\%) & 0.01920 & 0.01921 & 3920 (98\%) & 0.04087 & 0.04091 \\
\bottomrule
\end{tabular}%
}
\end{table}

The validation-selected checkpoints have lower test RMSE than the final checkpoints in these runs, with larger absolute RMSE differences on Cylinder Flow than on the other two scenarios. Table~\ref{tab:main} reports test metrics at the validation-selected checkpoints. These comparisons assess checkpoint choice within the fixed budget; they do not establish convergence or performance under a longer training schedule.

\subsection{Effective Batch-Size Control}
\label{app:robustness_batch}

For DeepONet on Cylinder Flow, we repeat all four methods from the same pretrained checkpoint with effective batch size 16, using four accumulation steps of batch size 4. Learning rate, optimizer, update count, validation protocol and seed 0 remain unchanged. Sample exposures increase from 16{,}000 to 64{,}000 (Table~\ref{tab:batch_control}).

\begin{table}[H]
\centering
\caption{DeepONet batch-size control on Cylinder Flow (seed 0; effective $b{=}16$ via accumulation). Best minimises validation RMSE, final denotes step 4{,}000, and $b{=}16$ uses best checkpoints; bold marks column minima.}
\label{tab:batch_control}
\footnotesize
\setlength{\tabcolsep}{5pt}
\begin{tabular}{l ccc cc}
\toprule
\multirow{2}{*}{Method} & \multicolumn{3}{c}{RMSE\,$\downarrow$} & \multicolumn{2}{c}{Low-$f$\,$\downarrow$} \\
\cmidrule(lr){2-4} \cmidrule(lr){5-6}
& best ($b{=}4$) & final ($b{=}4$) & $b{=}16$ & $b{=}4$ & $b{=}16$ \\
\midrule
DFT & 0.09309 & 0.09426 & 0.09091 & 0.02527 & 0.02473 \\
EWC & 0.09055 & 0.09607 & 0.08868 & 0.02599 & 0.02552 \\
L$_2$-SP & 0.09121 & 0.09237 & 0.08950 & 0.02627 & 0.02583 \\
PhysGuard & \textbf{0.07435} & \textbf{0.07512} & \textbf{0.07308} & \textbf{0.01957} & \textbf{0.01941} \\
\bottomrule
\end{tabular}
\end{table}

For DeepONet on Cylinder Flow, the larger effective batch modestly improves all four methods while preserving their order (PhysGuard $<$ EWC $<$ L$_2$-SP $<$ DFT). PhysGuard retains its advantage in RMSE and Low-$f$ error under both tested settings.

\subsection{Official Pretrained Checkpoint Control}
\label{app:robustness_ckpt}

Our retrained checkpoints and smaller fine-tuning budget differ from the original RealPDEBench settings. For DeepONet on Cylinder Flow, the published DFT results are RMSE 0.0661 and $R^2$ 0.718, compared with 0.09309 and 0.5371 in Table~\ref{tab:main}. We therefore also fine-tune the official pretrained checkpoint under our shared protocol (Table~\ref{tab:official_ckpt}).

\begin{table}[H]
\centering
\caption{Fine-tuning the official RealPDEBench DeepONet checkpoint on Cylinder Flow.}
\label{tab:official_ckpt}
\footnotesize
\setlength{\tabcolsep}{6pt}
\begin{tabular}{l cccc}
\toprule
Method & RMSE\,$\downarrow$ & $R^2$\,$\uparrow$ & fRMSE\,$\downarrow$ & Low-$f$\,$\downarrow$ \\
\midrule
DFT (official checkpoint) & 0.06613 & 0.71811 & 0.01067 & 0.02661 \\
PhysGuard (official checkpoint) & \textbf{0.05360} & \textbf{0.81450} & \textbf{0.00880} & \textbf{0.02267} \\
\midrule
Relative change (DFT $\rightarrow$ PhysGuard) & $-$18.9\% & $+$13.4\% & $-$17.5\% & $-$14.8\% \\
\bottomrule
\end{tabular}
\end{table}

With the official checkpoint, DFT closely reproduces the published RMSE and $R^2$. PhysGuard improves all four metrics, including an 18.9\% RMSE reduction. This control supports its relative advantage for this checkpoint and configuration.

\newpage
\subsection{Relative Low-Frequency Error}
\label{app:relative_low_frequency}

Table~\ref{tab:relative_low_frequency} extends the Cylinder Flow comparison in Figure~\ref{fig:spectrum}(a) to all three scenarios.

\begin{table}[H]
\centering
\caption{Relative low-frequency error across architectures and scenarios.}
\label{tab:relative_low_frequency}
\footnotesize
\setlength{\tabcolsep}{4pt}
\renewcommand{\arraystretch}{1.15}
\resizebox{\textwidth}{!}{%
\begin{tabular}{ll ccc rrr}
\toprule
\multirow{2}{*}{Scenario} & \multirow{2}{*}{Architecture} & \multicolumn{3}{c}{Relative low-frequency error\,$\downarrow$} & \multicolumn{3}{c}{Relative change} \\
\cmidrule(lr){3-5} \cmidrule(lr){6-8}
& & Pretrained & DFT & PhysGuard & \shortstack{PhysGuard\\vs DFT} & \shortstack{PhysGuard\\vs Pretrained} & \shortstack{DFT\\vs Pretrained} \\
\midrule
\multirow{4}{*}{Cylinder Flow} & FNO & 0.690 & 0.794 & 0.509 & $-35.9$\% & $-26.1$\% & $+15.1$\% \\
 & CNO & 0.532 & 0.657 & 0.386 & $-41.3$\% & $-27.4$\% & $+23.6$\% \\
 & DeepONet & 0.862 & 0.837 & 0.724 & $-13.6$\% & $-16.0$\% & $-2.8$\% \\
 & Transolver & 0.824 & 0.711 & 0.651 & $-8.4$\% & $-21.0$\% & $-13.7$\% \\
\midrule
\multirow{4}{*}{\shortstack[l]{Controlled\\Cylinder}} & FNO & 0.266 & 0.057 & 0.058 & $+0.7$\% & $-78.3$\% & $-78.5$\% \\
 & CNO & 0.103 & 0.050 & 0.048 & $-4.4$\% & $-53.6$\% & $-51.5$\% \\
 & DeepONet & 0.623 & 0.221 & 0.219 & $-0.9$\% & $-64.8$\% & $-64.5$\% \\
 & Transolver & 0.426 & 0.117 & 0.114 & $-2.1$\% & $-73.2$\% & $-72.7$\% \\
\midrule
\multirow{4}{*}{\shortstack[l]{Turbulent\\Combustion}} & FNO & 0.566 & 0.297 & 0.294 & $-1.2$\% & $-48.1$\% & $-47.5$\% \\
 & CNO & 0.725 & 0.339 & 0.326 & $-4.0$\% & $-55.1$\% & $-53.2$\% \\
 & DeepONet & 0.651 & 0.322 & 0.318 & $-1.4$\% & $-51.2$\% & $-50.5$\% \\
 & Transolver & 0.763 & 0.587 & 0.540 & $-7.9$\% & $-29.2$\% & $-23.1$\% \\
\bottomrule
\end{tabular}%
}
\par\smallskip
\begin{minipage}{\textwidth}
\footnotesize\raggedright
Errors are shown to three decimal places. Percentage changes are computed from unrounded errors relative to the second named method; negative values indicate lower error.
\end{minipage}
\end{table}

PhysGuard has lower relative low-frequency error than the pretrained model in all twelve configurations and lower error than DFT in eleven. The exception is FNO on Controlled Cylinder, where its error is 0.7\% higher than DFT.

\newpage
\section{Subspace Selection and Orthogonality}
\label{app:orthogonality}

We test how subspace selection affects preservation, whether parameter constraints reduce output drift, and whether strict parameter orthogonality provides additional benefit. All experiments use Cylinder Flow with FNO, CNO, DeepONet and Transolver, following the same fine-tuning protocol (4{,}000 steps, the same optimiser and learning rate, seed 0). Orthogonality of projected gradients does not guarantee orthogonality of optimiser steps or output changes.

\subsection{Matched-Rank Random-Subspace Control}

We replace each layer's Fisher basis $\mU^{(m)}$ with a fixed random orthonormal basis of the same dimension $k_m$. The pretrained model, projection, training steps and learning rate remain unchanged. This extends the DeepONet control in Table~\ref{tab:extra_baselines} to four architectures. We measure simulation-domain forgetting by
\begin{equation}
\mathrm{Forgetting} = \frac{\mathrm{RMSE}_{\mathrm{sim}}^{\mathrm{after\ FT}} - \mathrm{RMSE}_{\mathrm{sim}}^{\mathrm{pretrained}}}{\mathrm{RMSE}_{\mathrm{sim}}^{\mathrm{pretrained}}},
\label{eq:orth_forgetting}
\end{equation}
using the same simulation test set before and after fine-tuning. Lower values indicate better retention.

\begin{table}[H]
\centering
\caption{Matched-rank random-subspace control on Cylinder Flow. Best values are bold.}
\label{tab:random_subspace}
\footnotesize
\setlength{\tabcolsep}{5pt}
\begin{tabular}{llccc}
\toprule
Architecture & Method & Real-domain RMSE\,$\downarrow$ & Low-$f$\,$\downarrow$ & Sim-domain forgetting (\%)\,$\downarrow$ \\
\midrule
\multirow{3}{*}{FNO} & DFT & 0.06487 & 0.01651 & 15.3 \\
 & Random subspace & 0.06491 & 0.01648 & 15.5 \\
 & PhysGuard & \textbf{0.06036} & \textbf{0.01131} & \textbf{6.1} \\
\midrule
\multirow{3}{*}{CNO} & DFT & 0.04988 & 0.01585 & 12.8 \\
 & Random subspace & 0.04983 & 0.01591 & 12.6 \\
 & PhysGuard & \textbf{0.04635} & \textbf{0.01148} & \textbf{5.2} \\
\midrule
\multirow{3}{*}{DeepONet} & DFT & 0.09309 & 0.02527 & 28.7 \\
 & Random subspace & 0.09114 & 0.02441 & 27.5 \\
 & PhysGuard & \textbf{0.07435} & \textbf{0.01957} & \textbf{13.4} \\
\midrule
\multirow{3}{*}{Transolver} & DFT & 0.11236 & 0.03349 & 22.1 \\
 & Random subspace & 0.11229 & 0.03354 & 21.8 \\
 & PhysGuard & \textbf{0.10695} & \textbf{0.03097} & \textbf{11.8} \\
\bottomrule
\end{tabular}
\end{table}

The random-subspace results are close to DFT, whereas PhysGuard lowers simulation-domain forgetting, real-domain RMSE and low-frequency error on all four architectures (Table~\ref{tab:random_subspace}). Under this matched control, the results support the role of Fisher-based direction selection in preservation.

\subsection{Output-Space Probe and Leakage}

Layer-wise parameter orthogonality need not translate to output orthogonality. We therefore evaluate output changes on held-out simulation samples.

\paragraph{Probe definition.}
We slightly perturb the two largest protected tensors along their dominant Fisher directions. Orthonormalising the resulting output changes gives a basis $\boldsymbol{Q}$ with 13--40 directions. For an update with output change $\Delta f$ on the same inputs, we define
\begin{equation}
\mathrm{Leakage}(\Delta f) = \frac{\|\boldsymbol{Q}^{\T}\Delta f\|_2^2}{\|\Delta f\|_2^2},
\label{eq:orth_leakage}
\end{equation}
which measures overlap with these sampled output modes. Perturbations along protected Fisher directions place 96--100\% of their output-change energy inside $\boldsymbol{Q}$. This probe measures preservation relative to $\boldsymbol{Q}$ rather than every physical property of the output.

\paragraph{Output changes after fine-tuning.}
Table~\ref{tab:output_change} compares total prediction changes and their protection-related components overlapping $\boldsymbol{Q}$. Both are measured on the same simulation inputs and normalised by the corresponding DFT value.

\begin{table}[H]
\centering
\caption{Output changes on Cylinder Flow, normalized to DFT ($=1$).}
\label{tab:output_change}
\small
\setlength{\tabcolsep}{8pt}
\begin{tabular}{lcc}
\toprule
Model & Total output change & Protection-related change \\
\midrule
FNO & 0.78 & 0.71 \\
CNO & 0.85 & 0.75 \\
DeepONet & 0.80 & 0.49 \\
Transolver & 0.83 & 0.50 \\
\bottomrule
\end{tabular}
\end{table}

The protection-related change decreases more than the total change on all four architectures. For DeepONet, the reductions are 51\% and 20\%, respectively. This indicates preferential suppression of changes along the probed output modes.

\paragraph{Output leakage.}
We compare random updates orthogonal to $\mU^{(m)}$ in every layer, the actual PhysGuard update, and DFT. The random updates test whether layer-wise parameter orthogonality persists at the output. Their ranges reflect repeated sampling.

\begin{table}[H]
\centering
\caption{Leakage into sampled output modes on Cylinder Flow; ranges are over repeated random updates.}
\label{tab:leakage}
\small
\setlength{\tabcolsep}{8pt}
\begin{tabular}{lccc}
\toprule
Model & Layer-wise-orthogonal random & PhysGuard update & Direct fine-tuning \\
\midrule
FNO & 0.004--0.012 & 0.005 & 0.006 \\
CNO & 0.15--0.24 & 0.14 & 0.18 \\
DeepONet & 0.35--0.85 & 0.28 & 0.76 \\
Transolver & 0.18--0.44 & 0.20 & 0.56 \\
\bottomrule
\end{tabular}
\end{table}

Nonzero leakage from the random updates confirms that the layer-wise constraint does not guarantee output orthogonality. PhysGuard has lower leakage than DFT in all four cases, although leakage remains nonzero (Table~\ref{tab:leakage}).

\subsection{Parameter Changes and Output Drift}

Each run saves a checkpoint every 80 steps, giving 50 checkpoints over 4{,}000 steps.

\paragraph{Cumulative parameter change.}
Relative to the pretrained parameters $\theta_0=\theta^*$, we compute
\begin{equation}
\Delta\theta_t = \theta_t - \theta_0, \qquad
r_t = \frac{\|\mU^{\T}\Delta\theta_t\|_2^2}{\|\Delta\theta_t\|_2^2},
\label{eq:orth_ratio}
\end{equation}
where $\mU$ collects the per-layer bases block-diagonally. The ratio is the protected share of the cumulative parameter change and is reported in percent.

\begin{table}[H]
\centering
\caption{Protected share $r_t$ of cumulative parameter changes across 50 PhysGuard checkpoints on Cylinder Flow.}
\label{tab:constraint_param}
\small
\setlength{\tabcolsep}{8pt}
\begin{tabular}{lccc}
\toprule
Architecture & Mean (\%)\,$\downarrow$ & Max (\%)\,$\downarrow$ & Final (\%)\,$\downarrow$ \\
\midrule
FNO & 0.06 & 0.15 & 0.04 \\
CNO & 0.09 & 0.21 & 0.07 \\
DeepONet & 0.33 & 0.62 & 0.27 \\
Transolver & 0.87 & 1.45 & 0.81 \\
\bottomrule
\end{tabular}
\end{table}

The protected share remains below 1.5\% at all sampled checkpoints (Table~\ref{tab:constraint_param}). These small but nonzero values also distinguish gradient projection from a strict constraint on the accumulated parameter change.

\paragraph{Output drift.}
On a fixed simulation evaluation set $X_{\mathrm{sim}}$, we compute
\begin{equation}
D_t = \frac{\|f_{\theta_t}(X_{\mathrm{sim}}) - f_{\theta_0}(X_{\mathrm{sim}})\|_2}{\|f_{\theta_0}(X_{\mathrm{sim}})\|_2}, \qquad
R_t = \frac{D_t^{\mathrm{PhysGuard}}}{D_t^{\mathrm{DFT}}},
\label{eq:orth_drift}
\end{equation}
where $R_t<1$ indicates less output drift under PhysGuard than DFT.

\begin{table}[H]
\centering
\caption{Function-space drift relative to DFT on Cylinder Flow (mean over 50 checkpoints and final value).}
\label{tab:constraint_func}
\small
\setlength{\tabcolsep}{8pt}
\begin{tabular}{lccc}
\toprule
Architecture & Checkpoints with $R_t<1$ & Mean $R_t$\,$\downarrow$ & Final $R_t$\,$\downarrow$ \\
\midrule
FNO & 30/50 & 0.91 & 0.78 \\
CNO & 43/50 & 0.78 & 0.85 \\
DeepONet & 46/50 & 0.73 & 0.80 \\
Transolver & 38/50 & 0.83 & 0.83 \\
\bottomrule
\end{tabular}
\end{table}

PhysGuard has less output drift at most checkpoints, with mean and final ratios below one on all four architectures (Table~\ref{tab:constraint_func}). This reduction does not hold at every checkpoint. The final ratios match the total output changes in Table~\ref{tab:output_change}.

\subsection{Post-hoc Orthogonality Correction}

At the best-validation checkpoint, we apply a one-time correction
\begin{equation}
\theta_{\mathrm{corr}} = \theta_0 + (\mI - \mU\mU^{\T})(\theta_{\mathrm{FT}} - \theta_0),
\label{eq:orth_posthoc}
\end{equation}
independently to each layer. This removes the protected component of the cumulative parameter change. Table~\ref{tab:posthoc} reports the result for PhysGuard and DFT.

\begin{table}[H]
\centering
\caption{Post-hoc orthogonality correction on Cylinder Flow at the best-validation checkpoint; ``Protected'' is the protected share of cumulative parameter change.}
\label{tab:posthoc}
\small
\setlength{\tabcolsep}{5pt}
\resizebox{\textwidth}{!}{%
\begin{tabular}{lllcccc}
\toprule
Architecture & Training & Correction & Protected before\,$\downarrow$ & Protected after\,$\downarrow$ & Real RMSE\,$\downarrow$ & Forgetting (\%)\,$\downarrow$ \\
\midrule
\multirow{4}{*}{FNO} & DFT & None & 7.6\% & --- & 0.06487 & 15.3 \\
 & DFT & Fisher-$\mU$ & 7.6\% & $\approx 0$ & 0.06522 & 14.9 \\
 & PhysGuard & None & 0.05\% & --- & 0.06036 & 6.1 \\
 & PhysGuard & Fisher-$\mU$ & 0.05\% & $\approx 0$ & 0.06037 & 6.1 \\
\midrule
\multirow{4}{*}{CNO} & DFT & None & 5.8\% & --- & 0.04988 & 12.8 \\
 & DFT & Fisher-$\mU$ & 5.8\% & $\approx 0$ & 0.05021 & 12.4 \\
 & PhysGuard & None & 0.08\% & --- & 0.04635 & 5.2 \\
 & PhysGuard & Fisher-$\mU$ & 0.08\% & $\approx 0$ & 0.04636 & 5.2 \\
\midrule
\multirow{4}{*}{DeepONet} & DFT & None & 11.3\% & --- & 0.09309 & 28.7 \\
 & DFT & Fisher-$\mU$ & 11.3\% & $\approx 0$ & 0.09372 & 28.0 \\
 & PhysGuard & None & 0.31\% & --- & 0.07435 & 13.4 \\
 & PhysGuard & Fisher-$\mU$ & 0.31\% & $\approx 0$ & 0.07436 & 13.4 \\
\midrule
\multirow{4}{*}{Transolver} & DFT & None & 9.1\% & --- & 0.11236 & 22.1 \\
 & DFT & Fisher-$\mU$ & 9.1\% & $\approx 0$ & 0.11283 & 21.5 \\
 & PhysGuard & None & 0.92\% & --- & 0.10695 & 11.8 \\
 & PhysGuard & Fisher-$\mU$ & 0.92\% & $\approx 0$ & 0.10696 & 11.8 \\
\bottomrule
\end{tabular}%
}
\end{table}

The correction barely changes PhysGuard's reported RMSE and leaves its forgetting unchanged. For DFT, it removes the protected parameter component but only slightly reduces forgetting and slightly increases real-domain RMSE. Strict parameter orthogonality therefore does not recover function-space preservation in this experiment.

\newpage
\section{Qualitative Visualisations}
\label{app:visualization}

In this section, we provide the full visualization of results for FNO, CNO, DeepONet, and Transolver across all three datasets, complementing the quantitative results reported in Table~\ref{tab:main}.

\subsection{Cylinder Flow}

\begin{figure}[H]
  \centering
  \includegraphics[width=0.66\linewidth]{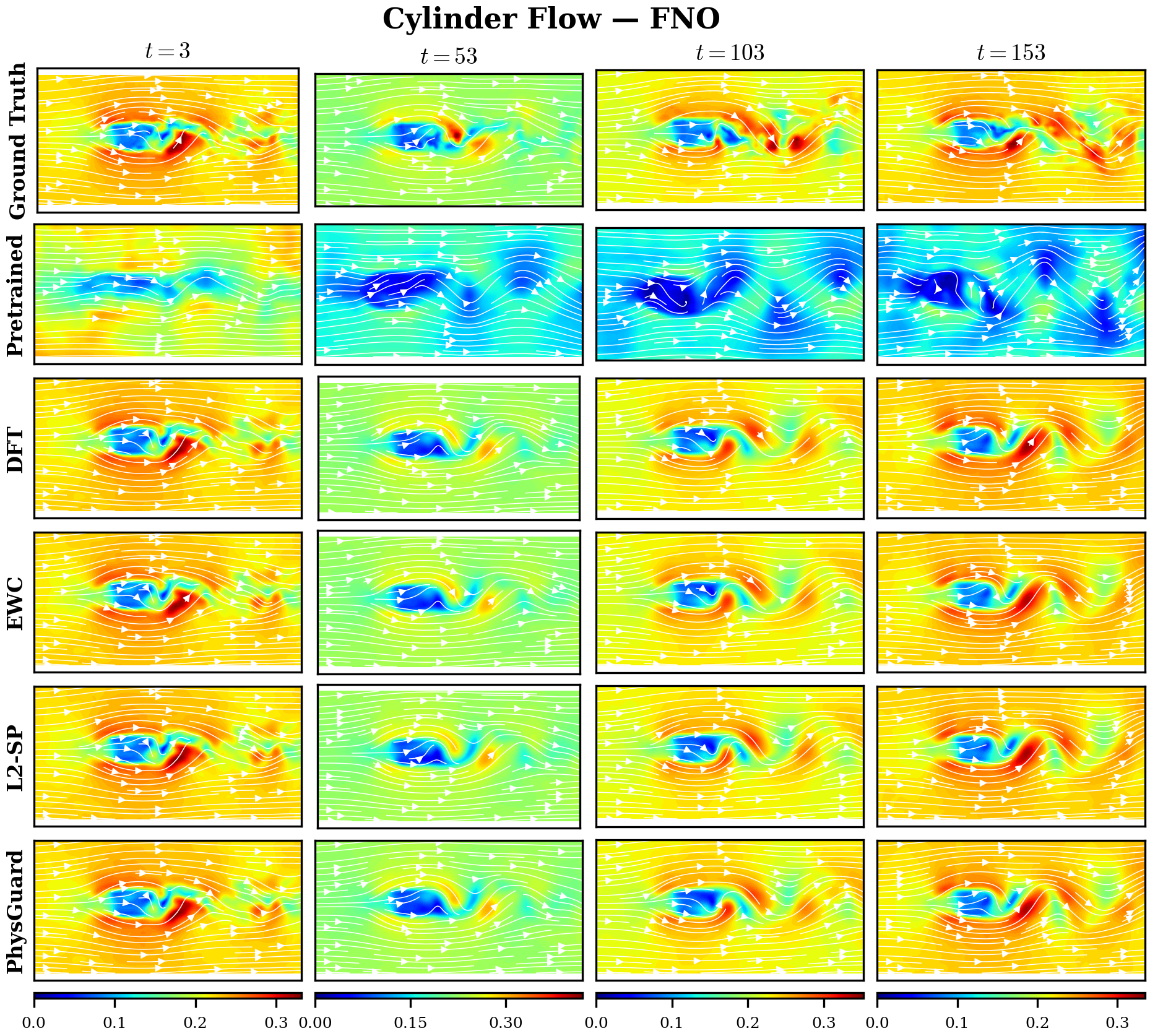}
  \caption{Cylinder Flow -- FNO predictions.}
  \label{fig:app_cylinder_fno}
\end{figure}
\begin{figure}[H]
  \centering
  \includegraphics[width=0.66\linewidth]{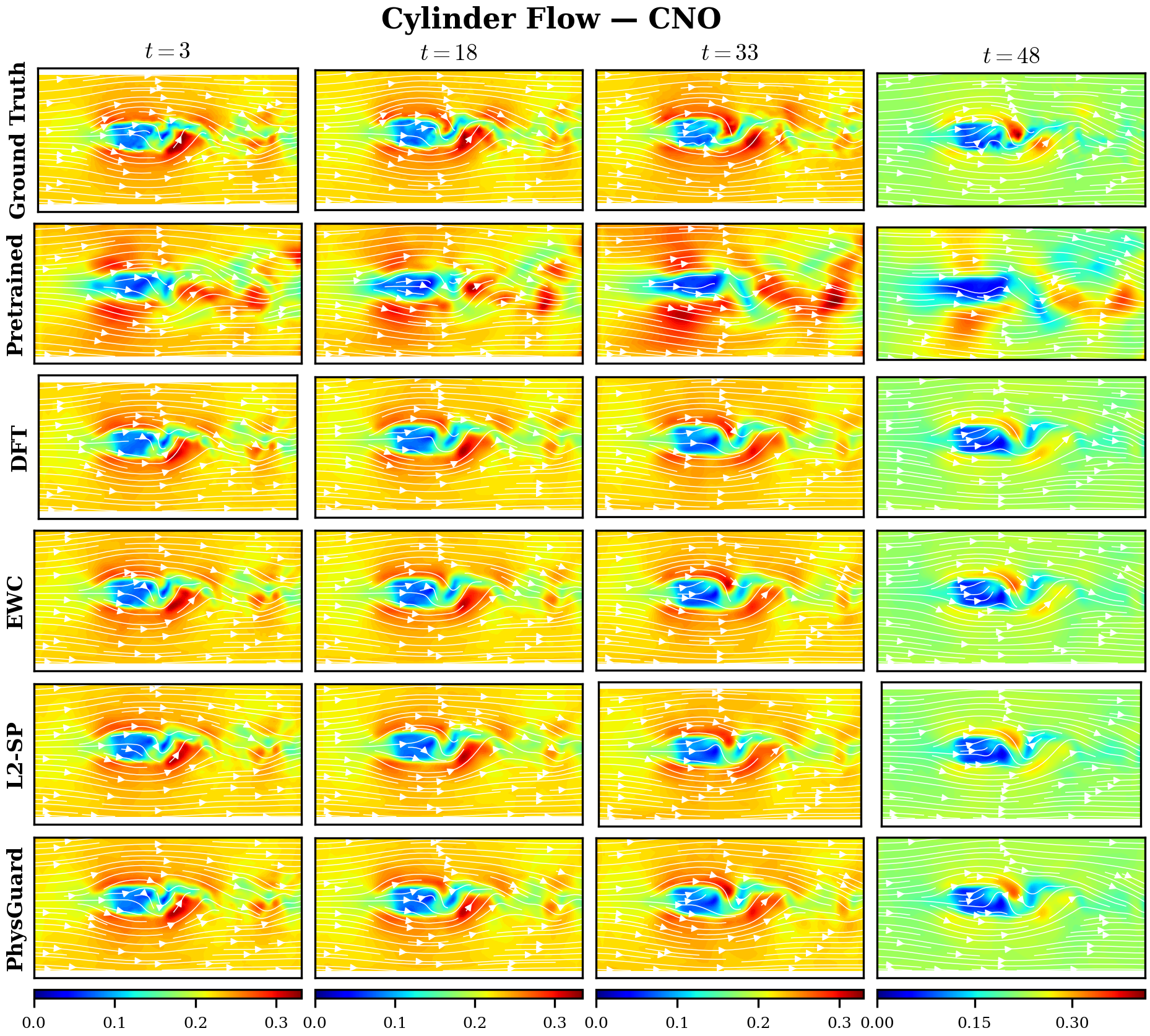}
  \caption{Cylinder Flow -- CNO predictions.}
  \label{fig:app_cylinder_cno}
\end{figure}

\newpage
\begin{figure}[H]
  \centering
  \includegraphics[width=0.66\linewidth]{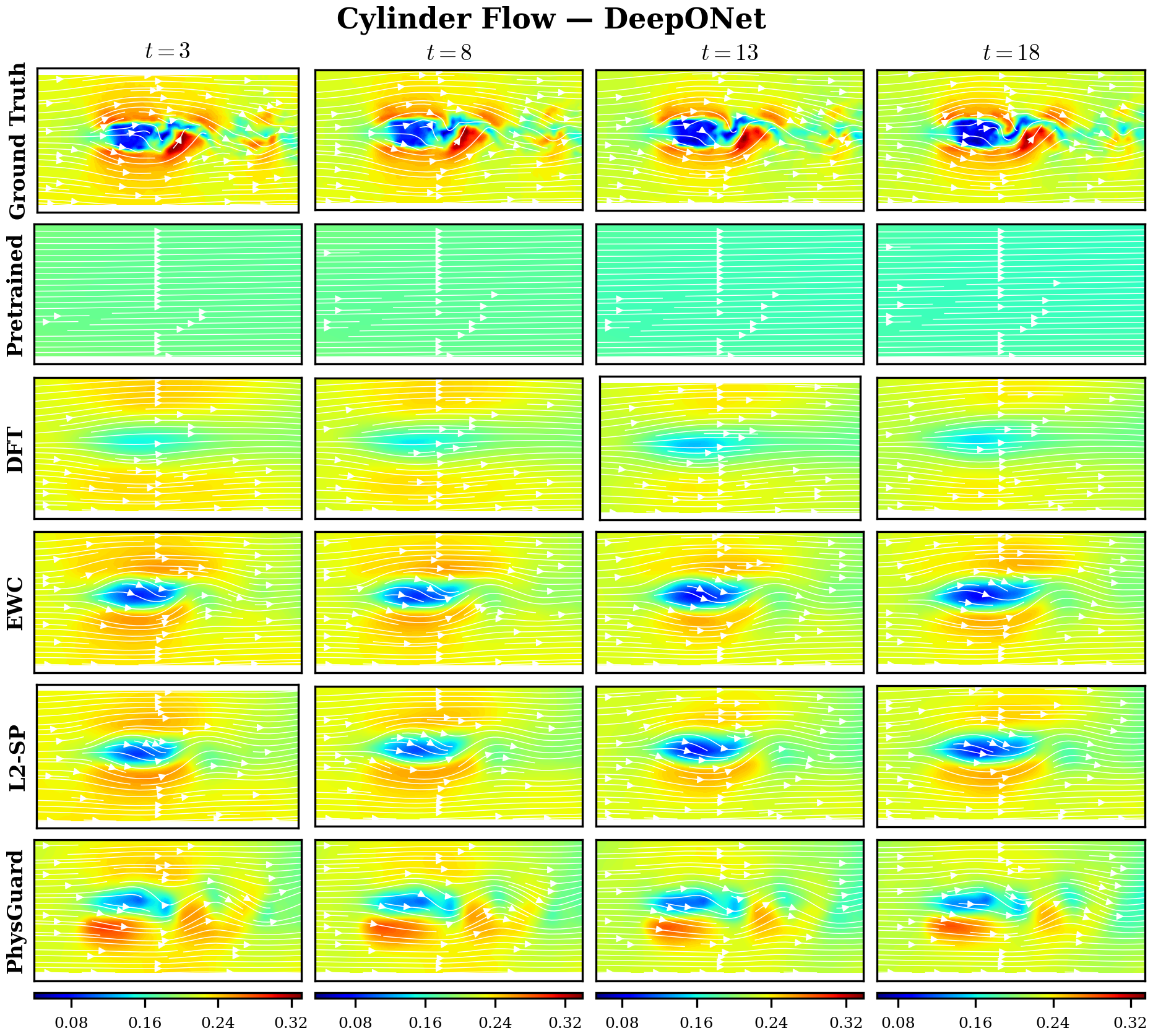}
  \caption{Cylinder Flow -- DeepONet predictions.}
  \label{fig:app_cylinder_deeponet}
\end{figure}
\begin{figure}[H]
  \centering
  \includegraphics[width=0.66\linewidth]{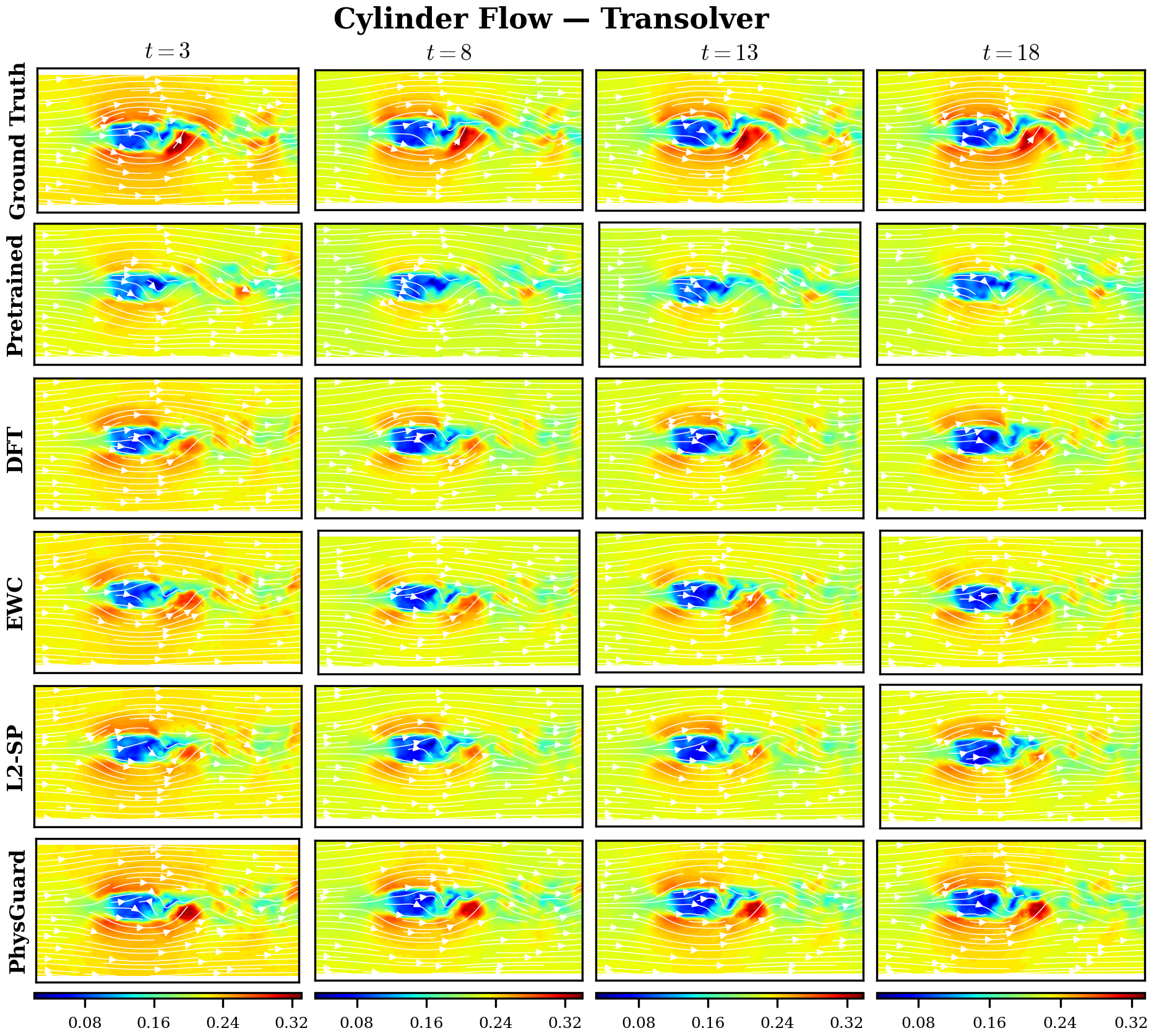}
  \caption{Cylinder Flow -- Transolver predictions.}
  \label{fig:app_cylinder_transolver}
\end{figure}

\newpage
\subsection{Controlled Cylinder}

\begin{figure}[H]
  \centering
  \includegraphics[width=0.66\linewidth]{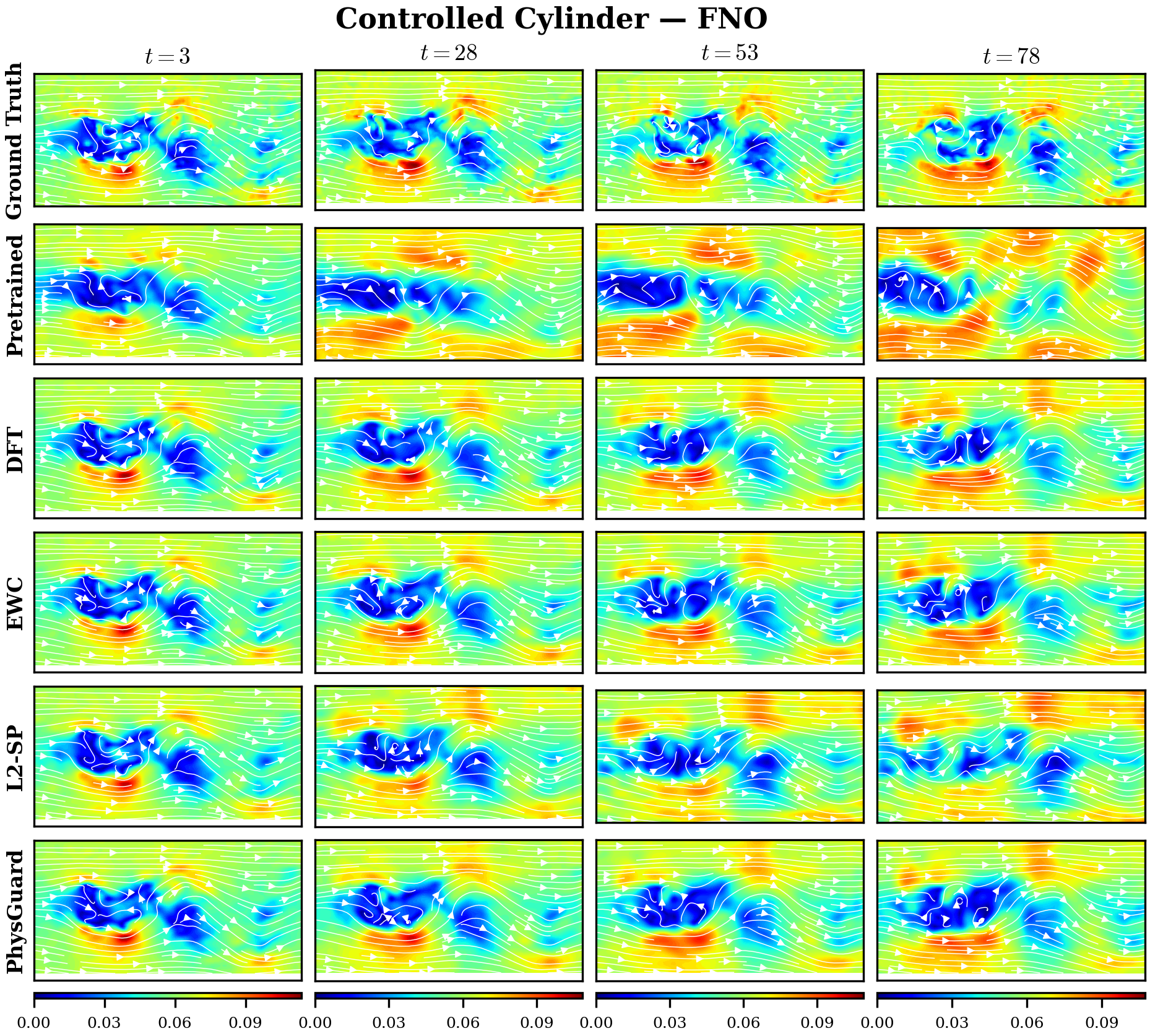}
  \caption{Controlled Cylinder Flow -- FNO predictions.}
  \label{fig:app_controlled_cylinder_fno}
\end{figure}
\begin{figure}[H]
  \centering
  \includegraphics[width=0.66\linewidth]{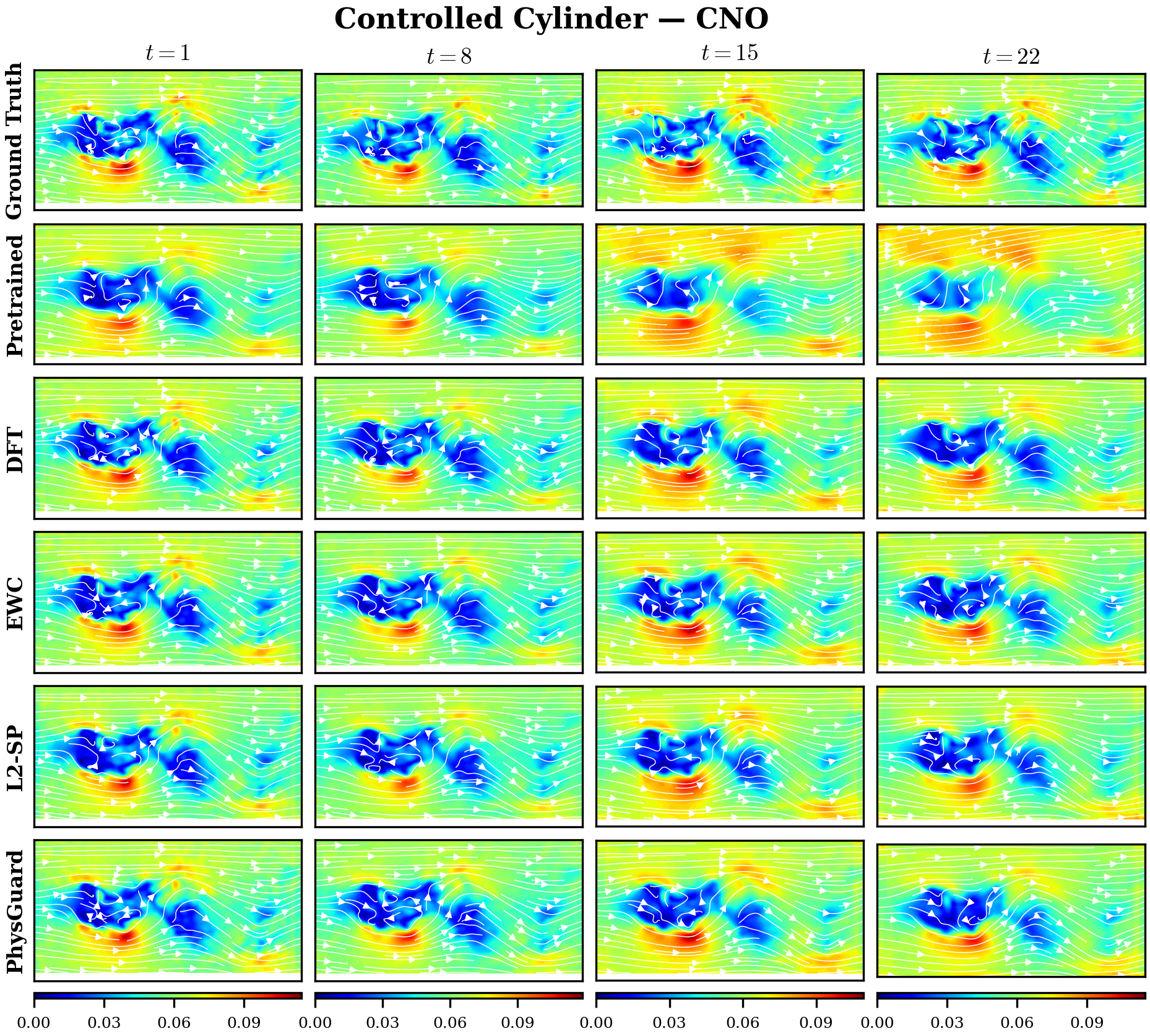}
  \caption{Controlled Cylinder Flow -- CNO predictions.}
  \label{fig:app_controlled_cylinder_cno}
\end{figure}

\newpage
\begin{figure}[H]
  \centering
  \includegraphics[width=0.66\linewidth]{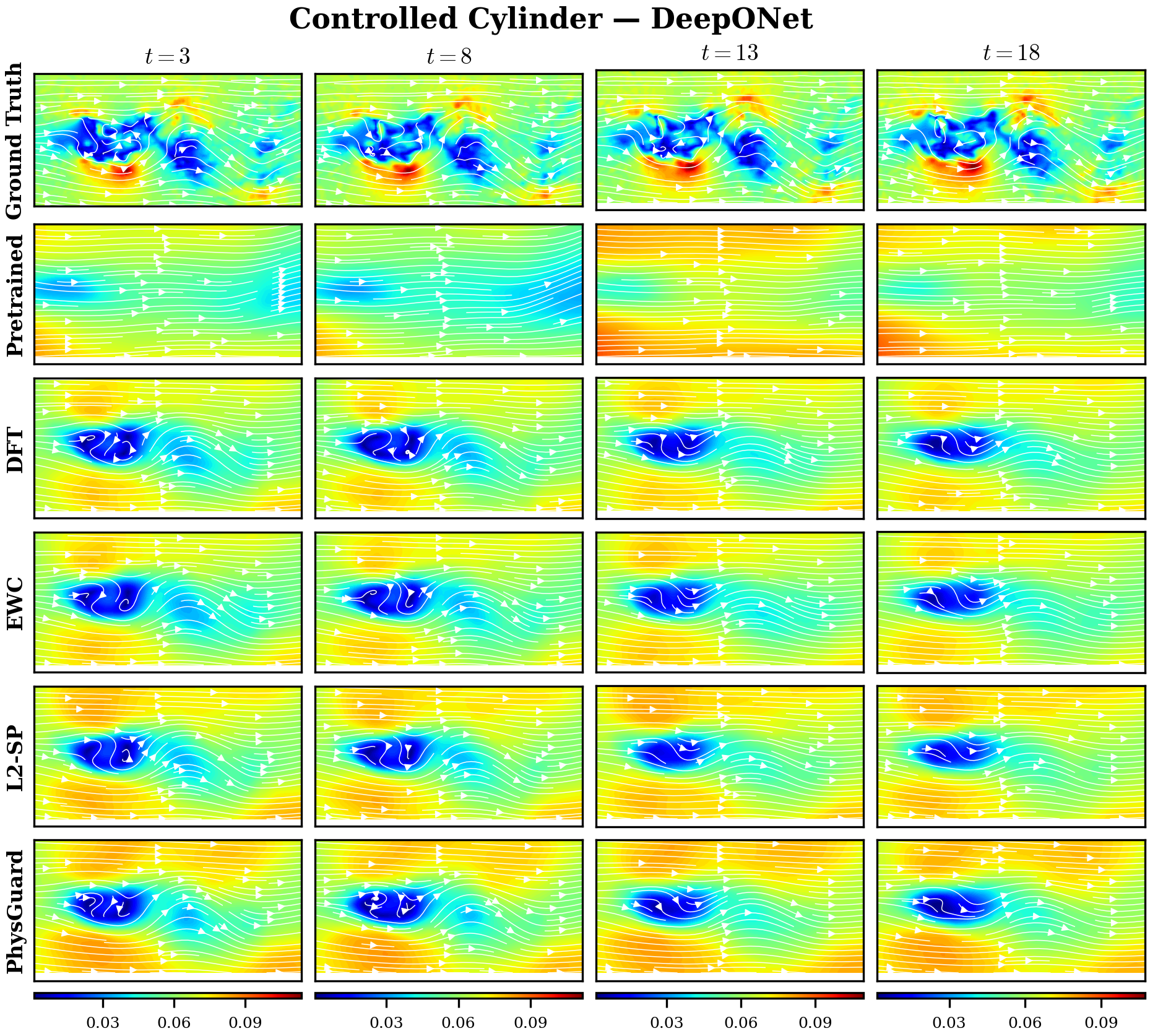}
  \caption{Controlled Cylinder Flow -- DeepONet predictions.}
  \label{fig:app_controlled_cylinder_deeponet}
\end{figure}
\begin{figure}[H]
  \centering
  \includegraphics[width=0.66\linewidth]{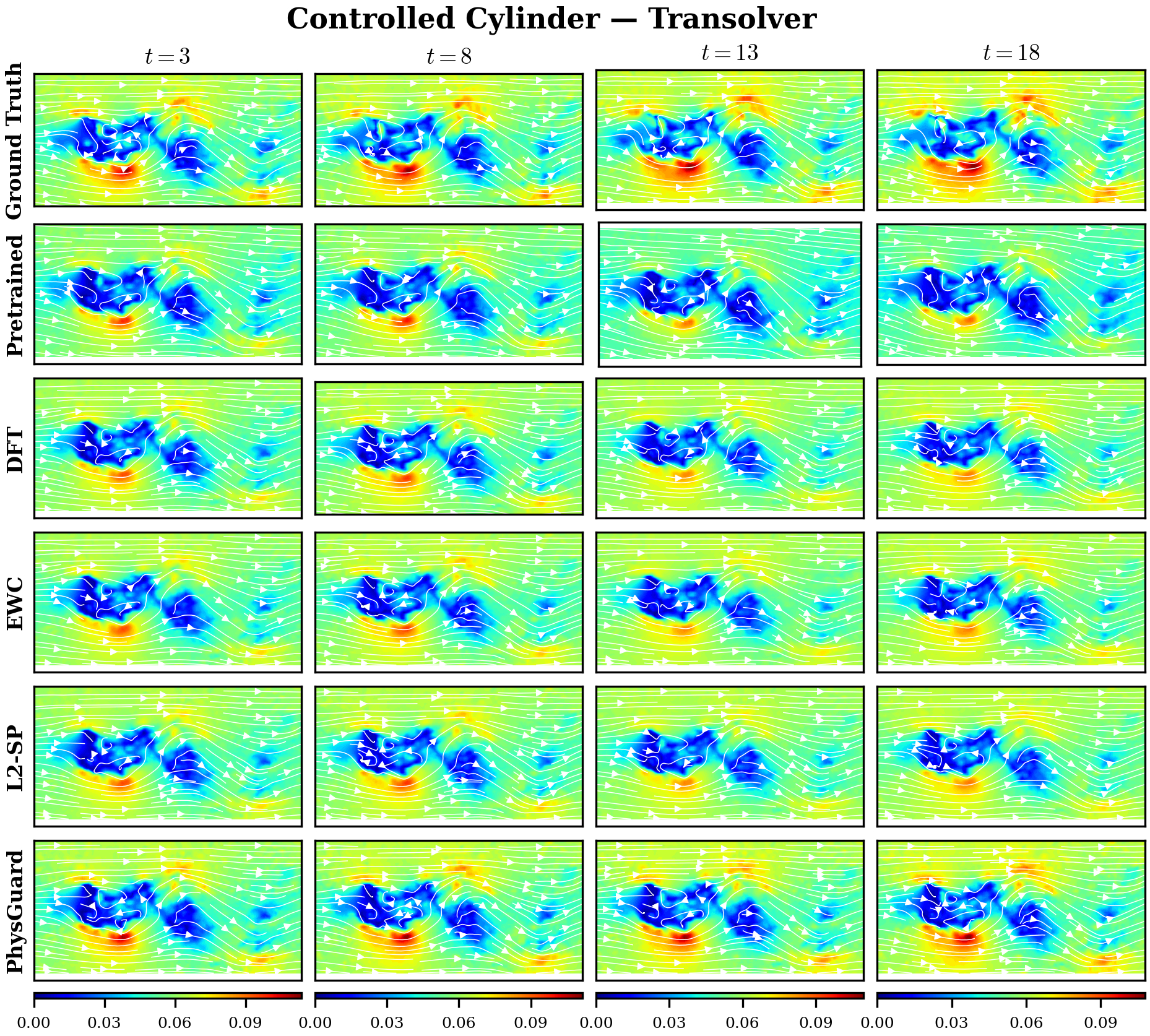}
  \caption{Controlled Cylinder Flow -- Transolver predictions.}
  \label{fig:app_controlled_cylinder_transolver}
\end{figure}

\newpage
\subsection{Turbulent Combustion}

\begin{figure}[H]
  \centering
  \begin{minipage}[t]{0.48\linewidth}
    \centering
    \includegraphics[width=\linewidth]{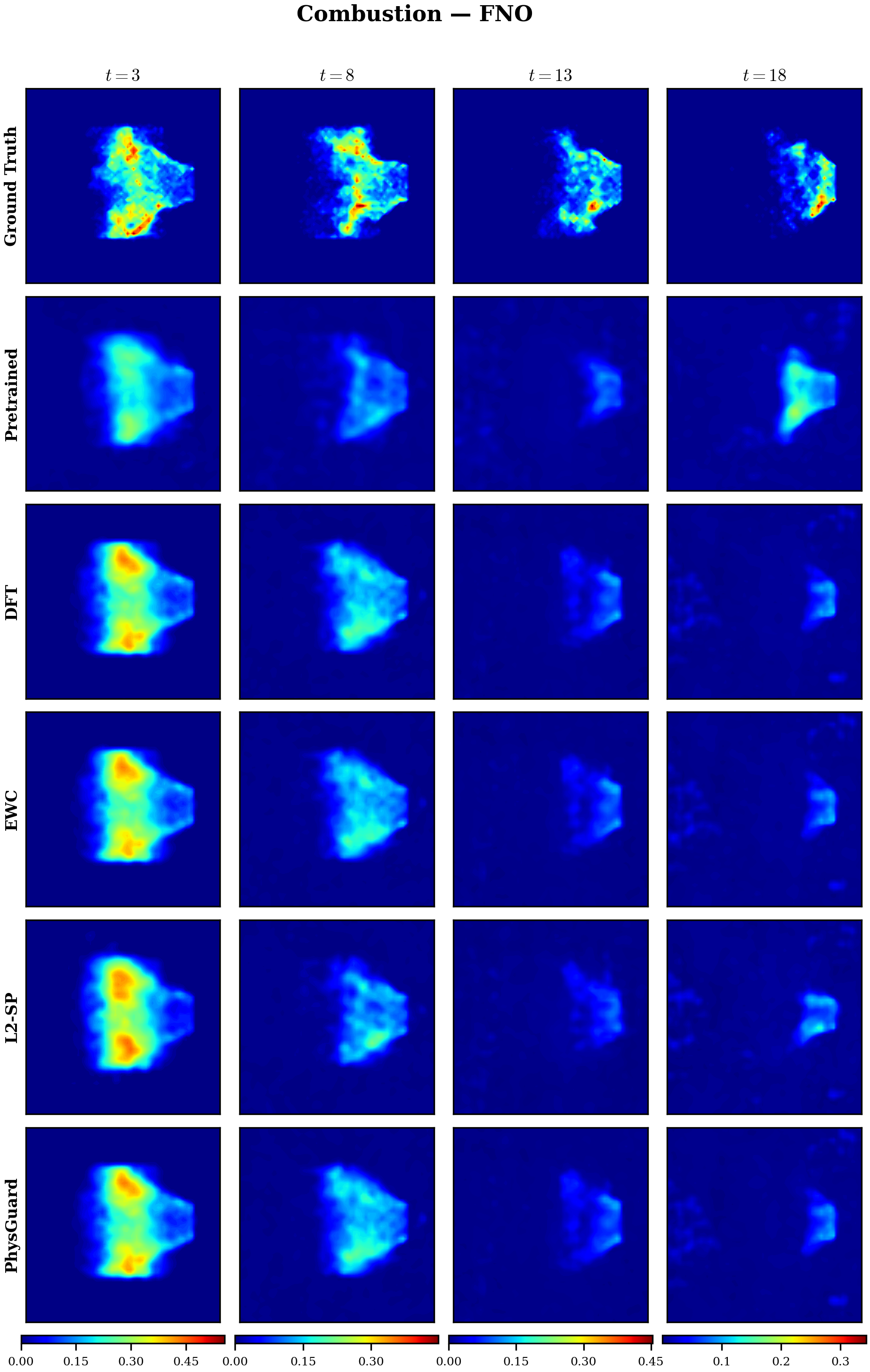}
    \caption{Turbulent Combustion -- FNO predictions.}
    \label{fig:app_combustion_fno}
  \end{minipage}\hfill
  \begin{minipage}[t]{0.48\linewidth}
    \centering
    \includegraphics[width=\linewidth]{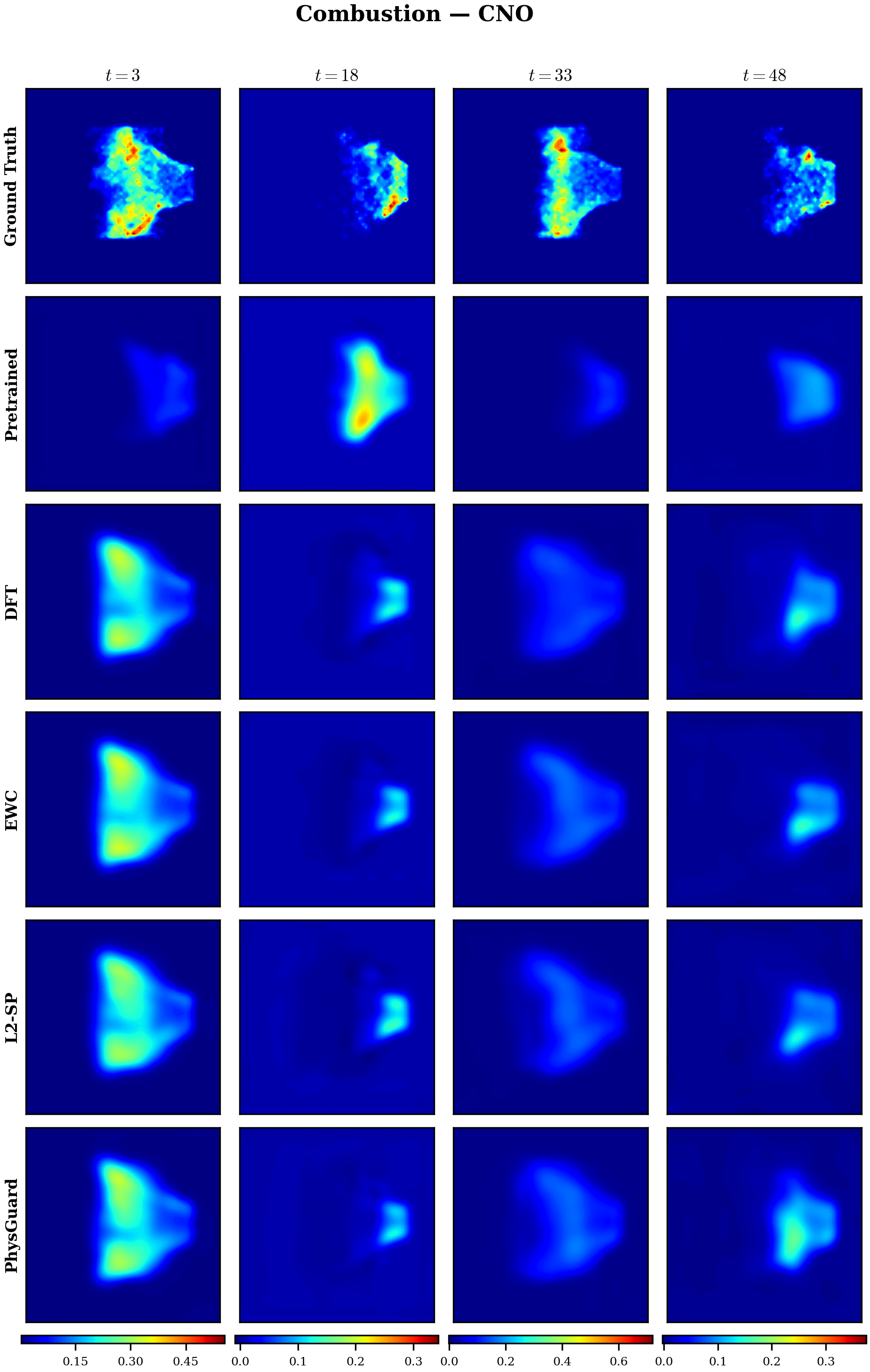}
    \caption{CNO predictions on Turbulent Combustion.}
    \label{fig:app_combustion_cno}
  \end{minipage}
\end{figure}

\newpage
\begin{figure}[H]
  \centering
  \begin{minipage}[t]{0.48\linewidth}
    \centering
    \includegraphics[width=\linewidth]{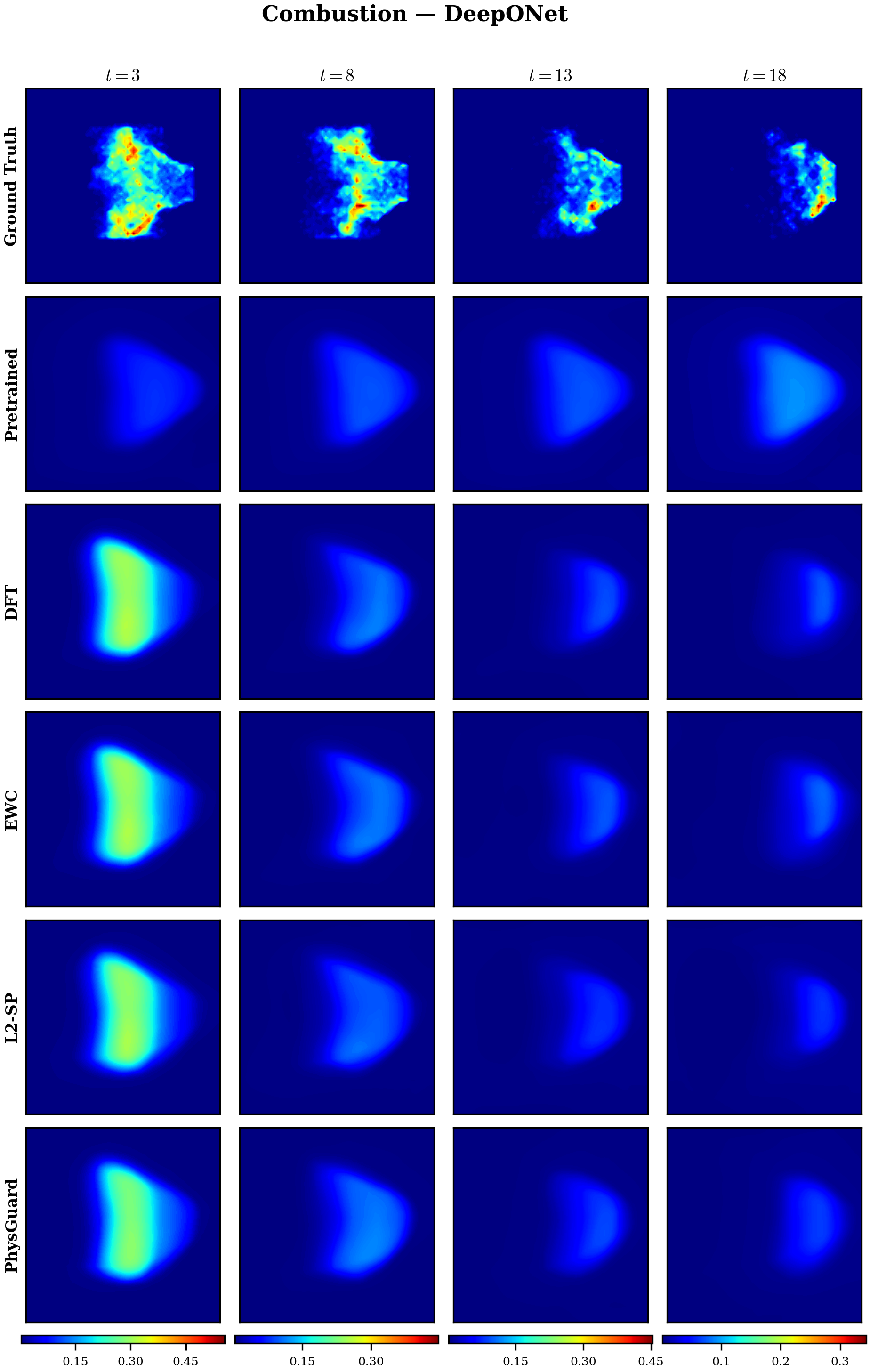}
    \caption{DeepONet predictions on Turbulent Combustion.}
    \label{fig:app_combustion_deeponet}
  \end{minipage}\hfill
  \begin{minipage}[t]{0.48\linewidth}
    \centering
    \includegraphics[width=\linewidth]{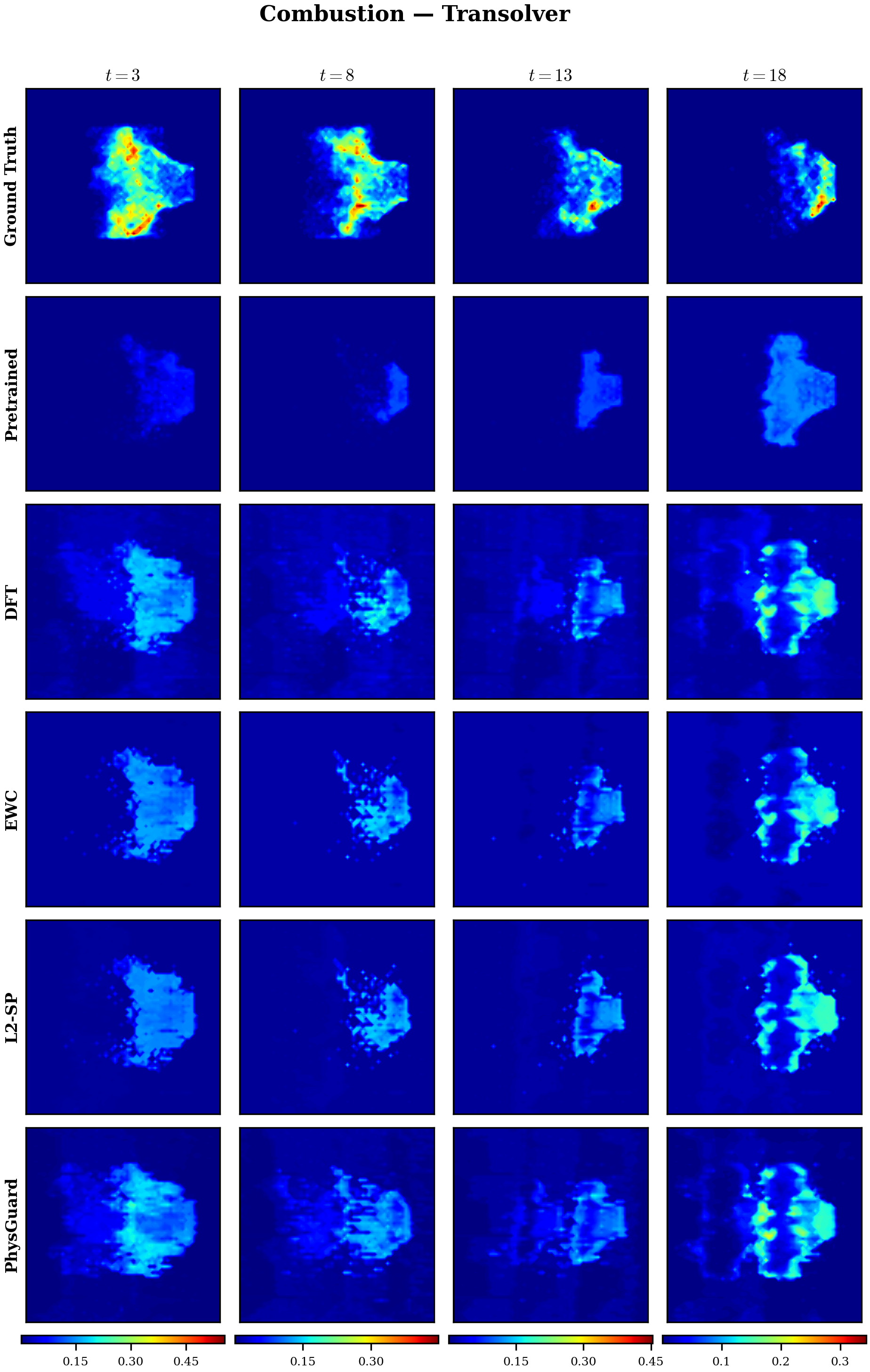}
    \caption{Transolver predictions on Turbulent Combustion.}
    \label{fig:app_combustion_transolver}
  \end{minipage}
\end{figure}

\newpage
\section{DPOT-S as a Negative Case}
\label{app:dpot}

PhysGuard does not consistently improve performance on the foundation-scale operator DPOT-S~\citep{hao2024dpot}, as noted in Section~\ref{sec:limitations}. This appendix examines this case in relation to the method's assumption of a \emph{low-rank} Fisher spectrum.

\subsection{Architecture}
\label{app:dpot_arch}

Denoising Pre-training Operator Transformer (DPOT) combines FNO-style spectral feature extraction with a scalable Transformer backbone. Input fields are projected onto a spectral basis, and the resulting tokens are processed by Diffusion-Transformer (DiT) style attention blocks~\citep{peebles2023scalable}. A shared backbone is pretrained across multiple PDE datasets and then fine-tuned for each task.

We use DPOT-S with 6 Transformer blocks, hidden dimension 1024, 8 attention heads and 41.3M parameters. Its Fisher spectrum at the checkpoint used for adaptation is less concentrated than those of the four main architectures (Appendix~\ref{app:dpot_fisher}).

\subsection{Pretraining Configuration}
\label{app:dpot_pretrain}

Table~\ref{tab:dpot_pretrain} reports the pretraining hyperparameters used for DPOT-S on each of
the three RealPDEBench scenarios. The overall protocol (AdamW with cosine annealing, BF16 AMP,
DDP across 4 RTX~4090 GPUs) matches the one described in Appendix~\ref{app:exp_details}; only the
architecture-specific and scenario-specific values differ.

\begin{table}[h]
\centering
\definecolor{PhHdr}{HTML}{FFFFFF}
\definecolor{PhAlt}{HTML}{FFFFFF}
\definecolor{PhRule}{HTML}{000000}
\caption{DPOT-S pretraining settings for the three RealPDEBench scenarios.}
\label{tab:dpot_pretrain}
\renewcommand{\arraystretch}{1.35}
\resizebox{\textwidth}{!}{%
\setlength{\tabcolsep}{7pt}
\begin{tabular}{l r l r c c r >{\footnotesize\raggedright\arraybackslash}p{7.4cm}}
\specialrule{1.6pt}{0pt}{0pt}
\rowcolor{PhHdr}
\textbf{\color{black}Scenario} &
\textbf{\color{black}Params} &
\textbf{\color{black}LR} &
\textbf{\color{black}Iters} &
\textbf{\color{black}Batch/GPU} &
\textbf{\color{black}GPUs} &
\textbf{\color{black}Eff.\,batch} &
\textbf{\color{black}Architecture hyperparameters} \\[-1pt]
\arrayrulecolor{PhRule}\specialrule{0.45pt}{0pt}{0pt}
Cylinder flow
  & 41.3\,M & $1\!\times\!10^{-3}$ &  4{,}000 &  4 & 4 & 16
  & \texttt{embed\_dim}=1024, \texttt{depth}=6, \texttt{n\_blocks}=8, \texttt{modes}=32, \texttt{patch\_size}=8;\enspace warm-start \texttt{model\_S.pth} \\
\rowcolor{PhAlt}
Controlled cylinder
  & 41.3\,M & $5\!\times\!10^{-4}$ &  4{,}000 &  4 & 4 & 16
  & \textit{Same architecture as Cylinder flow} \\
Turbulent combustion
  & 41.3\,M & $1\!\times\!10^{-3}$ & 10{,}000 &  4 & 4 & 16
  & \textit{Same architecture as Cylinder flow} \\[1pt]
\arrayrulecolor{black}\specialrule{1.6pt}{0pt}{0pt}
\end{tabular}%
}
\end{table}

\subsection{Fisher Spectrum across Training Stages}
\label{app:dpot_fisher}

For the four architectures evaluated in the main text, the empirical Fisher Information Matrix concentrates its energy in dominant directions (Table~\ref{tab:arch_comparison}). This allows PhysGuard to protect a compact subspace while retaining an orthogonal complement for adaptation.

DPOT-S requires a larger fraction of Fisher directions to capture 90\% of the Fisher mass at the checkpoint used for adaptation. We track the training stages to distinguish this observation from an attribution to joint pretraining or the DiT backbone.

We evaluate four stages of DPOT-S on Cylinder Flow with the same 200 simulation samples and layer-wise estimator. They are random initialization, the jointly pretrained checkpoint, and checkpoints after 80 and 4{,}000 steps of scenario-specific training. The last is the checkpoint used for adaptation. Table~\ref{tab:dpot_stages} reports the fraction of eigendirections needed to capture 90\% of Fisher mass, summarized by the median across layers and separately within DiT attention layers. A larger fraction indicates a flatter spectrum. Evenness measures how uniformly Fisher mass is distributed across the six DiT blocks, with 1.0 denoting uniformity.

\begin{table}[H]
\centering
\caption{DPOT-S Fisher-spectrum flatness across training stages on Cylinder Flow. Dashes indicate unreported values.}
\label{tab:dpot_stages}
\small
\setlength{\tabcolsep}{5pt}
\begin{tabular}{>{\raggedright\arraybackslash}p{6.0cm} ccc}
\toprule
Training stage & Median fraction & Within & Evenness across \\
 & for 90\% Fisher mass & DiT blocks & DiT blocks \\
\midrule
Random initialisation (control) & --- & 2.0\% & 0.88 \\
After joint pretraining, before scenario training & 5.5\% & 6.0\% & 0.81 \\
After 80 steps of scenario training & 5.5\% & 7.0\% & 0.87 \\
After 4{,}000 steps (evaluated checkpoint) & 13.0\% & 21.0\% & 0.97 \\
\bottomrule
\end{tabular}
\end{table}

The median fraction remains at 5.5\% through the first 80 steps and rises to 13.0\% at step 4{,}000. Within DiT blocks it rises from 6.0\% after joint pretraining to 21.0\%, while evenness reaches 0.97. Thus the pronounced flattening develops during scenario-specific training. These measurements identify when it appears, but do not isolate an architectural cause. The negative transfer results in Table~\ref{tab:dpot} are consistent with this checkpoint dependence; the spectrum should be assessed at the checkpoint used for adaptation.

\subsection{Quantitative Results}
\label{app:dpot_results}

Table~\ref{tab:dpot} reports DPOT-S performance across all three RealPDEBench scenarios using the
same baselines and metrics as Table~\ref{tab:main}. 

\begin{table}[h]
\centering
\caption{DPOT-S results on RealPDEBench. Blue and green mark the best and second-best values.}
\label{tab:dpot}
\scriptsize
\setlength{\tabcolsep}{1.6pt}
\begin{tabular}{ll cccc cccc cccc}
\toprule
\multirow{2}{*}{Architecture} & \multirow{2}{*}{Method} & \multicolumn{4}{c}{Cylinder Flow} & \multicolumn{4}{c}{Controlled Cylinder} & \multicolumn{4}{c}{Turbulent Combustion} \\
\cmidrule(lr){3-6} \cmidrule(lr){7-10} \cmidrule(lr){11-14}
& & RMSE\,$\downarrow$ & $R^2$\,$\uparrow$ & fRMSE\,$\downarrow$ & Low-$f$\,$\downarrow$ & RMSE\,$\downarrow$ & $R^2$\,$\uparrow$ & fRMSE\,$\downarrow$ & Low-$f$\,$\downarrow$ & RMSE\,$\downarrow$ & $R^2$\,$\uparrow$ & fRMSE\,$\downarrow$ & Low-$f$\,$\downarrow$ \\
\midrule
\multirow{5}{*}{\shortstack{\textbf{DPOT-S}\\ (41.3\,M)}} & Pretrained & 0.05663 & 0.8287 & 0.00838 & 0.01202 & 0.02187 & 0.9049 & 0.00333 & 0.00278 & 0.03957 & 0.1351 & 0.00549 & 0.00935 \\
  & DFT & \cellcolor[rgb]{ .886,  .937,  .855}0.04544 & \cellcolor[rgb]{ .886,  .937,  .855}0.8897 & \cellcolor[rgb]{ .886,  .937,  .855}0.00667 & 0.01026 & \cellcolor[rgb]{ .741,  .843,  .933}\textbf{0.00996} & \cellcolor[rgb]{ .886,  .937,  .855}0.9803 & \cellcolor[rgb]{ .886,  .937,  .855}0.00106 & \cellcolor[rgb]{ .886,  .937,  .855}0.00089 & \cellcolor[rgb]{ .741,  .843,  .933}\textbf{0.02668} & \cellcolor[rgb]{ .741,  .843,  .933}\textbf{0.6066} & \cellcolor[rgb]{ .741,  .843,  .933}\textbf{0.00350} & \cellcolor[rgb]{ .741,  .843,  .933}\textbf{0.00388} \\
  & L$_2$-SP & 0.04687 & 0.8827 & 0.00698 & 0.01075 & \cellcolor[rgb]{ .886,  .937,  .855}0.01126 & 0.9748 & 0.00119 & 0.00098 & 0.02866 & 0.5463 & 0.00388 & 0.00484 \\
  & EWC & \cellcolor[rgb]{ .741,  .843,  .933}\textbf{0.04509} & \cellcolor[rgb]{ .741,  .843,  .933}\textbf{0.8914} & \cellcolor[rgb]{ .741,  .843,  .933}\textbf{0.00660} & \cellcolor[rgb]{ .886,  .937,  .855}0.00989 & \cellcolor[rgb]{ .741,  .843,  .933}\textbf{0.00996} & \cellcolor[rgb]{ .741,  .843,  .933}\textbf{0.9803} & \cellcolor[rgb]{ .741,  .843,  .933}\textbf{0.00105} & \cellcolor[rgb]{ .741,  .843,  .933}\textbf{0.00088} & \cellcolor[rgb]{ .886,  .937,  .855}0.02723 & \cellcolor[rgb]{ .886,  .937,  .855}0.5904 & \cellcolor[rgb]{ .886,  .937,  .855}0.00361 & \cellcolor[rgb]{ .886,  .937,  .855}0.00415 \\
  & PhysGuard & 0.04893 & 0.8721 & 0.00690 & \cellcolor[rgb]{ .741,  .843,  .933}\textbf{0.00935} & 0.01350 & 0.9638 & 0.00138 & 0.00109 & 0.03712 & 0.2388 & 0.00493 & 0.00767 \\
\bottomrule
\end{tabular}
\end{table}

\subsection{Qualitative Visualisations}
\label{app:dpot_viz}

Figures~\ref{fig:app_cylinder_dpot}--\ref{fig:app_combustion_dpot} show DPOT-S predictions on the
three scenarios under each baseline. They use the same layout as the per-architecture panels in
Appendix~\ref{app:visualization}. Consistent with Table~\ref{tab:dpot}, PhysGuard traces follow
the ground truth closely on Cylinder Flow but underperform DFT/EWC on Controlled Cylinder and
Turbulent Combustion, visible as slightly blurrier vortex boundaries and weaker local intensity.

\begin{figure}[H]
  \centering
  \includegraphics[width=0.66\linewidth]{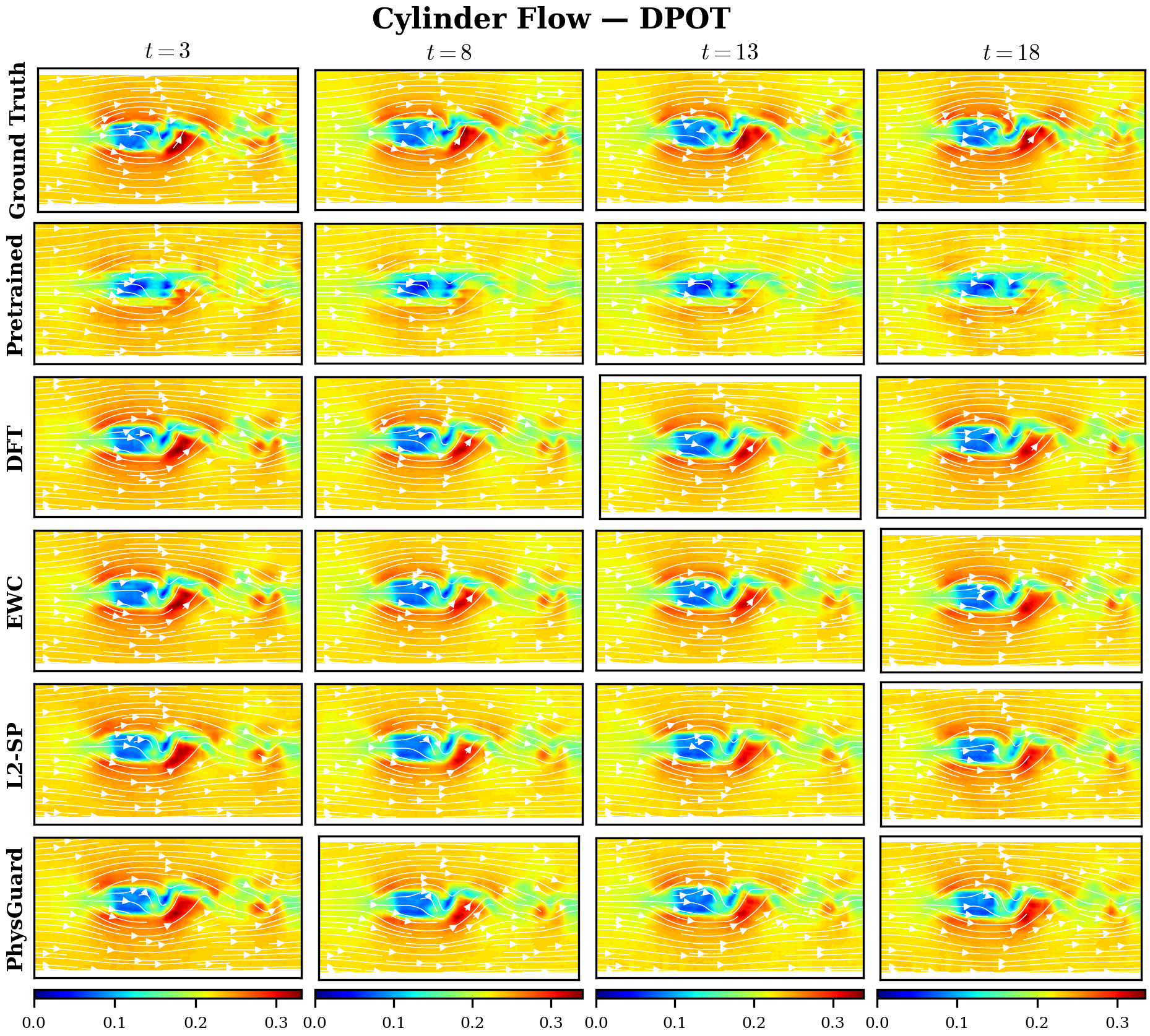}
  \caption{DPOT-S predictions on Cylinder Flow.}
  \label{fig:app_cylinder_dpot}
\end{figure}
\begin{figure}[H]
  \centering
  \includegraphics[width=0.66\linewidth]{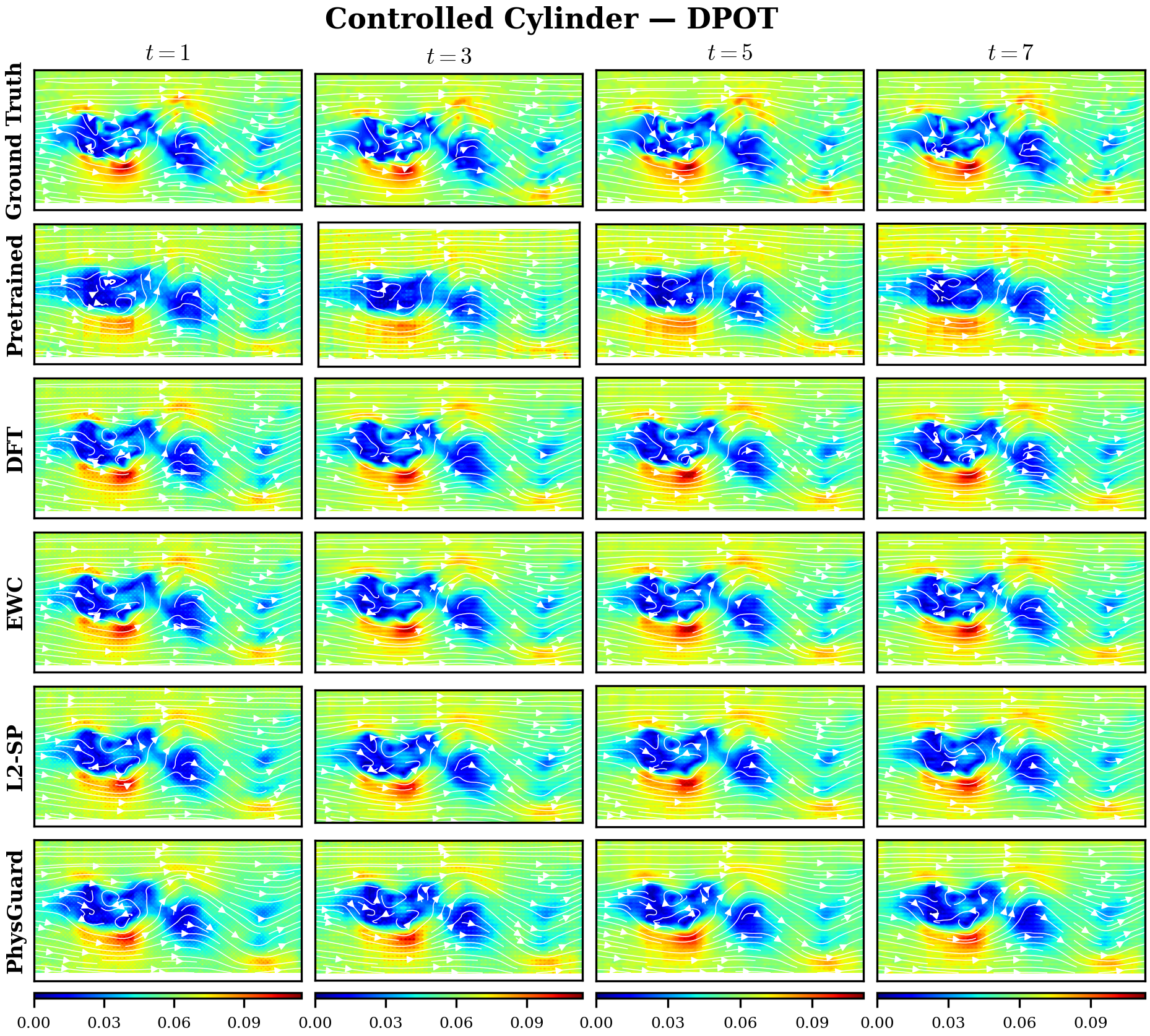}
  \caption{DPOT-S predictions on Controlled Cylinder Flow.}
  \label{fig:app_controlled_cylinder_dpot}
\end{figure}

\begin{figure}[H]
  \centering
  \includegraphics[width=0.55\linewidth]{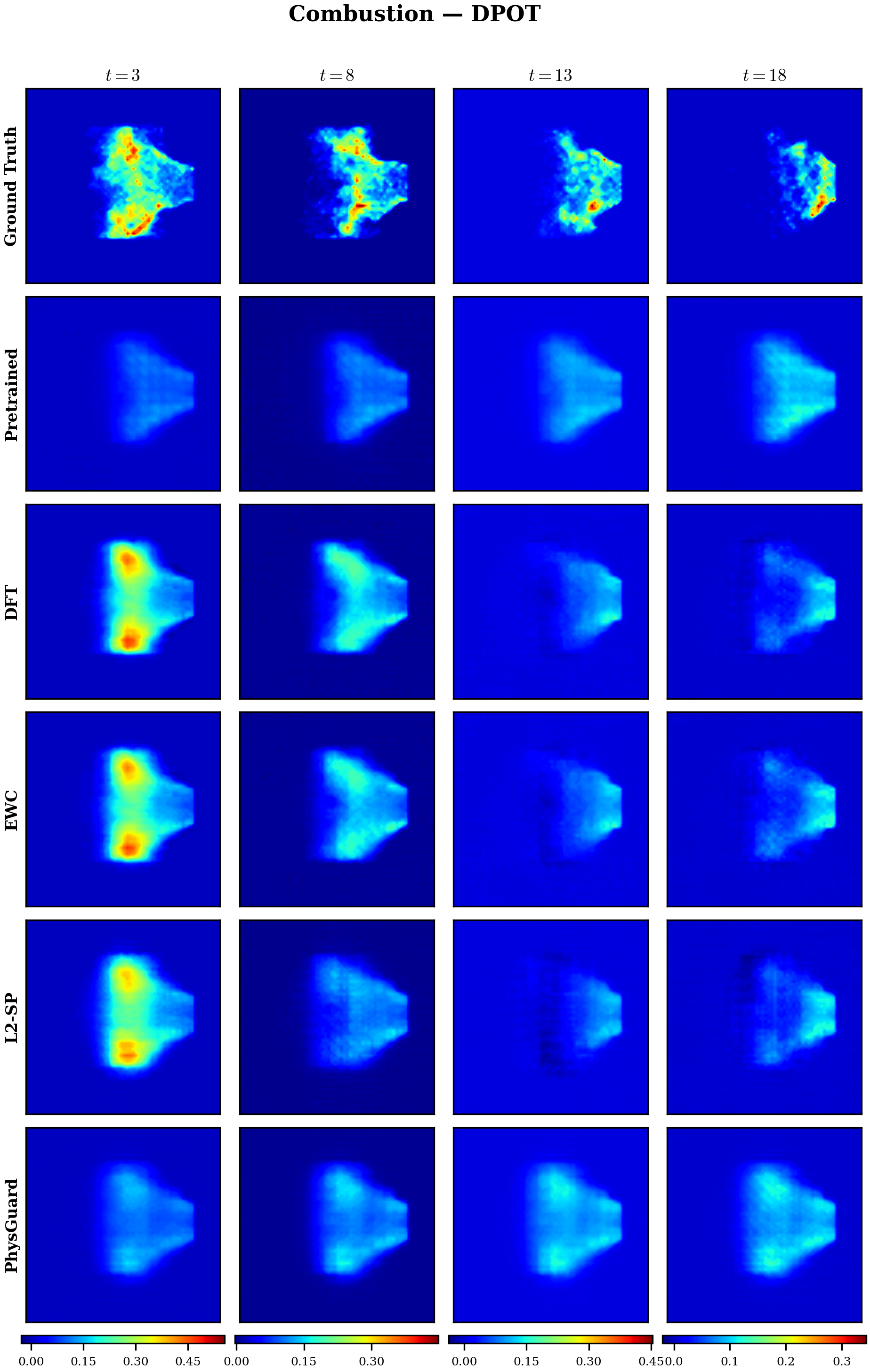}
  \caption{DPOT-S predictions on Turbulent Combustion.}
  \label{fig:app_combustion_dpot}
\end{figure}

\newpage

\section*{NeurIPS Paper Checklist}

\begin{enumerate}

\item {\bf Claims}
    \item[] Question: Do the main claims made in the abstract and introduction accurately reflect the paper's contributions and scope?
    \item[] Answer: \answerYes{}
    \item[] Justification: The abstract and Section~\ref{sec:intro} describe the contributions evaluated in Section~\ref{sec:experiments}.

\item {\bf Limitations}
    \item[] Question: Does the paper discuss the limitations of the work performed by the authors?
    \item[] Answer: \answerYes{}
    \item[] Justification: Section~\ref{sec:limitations} discusses spectral assumptions and the limited benchmark scope.

\item {\bf Theory assumptions and proofs}
    \item[] Question: For each theoretical result, does the paper provide the full set of assumptions and a complete (and correct) proof?
    \item[] Answer: \answerYes{}
    \item[] Justification: Appendices~\ref{app:gram} and~\ref{app:complex} provide the Gram-matrix proof and complex-weight derivation.

    \item {\bf Experimental result reproducibility}
    \item[] Question: Does the paper fully disclose all the information needed to reproduce the main experimental results of the paper to the extent that it affects the main claims and/or conclusions of the paper (regardless of whether the code and data are provided or not)?
    \item[] Answer: \answerYes{}
    \item[] Justification: Algorithm~\ref{alg:physguard} and Appendix~\ref{app:implementation} describe the method and experimental protocol.

\item {\bf Open access to data and code}
    \item[] Question: Does the paper provide open access to the data and code, with sufficient instructions to faithfully reproduce the main experimental results, as described in supplemental material?
    \item[] Answer: \answerYes{}
    \item[] Justification: The code is open source, and RealPDEBench is publicly available.

\item {\bf Experimental setting/details}
    \item[] Question: Does the paper specify all the training and test details (e.g., data splits, hyperparameters, how they were chosen, type of optimizer) necessary to understand the results?
    \item[] Answer: \answerYes{}
    \item[] Justification: Appendix~\ref{app:implementation} specifies data splits, architectures, optimizers and training schedules.

\item {\bf Experiment statistical significance}
    \item[] Question: Does the paper report error bars suitably and correctly defined or other appropriate information about the statistical significance of the experiments?
    \item[] Answer: \answerYes{}
    \item[] Justification: Appendix~\ref{app:robustness_seeds} reports three-seed means and available standard deviations for DFT and PhysGuard.

\item {\bf Experiments compute resources}
    \item[] Question: For each experiment, does the paper provide sufficient information on the computer resources (type of compute workers, memory, time of execution) needed to reproduce the experiments?
    \item[] Answer: \answerYes{}
    \item[] Justification: Appendices~\ref{app:implementation} and~\ref{app:gram} report hardware, memory, training budgets and Fisher estimation times.
    
\item {\bf Code of ethics}
    \item[] Question: Does the research conducted in the paper conform, in every respect, with the NeurIPS Code of Ethics \url{https://neurips.cc/public/EthicsGuidelines}?
    \item[] Answer: \answerYes{}
    \item[] Justification: The study uses public scientific data and involves no human subjects or private data collection.

\item {\bf Broader impacts}
    \item[] Question: Does the paper discuss both potential positive societal impacts and negative societal impacts of the work performed?
    \item[] Answer: \answerYes{}
    \item[] Justification: Appendix~\ref{app:broader_impact} discusses model reuse and risks from residual prediction errors.

\item {\bf Safeguards}
    \item[] Question: Does the paper describe safeguards that have been put in place for responsible release of data or models that have a high risk for misuse (e.g., pre-trained language models, image generators, or scraped datasets)?
    \item[] Answer: \answerNA{}
    \item[] Justification: No high-risk generative models or scraped datasets are introduced.

\item {\bf Licenses for existing assets}
    \item[] Question: Are the creators or original owners of assets (e.g., code, data, models), used in the paper, properly credited and are the license and terms of use explicitly mentioned and properly respected?
    \item[] Answer: \answerYes{}
    \item[] Justification: We credit \href{https://huggingface.co/datasets/AI4Science-WestlakeU/RealPDEBench}{RealPDEBench} (CC BY-NC 4.0) and \href{https://huggingface.co/hzk17/DPOT}{DPOT} (Apache 2.0).

\item {\bf New assets}
    \item[] Question: Are new assets introduced in the paper well documented and is the documentation provided alongside the assets?
    \item[] Answer: \answerYes{}
    \item[] Justification: The \href{https://github.com/Data-Driven-Computational-Geotechnics/PhysGuard#readme}{repository README} documents installation, training and evaluation.

\item {\bf Crowdsourcing and research with human subjects}
    \item[] Question: For crowdsourcing experiments and research with human subjects, does the paper include the full text of instructions given to participants and screenshots, if applicable, as well as details about compensation (if any)? 
    \item[] Answer: \answerNA{}
    \item[] Justification: No crowdsourcing or human subjects are involved.

\item {\bf Institutional review board (IRB) approvals or equivalent for research with human subjects}
    \item[] Question: Does the paper describe potential risks incurred by study participants, whether such risks were disclosed to the subjects, and whether Institutional Review Board (IRB) approvals (or an equivalent approval/review based on the requirements of your country or institution) were obtained?
    \item[] Answer: \answerNA{}
    \item[] Justification: No human subjects research is conducted.

\item {\bf Declaration of LLM usage}
    \item[] Question: Does the paper describe the usage of LLMs if it is an important, original, or non-standard component of the core methods in this research? Note that if the LLM is used only for writing, editing, or formatting purposes and does \emph{not} impact the core methodology, scientific rigor, or originality of the research, declaration is not required.
    \item[] Answer: \answerNA{}
    \item[] Justification: LLMs are used only for writing assistance.

\end{enumerate}

\end{document}